\documentclass[twoside,11pt]{article}

\usepackage{blindtext}

\usepackage{jmlr2e}

\usepackage{amsmath}

\usepackage{subfigure}
\usepackage{multirow}
\usepackage{graphicx}
\usepackage{epstopdf}
\usepackage{amssymb}
\usepackage[normalem]{ulem}
\usepackage[ruled]{algorithm2e}
\allowdisplaybreaks[2]

\newtheorem{Definition}{Definition}
\newtheorem{Theorem}{Theorem}
\newtheorem{Lemma}{Lemma}
\newtheorem{Remark}{Remark}
\newtheorem{Corollary}{Corollary}
\newtheorem{Assumption}{Assumption}

\usepackage{lastpage}
\firstpageno{1}

\begin{document}

% \title{The Dual Role of Edge in Graph Learning: Simplifying Learning but Hindering Generalization}

\title{Rethinking Generalization in Graph Neural Networks: A Structural Complexity Perspective}

% \title{Edge-Driven Complexity in Graph Learning: A Theory of  Generalization and Learnability Difficulty}

\author{\name Peiyao Wang \email wangpeiyao1128@163.com \\
       \addr Institute of Intelligent Information Processing\\
       Shanxi University, Taiyuan, 030006, China
       \AND
       \name Liang Bai \thanks{Corresponding author} \email bailiang@sxu.edu.cn \\
       \addr Institute of Intelligent Information Processing\\
       Shanxi University, Taiyuan, 030006, China
       \AND
       \name Xian Yang \email xian.yang@manchester.ac.uk \\
       \addr Alliance Manchester Business School\\
       University of Manchester, Manchester, M13 9PL, UK
       \AND
       \name Richard Yi Da Xu \email xuyida@hkbu.edu.hk \\ \addr Department of Mathematics\\ Hong Kong Baptist University, Hong Kong SAR, China
       \AND
       \name Jiye Liang \email ljy@sxu.edu.cn \\
       \addr Institute of Intelligent Information Processing\\
       Shanxi University, Taiyuan, 030006, China}

\editor{}

\maketitle

\begin{abstract}%   

Graph neural networks (GNNs) have emerged as a fundamental tool for learning from graph-structured data, achieving strong performance across a wide range of applications. However, understanding their generalization capabilities remains challenging due to the complex structural dependencies inherent in such data. Existing generalization analyses largely follow the classical machine learning paradigm, focusing primarily on model complexity while overlooking the fundamental role of graph structure. Therefore, in this work, we systematically investigate this role by asking: does the graph structure actually influence generalization, and if so, by how much? To answer the first question and validate our intuition, we theoretically prove that incorporating more edges into the prediction process transforms the input representations to be overly accommodating to the output model, thereby inducing overfitting. To address the second question, we formulate a structural complexity measure based on the number of effective edges and derive a Rademacher complexity-based generalization bound. In doing so, we demonstrate that GNN generalization depends explicitly on structural complexity, alongside traditional parameter-dependent factors. Motivated by these theoretical findings, we propose a structural entropy regularization method. This approach controls structural complexity by regulating effective edges to balance underfitting and overfitting, ultimately improving the generalization performance of GNNs.
\end{abstract}


\section{Introduction}

Graph neural networks (GNNs) have emerged as a fundamental tool for processing graph-structured data, effectively capturing both node features and relational information. In recent years, GNNs have achieved remarkable success across a wide range of real-world applications, including drug discovery \citep{sun2022does, hoang2024knowledge}, traffic forecasting \citep{fang2021spatial, zhang2025efficient}, and recommender systems \citep{yan2025enhancing, akbari2025optimal}. In these applications, GNNs iteratively aggregate information from neighboring nodes to learn expressive node and graph representations, demonstrating strong performance on tasks such as node classification \citep{fan2025towards, yuan2025much}, link prediction \citep{zhang2025trapt, li2025out}, and graph-level prediction \citep{wang2025rag4gfm, ju2025cluster}.

Despite their empirical success, understanding the generalization of GNNs remains a significant theoretical challenge. In classical statistical learning theory, generalization refers to a model's ability to achieve low prediction error on unseen data drawn from the same distribution as the training data \citep{vapnik1999overview}. In graph-based learning, this goal is complicated by the non-i.i.d. nature of nodes. The graph topology is directly incorporated into the feature encoding and prediction process, which provides a strong structural prior. This structural information facilitates fitting on the training graph, but it also makes GNNs highly sensitive to changes in graph structure. When the structural patterns differ in unseen graphs, the learned representations may fail to generalize and the model performance can degrade significantly.

To address this challenge, existing studies have explored the generalization of GNNs from both algorithmic and theoretical perspectives. On the algorithmic side, various strategies have been proposed to improve robustness to structural variations, including graph augmentation \citep{lu2024graph, wu2024graph, li2024graph} and graph invariant learning \citep{yuan2025structure, wang2025uncertainty, zhang2025gi}. These approaches aim to simulate unseen test graphs or enhance invariance to topological changes. On the theoretical side, most existing studies partition the original graph into subgraphs composed of a central node and its neighbors and treat these subgraphs as approximately independent samples. This strategy allows analytical tools such as stability analysis \citep{verma2019stability, yang2025deeper}, Rademacher complexity \citep{garg2020generalization, esser2021learning, lv2021generalization} and PAC-Bayesian frameworks \citep{ju2023generalization} to be extended to graph-structured data, enabling the study of how factors such as parameter complexity and the spectral properties of graph filters influence the generalization performance of GNNs.

Despite these advances, current approaches mainly project graphs into latent spaces or partition them into local subgraphs treated as independent samples, thereby enabling the use of classical machine learning tools to analyze the generalization of GNNs. Such analyses fail to consider how the structure of the graph itself influences GNN learning. In practice, GNNs often face generalization difficulties, which are influenced not only by model complexity but also to a large extent by the strong structural priors embedded through message passing. Therefore, the generalization problem in GNNs fundamentally involves a dual overfitting of both parameters and structure. However, existing studies have not provided a theoretical explanation for this phenomenon, leaving a significant gap in our understanding.

In this paper, we develop a structure-driven theoretical framework to bridge this gap, linking three key components into a coherent account of GNN generalization: identifying structure-induced overfitting, quantifying it through structural complexity, and controlling it via structure-aware regularization. Specifically, we first theoretically reveal that graph structure itself can induce overfitting in GNNs. As a structural prior embedded into the encoding and propagation process, graph structure promotes smoother node representations through neighborhood aggregation, which makes the model easier to fit to the training data. However, when the training and test graphs exhibit structural discrepancies, such structural reliance can impair generalization performance. Building on this analysis, we introduce a structural complexity defined by the number of effective aggregation edges and investigate how it specifically influences generalization. Then, we leverage Rademacher complexity to derive explicit generalization error bounds for representative message-passing GNN architectures. Our results reveal that the excess risk depends explicitly on structural complexity induced by aggregation, in addition to classical parameter-dependent factors. These theoretical findings further motivate a practical regularization strategy. To control structural complexity in a principled manner, we propose Structure Entropy Regularization (SER), a differentiable surrogate that preserves the smoothing effect of aggregation while regulating structural complexity. By balancing underfitting and overfitting, SER improves generalization performance in practice.

The main contributions of this paper are summarized as follows:

\begin{itemize}

\item[1.] We theoretically show that graph structure can cause overfitting in GNNs, making it easier to fit the training data while potentially harming generalization.

\item[2.] We define structural complexity as the number of effective aggregation edges and use this to derive generalization error bounds for GNNs, revealing that generalization depends on both structural and model complexity.

\item[3.] We propose a structure entropy–based regularization method to control effective edge usage during training. Experiments on benchmark datasets demonstrate that this approach improves generalization, confirming our theoretical findings.

\end{itemize}

The remainder of this paper is organized as follows. In Section \ref{preli}, we introduce the notations used throughout the paper, provide a preliminary description of GNNs, and review theoretical studies on GNN generalization along with their limitations. Section \ref{Difficulty} demonstrates theoretically that graph structure can induce overfitting in GNNs. In Section \ref{generalization}, we introduce edge-based structural complexity and derive generalization error bounds for GNNs based on it, which further motivate the development of a structure entropy regularization method. Section \ref{experiments} reports experimental results on benchmark datasets to validate the effectiveness of the proposed regularization method. Finally, Section \ref{conclusion} concludes this paper.

\section{Preliminaries}
\label{preli}

In this section, we introduce the key notations and mathematical concepts necessary for understanding the theoretical framework and analysis of GNNs presented in the subsequent sections. We also review existing theoretical studies on the generalization of GNNs and discuss their limitations.

\begin{table}[t]
\caption{Notations used in this paper}
\centering
\begin{tabular}{cc}
\hline
Notation & Description \\
\hline
$G=(V,E)$ & input graph with node set $V$ and edge set $E$ \\
$X \in \mathbb{R}^{n \times d}$&node feature matrix \\
$A \in \mathbb{R}^{n \times n}$&adjacency matrix \\
$D \in \mathbb{R}^{n \times n}$ & degree matrix \\
$\widetilde{A}, \widetilde{D}$ & adjacency and degree matrices with self-loops \\
$\widetilde{L}$ & augmented normalized Laplacian matrix \\
$Y \in \{0,1\}^{n \times c}$ & one-hot label matrix \\
$\mathcal{S}, \mathcal{U}$ & labeled and unlabeled node sets \\
$m$ & number of labeled nodes, $|\mathcal{S}|$ \\
$\mathcal{F}$ & hypothesis class induced by the GNN \\
$\mathcal{H}$ & loss-composed hypothesis class \\
$H^{(k)}$ & node representations at layer $k$ \\
$E(H^{(k)})$ & Dirichlet energy at layer $k$ \\
$\hat{A}$ & normalized aggregation matrix \\
$A^{(k)}$ & attention aggregation matrix at layer $k$ \\
$\eta$ & structural complexity under fixed aggregation \\
$\eta_\tau^{(k)}$ & effective structural complexity at layer $k$ \\
$\Delta$ & number of attention heads \\
$\hat{\epsilon}(f)$ & empirical risk over labeled nodes \\
$\epsilon(f)$ & population risk \\
$\widehat{\mathfrak{R}}(\mathcal{F})$ & empirical Rademacher complexity of $\mathcal{F}$ \\
$\lambda$ & regularization coefficient \\
\hline
\end{tabular}
\label{notation}
\end{table}

\subsection{Notation}

We consider an undirected graph $G=(V, E)$, where $V$ denotes the set of nodes with size $n$, and $E \subseteq V \times V$ denotes the set of edges. The adjacency matrix of the graph is denoted by $A \in \{0,1\}^{n \times n}$, where $A_{ij} = 1$ if and only if there is an edge between nodes $v_i$ and $v_j$, and $A_{ij} = 0$ otherwise. For each node $v_i$, we define its neighborhood as $\mathcal{N}(i) = \left\{ v_j|A_{ij} = 1 \right\}$. Let $D \in \mathbb{R}^{n \times n}$ denote the corresponding degree matrix, where $D_{ii} = \sum_{j=1}^{n} A_{ij}$. We define $\widetilde{A} = A + I_n$ and $\widetilde{D} = D + I_n$ as the adjacency and degree matrices of the graph augmented with self-loops. The augmented normalized Laplacian matrix is given by $\widetilde{L} = I_n - \widetilde{D}^{-\frac{1}{2}}\widetilde{A}\widetilde{D}^{-\frac{1}{2}}$. The node features are represented by a matrix $X \in \mathbb{R}^{n \times d}$, where $x_v \in \mathbb{R}^d$ is the feature vector associated with node $v$. The node labels are denoted by $Y = \{ y_v \}_{v \in V}$, where $y_v \in \mathbb{R}^c$ represents the label of node $v$. We consider a semi-supervised node classification problem on the fixed graph $G$. Node features and the graph structure are observed for all nodes, while labels are available only on a subset $\mathcal{S} \subset V$ with $|\mathcal{S}| = m$. The remaining nodes $\mathcal{U} = V \setminus \mathcal{S}$ are unlabeled.

Let $\mathcal{F}$ denote the hypothesis class induced by a message-passing GNN, each hypothesis $f \in \mathcal{F}$ maps node features and the graph structure to node-level predictions $f_v \in \mathbb{R}^c$.  Given a bounded loss function $\ell(\cdot,\cdot)$, define the loss-composed function class $\mathcal{H} = \{ \ell(y, f(\cdot)), f\in\mathcal{F}\}$. The empirical risk over the labeled nodes is defined as $\hat{\epsilon}(f) = \frac{1}{m} \sum_{v \in \mathcal{S}} \ell(f_v, y_v)$, and the population risk is defined as $\epsilon(f) = \mathbb{E}_{v \sim \mathcal{D}} \, \ell(f_v, y_v)$, where $\mathcal{D}$ denotes an underlying node distribution.
The generalization error is measured by the gap $\epsilon(f) - \hat{\epsilon}(f)$. Throughout the paper, $\|\cdot\|_2$ denotes the spectral norm, $\|\cdot\|_F$ denotes the Frobenius norm, and
$\|\cdot\|_0$ denotes the entrywise $\ell_0$ norm. For clarity and ease of reference, we summarize the main notations used throughout the paper in Table \ref{notation}.

\subsection{Graph Neural Networks}

GNNs compute node representations by iteratively aggregating information from local neighborhoods on the graph, such that each node representation is updated based on its own features and the features of its neighboring nodes. This computation is commonly formalized under the message-passing framework. Specifically, the representation of node $v$ at layer $k$ is given by:

\begin{equation} \label{eq7}
h_v^{(k)} = f^{(k)}\left(h_v^{(k-1)}, AGG\left\{h_u^{(k-1)}|u \in \mathcal{N}(v)\right\}\right),
\end{equation}
where $AGG\left\{\cdot\right\}$ is a permutation-invariant aggregator, $f^{(k)}(\cdot)$ combines aggregated messages with the node’s previous state, and the initial representation is given by $h_v^{(0)} = x_v$. A prominent instantiation of this framework is the graph convolutional network (GCN), in which neighborhood aggregation is realized through a fixed, normalized graph operator derived from the adjacency structure. In matrix form, the update of a GCN at layer $k$ can be written as:

\begin{equation} \label{gcn}
H^{(k)} = \sigma\left( \hat{A}H^{(k-1)}W^{(k)}\right),
\end{equation}
where $\hat{A} = \widetilde{D}^{-\frac{1}{2}}\widetilde{A}\widetilde{D}^{-\frac{1}{2}}$ is the normalized adjacency matrix with self-loops, $W^{(k)}$ is the trainable weight matrix at layer $k$, and $\sigma(\cdot)$ denotes a nonlinear activation function. Graph attention networks (GATs) extend this framework by introducing learnable attention weights:

\begin{equation} \label{gat}
h_v^{(k)} = \sigma\left(\sum_{u \in \mathcal{N}(v)} \alpha_{vu}^{(k)}W^{(k)}h_u^{(k-1)} \right),
\end{equation}
where $\alpha_{vu}^{(k)}$ are normalized attention coefficients. More recently, graph transformers (GTs) generalize attention-based message passing by extending aggregation from local neighborhoods to a potentially global set of nodes, while incorporating structural information via positional encodings. At each layer, node representations are updated through multi-head self-attention, which can be written as:

\begin{equation} \label{gt}
\hat{h}_v^{(k)} = \mathit{O}^{k} \mathop{\big\|}_{\delta=1}^{\Delta} \left(\sum_{u \in \mathcal{N}(v)} \alpha_{vu}^{(k, \delta)}V^{(k, \delta)}h_u^{(k-1)} \right),
\end{equation}
where $\alpha_{vu}^{(k, \delta)}$ denotes attention weights, $\mathop{\big\|}$ denotes concatenation across heads, $V^{(k, \delta)}  \in \mathbb{R}^{d_\delta \times \widetilde{d}}$ is a trainable projection matrix, and $O^{k} \in \mathbb{R}^{\widetilde{d} \times \widetilde{d}}$ is a learnable output projection. Following the attention-based aggregation, residual connections, normalization and a position-wise feed-forward network are applied:

\begin{align}
\widetilde{h}_v^{(k)} &= \mathrm{Norm} \left(h_v^{(k-1)} + \hat{h}_v^{(k)} \right), \label{tra1} \\
\hat{\hat{h}}_v^{(k)} &= W_2^{(k)} \sigma \left( W_1^{(k)} \widetilde{h}_v^{(k)} \right), \label{tra2} \\
h_v^{(k)} &= \mathrm{Norm} \left(\widetilde{h}_v^{(k)} + \hat{\hat{h}}_v^{(k)} \right), \label{tra3}
\end{align}
where $W_1^{(k)} \in \mathbb{R}^{2\widetilde{d} \times \widetilde{d}}$ and $W_2^{(k)} \in \mathbb{R}^{\widetilde{d} \times 2\widetilde{d}}$ are trainable parameters, and $\mathrm{Norm}(\cdot)$ denotes either Layer Normalization or Batch Normalization.

\subsection{Generalization Theory for Graph Neural Networks}

Understanding the generalization of GNNs has attracted increasing attention in recent years. Unlike classical learning settings where training samples are assumed to be independent and identically distributed, node-level prediction on graphs induces strong dependencies through edges. These dependencies violate the i.i.d. assumption underlying many classical generalization analyses, making it challenging to directly apply standard learning theory. Early theoretical studies extend algorithmic stability frameworks to GNNs. \citet{verma2019stability} analyze the stability of GCNs by treating local subgraphs as approximately independent training samples obtained through random neighborhood sampling. Under this framework, for $0 < \delta < 1$, the following expected generalization gap bound holds with probability at least $1 - \delta$ for all $f \in \mathcal{F}$:

\begin{equation}
\mathbb{E}_{\text{SGD}}\big[\epsilon(A_{S}) - \hat{\epsilon}(A_{S})\big] 
\leq \frac{1}{m} \mathcal{O}\Big( (\lambda_G^{\max})^{2T} \Big)
+ \Bigg( \mathcal{O}\big((\lambda_G^{\max})^{2T}\big) + M \Bigg) 
\sqrt{\frac{\log \frac{1}{\delta}}{2m}},
\nonumber
\end{equation}
where $A_{S}$ denotes the GCN model trained on the sampled set of subgraphs $S$, $\epsilon(A_{S})$ is the expected risk of the learned model, and $\hat{\epsilon}(A_{S})$ is the empirical risk evaluated on the sampled subgraphs. Here, $m$ represents the number of sampled subgraphs used for training, $T$ is the number of GCN layers, and $\lambda_G^{\max}$ denotes the largest eigenvalue of the normalized adjacency matrix, which controls the magnitude of signal propagation across the graph. The constant $M$ depends on the Lipschitz continuity of the loss function and bounds the output range of the model. This bound reveals that the spectral norm of graph convolution filters plays a key role in controlling the generalization error. Despite providing a formal guarantee, this stability-based analysis relies on several strong assumptions. It requires artificially partitioning the graph into approximately independent subgraphs and assumes that these subgraphs are sampled independently, which may not hold in practice for densely connected or highly correlated graphs.

To overcome the independence assumptions underlying the above stability-based analysis, \citet{esser2021learning} introduce a transductive Rademacher complexity framework for graph-based learning. In this setting, the entire graph is treated as fixed, and the hypothesis class $\mathcal{F}$ is evaluated jointly over all nodes. Using this complexity measure, the generalization error can be bounded in terms of the empirical risk and a complexity term characterizing the capacity of the hypothesis class. Specifically, for $0 < \delta < 1$, with probability at least $1 - \delta$ for all $f \in \mathcal{F}$ the following holds:

\begin{equation}
\epsilon(f) \le \hat{\epsilon}(f)
+ \widehat{\mathfrak{R}}(\mathcal{F})
+ c_4 \frac{n\sqrt{\min \{m,n-m\}}}{m(n-m)} + c_5 \sqrt{\frac{n}{m(n-m)}\ln\left(\frac{1}{\delta}\right)},
\nonumber
\end{equation}
where $\widehat{\mathfrak{R}}(\mathcal{F})$ denotes the empirical transductive Rademacher complexity of the hypothesis class, while $c_4$ and $c_5$ denote absolute constants. This framework removes the need for independence assumptions by analyzing the hypothesis class over a fixed graph structure. By assigning random signs to nodes and measuring the expected oscillation of the hypothesis class, the analysis quantifies function complexity in graph-dependent settings. However, the resulting bounds are still largely governed by quantities related to parameter complexity and network architecture, such as the norms of learnable weights, spectral properties of graph filters, or the depth of the network. As a consequence, the role of the graph structure itself is mostly captured indirectly through global spectral quantities or diffusion operators.

From a different perspective, \citet{ju2023generalization} further derive data-dependent and potentially tighter bounds using a PAC-Bayesian approach, combining spectral norm constraints with measures of local model flatness. Their analysis captures both the spectral properties of diffusion operators and the sensitivity of the learned model. However, it typically relies on differentiability assumptions and involves higher-order quantities that are difficult to estimate in practice. In addition, \citet{tang2023towards} incorporate optimization dynamics into the analysis, explicitly characterizing how step sizes, iteration numbers, and network Lipschitz constants interact with the partition of training and test nodes in transductive settings. More recently, \citet{wang2025generalization} investigate the robustness of GNN generalization under generative model mismatch. Focusing on geometric graphs generated from manifold models, they analyze scenarios where training and testing graphs arise from different underlying manifolds. By characterizing the mismatch as node feature and edge perturbations, they establish bounds showing that the generalization gap decreases with the number of training nodes and increases with the manifold dimension and mismatch magnitude.

Overall, these studies estimate the generalization performance of GNNs from different perspectives. However, the resulting generalization bounds are still primarily characterized by quantities related to model parameters, spectral properties of graph operators or training dynamics, without explicitly analyzing how the graph structure itself affects GNN generalization. To address this limitation, we first theoretically examine how the graph structure can induce overfitting in GNNs. Building on this insight, we study generalization from the perspective of structural complexity, which quantifies the influence of graph structure on the learning process. Based on this measure, we derive a new generalization error bound that explicitly links graph structure to the generalization performance of GNNs.

\section{Overfitting in GNNs Caused by Graph Structure}
\label{Difficulty}

In this section, we theoretically examine how graph structure can lead to overfitting in GNNs. We show that incorporating edges during message passing makes node representations smoother, which simplifies fitting the training data. At the same time, this structural smoothing can cause the model to rely more on graph-specific patterns, increasing the risk of overfitting and reducing generalization.

\subsection{Smoothness as a Measure of Fitting Capacity}

The ability of a GNN to fit the data is influenced by the structural priors imposed by graph edges. Edges encourage node representations to vary smoothly across the graph, and smoother representations correspond to target functions that are easier to capture \citep{nt2019revisiting}. To quantify this effect, we first need a measure for smoothness on graphs. To this end, we employ the Dirichlet energy \citep{li2020dirichlet}, defined as follows.

\begin{Definition}(Dirichlet energy)\label{defini1}
Let $G=(V,E)$ be an undirected graph with augmented normalized Laplacian $\widetilde{L}$, and let $H^{(k)} = [h_1^{(k)}, \ldots, h_n^{(k)}]^\top \in \mathbb{R}^{n \times d}$ denote the node representation matrix learned by a graph neural network at the $k$-th layer. The Dirichlet energy of $H^{(k)}$ with respect to graph $G$ is defined as:

\begin{equation} \label{diri}
E(H^{(k)}) = tr\left( H^{(k) \top }\widetilde{L} H^{(k)} \right) = \frac{1}{2} \sum_{v \in [n]} \sum_{u \in \mathcal{N}(v)} \widetilde{a}_{vu} \left \| \frac{h_v^{(k)}}{\sqrt{1+d_v}}-\frac{h_u^{(k)}}{\sqrt{1+d_u}} \right \|_2^2,
\end{equation}
where $\widetilde{a}_{vu}$ denotes the $(v, u)$-th entry of the augmented adjacency matrix, and $d_v$ is the degree of node $v$ in the original graph.
\end{Definition}

The definition above shows that Dirichlet energy aggregates representation discrepancies across edges under the normalization induced by $\widetilde{L}$. As a quadratic form associated with the graph Laplacian, Dirichlet energy admits an equivalent characterization in the spectral domain, as stated in the following Lemma.

\begin{Lemma}
\label{lemma1}
Let $\widetilde{L}$ be the augmented normalized Laplacian of graph $G$, with eigenpairs $\{(\lambda_i, u_i)\}_{i=1}^n$, where $0=\lambda_1 \le \lambda_2 \le \ldots \le \lambda_n$.  
For any graph signal $f \in \mathbb{R}^n$ with expansion:
\begin{equation}
f = \sum_{i=1}^n \gamma_i u_i,
\end{equation}
the Dirichlet energy of $f$ is given by:
\begin{equation} \label{diri2}
E(f) = f^\top \widetilde{L} f = \sum_{i=1}^n \lambda_i \gamma_i^2.
\end{equation}
\end{Lemma}

\begin{proof}
Since $\widetilde{L}$ is symmetric and positive semi-definite, its eigenvectors $\{u_i\}_{i=1}^n$ form an orthonormal basis of $\mathbb{R}^n$. Substituting the spectral expansion of $f$ into the quadratic form yields:
\[
f^\top \widetilde{L} f
= \left( \sum_{i=1}^n \gamma_i u_i \right)^\top
\widetilde{L}
\left( \sum_{j=1}^n \gamma_j u_j \right)
= \sum_{i=1}^n \lambda_i \gamma_i^2,
\]
where we use the orthonormality of eigenvectors and the eigenvalue equation $\widetilde{L}u_i=\lambda_i u_i$.
\end{proof}

\begin{Remark}
\noindent
Lemma \ref{lemma1} shows that Dirichlet energy quantifies the spectral complexity of a graph signal by weighting its Laplacian eigencomponents according to their frequencies. Since larger eigenvalues correspond to rapidly varying modes across edges, a large Dirichlet energy indicates strong local variation on the graph, whereas a small value implies smoothness with respect to the graph structure.
\end{Remark}

This property makes Dirichlet energy a natural tool for analyzing representation smoothness in GNNs. Since GNNs act as low-pass filters, they preferentially learn low-frequency components of graph signals. Consequently, lower Dirichlet energy allows GNNs to fit the target data more easily. In the next subsection, we analyze how incorporating graph structure into GNNs affects the Dirichlet energy.

\subsection{Impact of Graph Structure on the Fitting Capacity of GNNs}

To analyze how graph structure affects fitting capacity, we compare a standard GNN against its edge-free extreme, where the model degrades to a multi-layer perceptron (MLP) that updates node representations using only their own features. We take GCNs as a representative example to make the analysis concrete. The difference in Dirichlet energy between these two scenarios is characterized by the following theorem.

\begin{Theorem}
\label{thm:gcn_mlp}
Let $H_{GCN}^{(1)}$ and $H_{MLP}^{(1)}$ be the outputs of a one-layer GCN with and without edge-based aggregation, respectively. Then the Dirichlet energy satisfies:

\begin{equation}
E(H_{GCN}^{(1)}) \le \left \| \hat{A} \right \|_2^2 E(H_{MLP}^{(1)}).
\end{equation}
\end{Theorem}

\begin{proof}
According to the definition of Dirichlet energy in Eq.(\ref{diri}), the Dirichlet energy of the output of a one-layer GCN can be written as:
\begin{equation} \label{gcnenergy}
E(H_{GCN}^{(1)}) = tr\left( H_{GCN}^{(1) \top }\widetilde{L} H_{GCN}^{(1)} \right).
\end{equation}

Similarly, the Dirichlet energy of the output of a one-layer MLP, which corresponds to a GNN without edge-based aggregation, can be written as:
\begin{equation} \label{mlpenergy}
E(H_{MLP}^{(1)}) = tr\left( H_{MLP}^{(1) \top }\widetilde{L} H_{MLP}^{(1)} \right).
\end{equation}

Substituting the expression for the GCN output from Eq.(\ref{gcn}) into Eq.(\ref{gcnenergy}), and for simplicity omitting the effect of the activation function to focus on the influence of the graph structure on Dirichlet energy, we have:

\begin{align} \label{theproof_1}
&tr\left( H_{GCN}^{(1) \top }\widetilde{L} H_{GCN}^{(1)} \right) \notag \\
=& tr\left(\left(\hat{A}X W^{(1)}\right)^{\top} \widetilde{L} \left(\hat{A} X W^{(1)}\right) \right)  \notag \\
=& tr\left(W^{(1)\top} X^{\top} \hat{A}^{\top} \widetilde{L} \hat{A} X W^{(1)} \right)  \notag \\
\le& tr\left(W^{(1)\top} X^{\top} \widetilde{L} X W^{(1)} \right) \mu_{\max} \left(\hat{A}^{\top} \hat{A} \right)  \notag \\
=& \left \| \hat{A} \right \|_2^2 tr\left(\left(X W^{(1)}\right)^{\top} \widetilde{L} \left(X W^{(1)}\right) \right)  \notag \\
=& \left \| \hat{A} \right \|_2^2 tr\left( H_{MLP}^{(1) \top }\widetilde{L} H_{MLP}^{(1)} \right)  \notag \\
=& \left \| \hat{A} \right \|_2^2 E(H_{MLP}^{(1)}),
\end{align}
where the equality in the second step follows from the cyclic property of the trace operator together with the standard transpose rules for matrix products. The inequality in the third step follows from the fact that for any symmetric positive semidefinite matrix $M$ and any matrix $B$, it holds that $tr(B^\top M B) \le \mu_{\max}(M)\, tr(B^\top B)$, where $\mu_{\max}(\cdot)$ denotes the largest eigenvalue. Here we can apply this inequality because the Laplacian $\widetilde{L} = I_n - \widetilde{D}^{-\frac{1}{2}}\widetilde{A}\widetilde{D}^{-\frac{1}{2}}$ is constructed from the adjacency matrix $\hat{A}$, so $\widetilde{L}$ and $\hat{A}$ commute, allowing us to rewrite $\hat{A}^{\top} \widetilde{L} \hat{A} = \hat{A}^{\top} \hat{A} \widetilde{L}$ and thus justify the use of the trace inequality. The fourth step is due to the definition of the spectral norm $\mu_{\max}(M^\top M)=\|M\|_2^2$ for any matrix $M$.
\end{proof}

\begin{Remark}
Although Theorem \ref{thm:gcn_mlp} is presented for GCN-style propagation, the result extends to a broader class of GNNs. The key ingredient in the proof is the linear propagation operator that governs
feature aggregation across nodes. For attention-based GNNs such as GATs or GTs, the normalized adjacency matrix $\hat{A}$ can be replaced by the learned attention weight matrix at the corresponding layer. In this case, the same inequality holds with $\|\hat{A}\|_2$ replaced by the spectral norm of the attention matrix $\|A^{(1)}\|_2$.
\end{Remark}

\begin{Remark}
Theorem \ref{thm:gcn_mlp} provides insight into the role of the graph structure in controlling smoothness. For GCNs, as the spectral norm of the normalized adjacency matrix satisfies $\|\hat{A}\|_2^2 \equiv 1$, the bound simplifies to the strict inequality $E(H_{GCN}^{(1)}) < E(H_{MLP}^{(1)})$, indicating an inherent reduction in Dirichlet energy through structural priors. Similar insights extend to attention-based architectures such as GATs and GTs. When the spectral norm of the attention matrix satisfies $\|A^{(1)}\|_2^2 \le 1$ under conditions such as uniform attention distributions or doubly stochastic symmetry, incorporating structural information likewise yields smoother representations and facilitates the fitting of spatial patterns in the training data. Moreover, the spectral norm of the propagation matrix quantitatively characterizes the extent to which graph structure regulates smoothness. A larger spectral norm implies that feature learning is more strongly influenced by structural propagation, causing the model to rely more heavily on graph-specific patterns. When the test graph differs from the training graph, such structural dependence increases the risk of overfitting.
\end{Remark}

\begin{figure*}[t!]
\centering
\subfigbottomskip=0.05pt
\subfigcapskip=-3pt

    \rotatebox{90}{\scriptsize~~~~~~~~~~~~~~~~~GT~~~~~~~~~~~~~~~~~~~~~~~~~~~~~~~~~~~GAT~~~~~~~~~~~~~~~~~~~~~~~~~~~~~~~~GCN}
\subfigure[Cora]
{
 	\begin{minipage}[b]{.3\linewidth}
        \centering
        \includegraphics[width=1.07\linewidth,height=3.9cm]{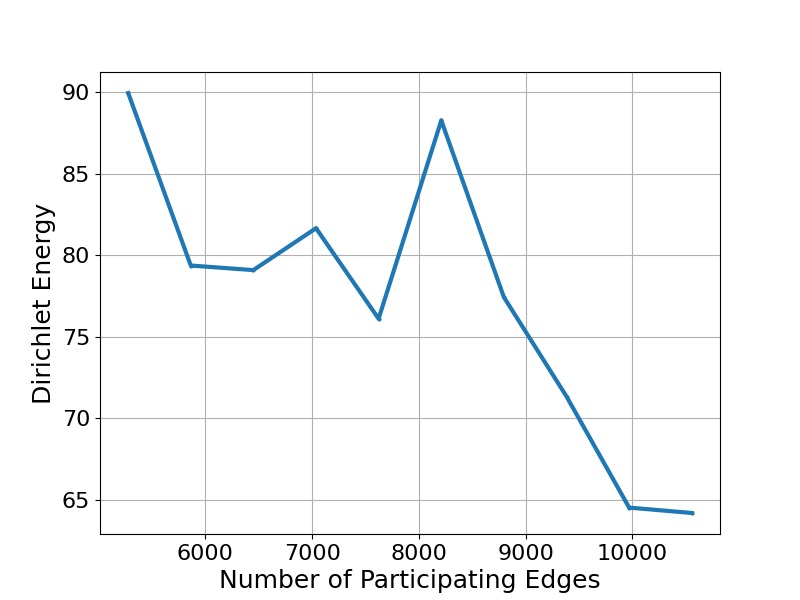}
        \includegraphics[width=1.07\linewidth,height=3.9cm]{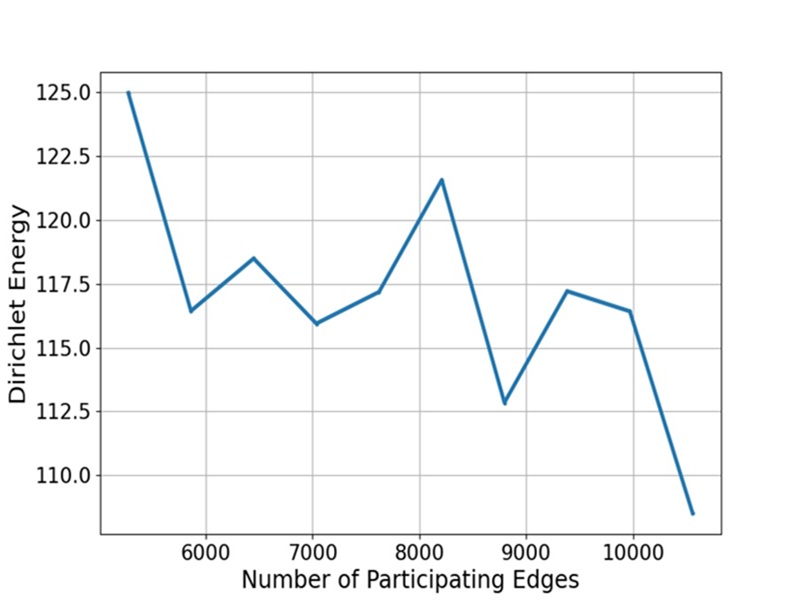}        \includegraphics[width=1.07\linewidth,height=3.9cm]{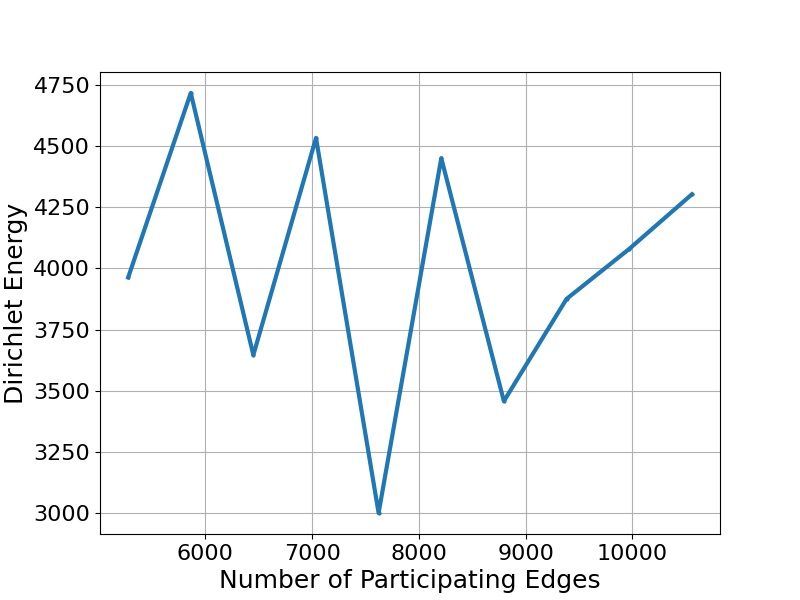}
    \end{minipage}
}
\subfigure[CiteSeer]
{
 	\begin{minipage}[b]{.3\linewidth}
        \centering
        \includegraphics[width=1.07\linewidth,height=3.9cm]{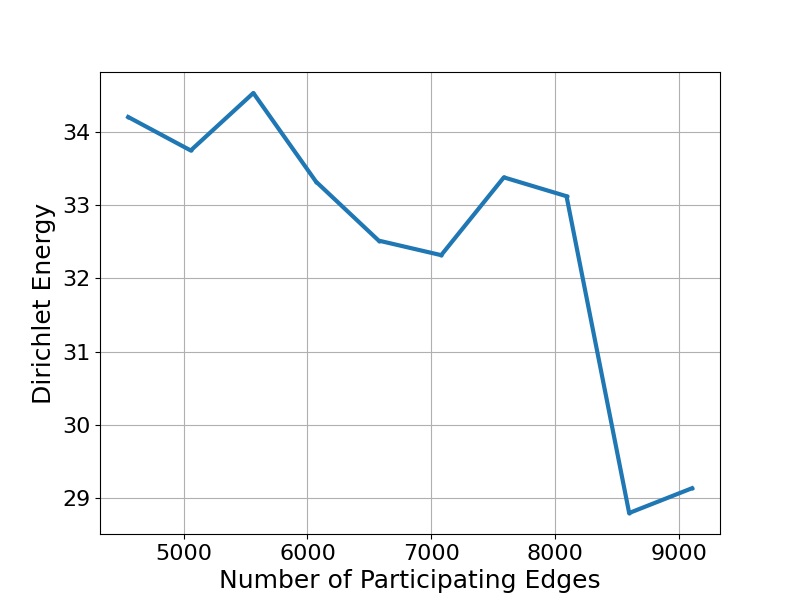}
        \includegraphics[width=1.07\linewidth,height=3.9cm]{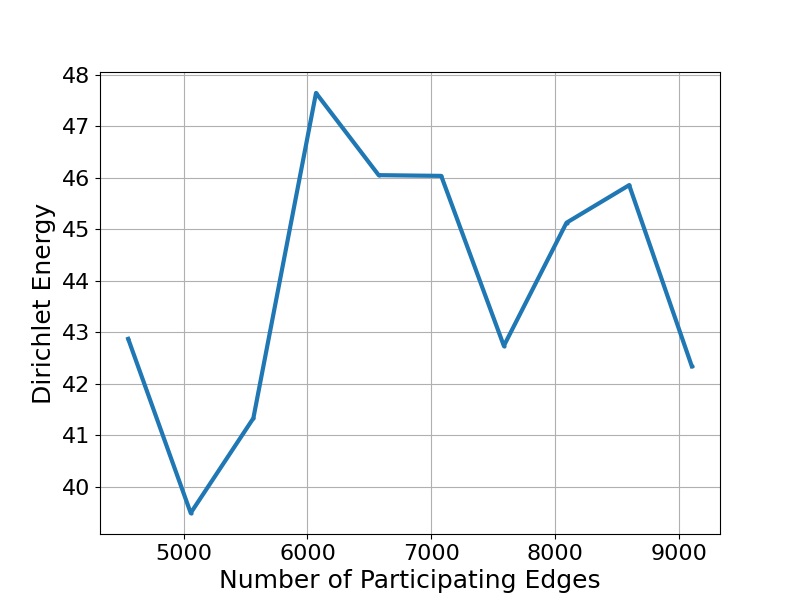}
        \includegraphics[width=1.07\linewidth,height=3.9cm]{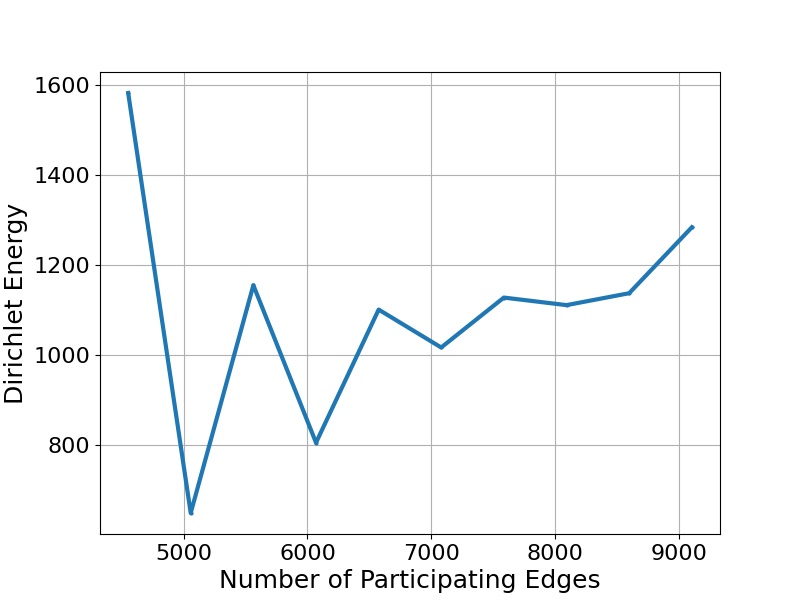}
    \end{minipage}
}
\subfigure[CS]
{
 	\begin{minipage}[b]{.3\linewidth}
        \centering
        \includegraphics[width=1.07\linewidth,height=3.9cm]{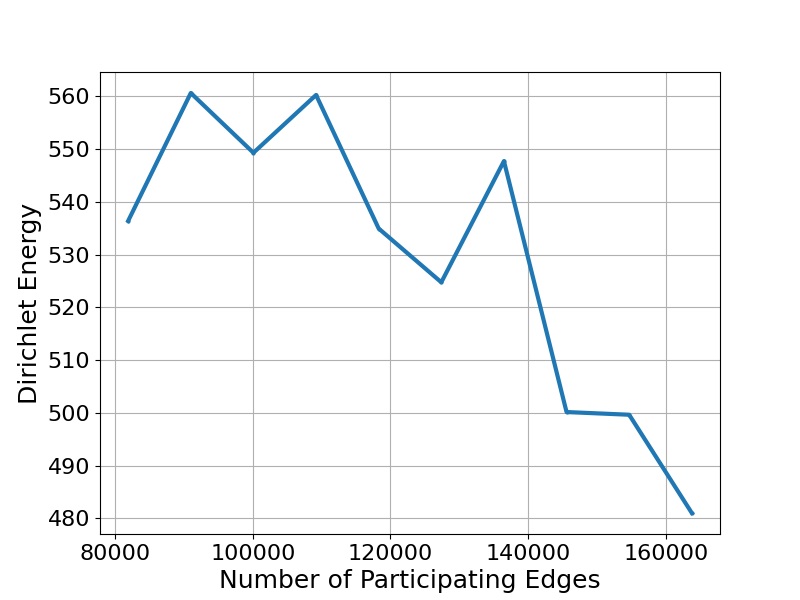}
        \includegraphics[width=1.07\linewidth,height=3.9cm]{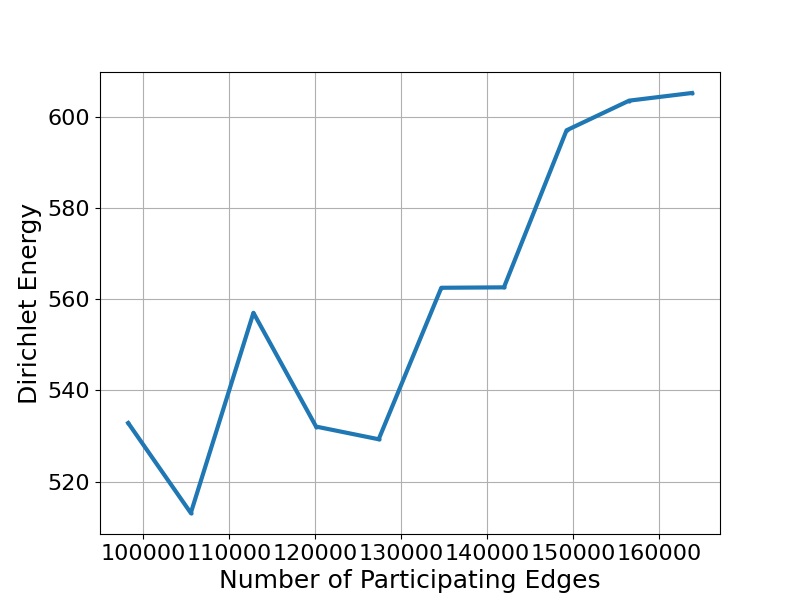}
        \includegraphics[width=1.07\linewidth,height=3.9cm]{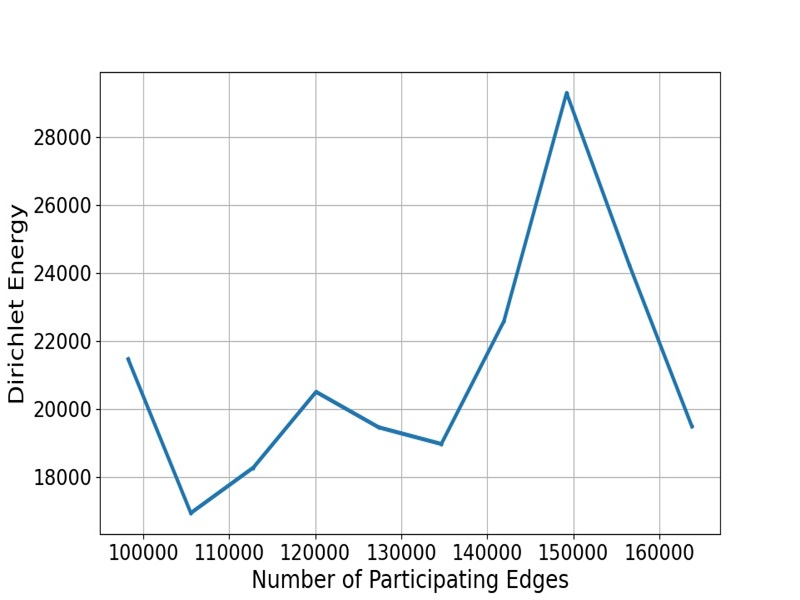}
    \end{minipage}
}
\caption{The relationship between Dirichlet energy and the number of effective edges in GCN, GAT and GT.}
\label{energy}
\end{figure*}

\begin{Remark}
To further validate our theoretical findings, we measure the Dirichlet energy of node representations under varying numbers of edges participating in the decision process. Specifically, we gradually increase the number of effective edges involved in aggregation and record the resulting Dirichlet energy for three representative GNN architectures: GCN, GAT, and GT. The results are presented in Figure \ref{energy}. The first row corresponds to GCN across three datasets, the second row to GAT, and the third row to GT. Across most settings, we observe a decreasing trend in Dirichlet energy as the number of decision-participating edges increases. These empirical observations corroborate our theoretical analysis, confirming that edge-based aggregation intrinsically reduces representation energy and induces a smoothing effect regardless of whether the aggregation weights are fixed or adaptive.
\end{Remark}

Taken together, these results reveal that structural information can drive overfitting in message-passing GNNs. Incorporating edges into aggregation systematically lowers the Dirichlet energy of learned representations, leading to smoother representations across the graph. This smoothing effect is quantitatively controlled by the spectral norm of the propagation (or attention) matrix, which acts as a contraction coefficient that determines the extent to which structural information reduces energy. While smoothing facilitates fitting on the training graph, it also increases the model's reliance on structural correlations induced by edges. As a result, when many edges, including potentially noisy ones, participate in representation formation, the learned model may become overly dependent on the specific structural patterns of the training graph, which increases the risk of overfitting. This observation highlights the effect of graph structure on the generalization ability of GNNs. In the next section, we formalize this effect through the notion of structural complexity, which characterizes the number of edges involved in model decisions and allows us to explicitly analyze its impact on the generalization of GNNs.

\section{Generalization of GNNs Based on Structural Complexity}
\label{generalization}

In the previous section, we showed that incorporating many edges into message passing may cause GNNs to rely heavily on structural correlations in the training graph, leading to structural overfitting. This observation suggests that the generalization performance of GNNs is closely related to how many edges effectively participate in the decision process. Therefore, in this section, we develop a generalization framework based on edge-driven structural complexity. Specifically, our analysis includes the following aspects: 1) We introduce the concept of structural complexity, which measures the extent of structural involvement in message passing through the number of effective edges. 2) Based on the concept of structural complexity, we establish generalization bounds based on Rademacher complexity for representative message-passing GNNs under fixed and attention-based aggregation mechanisms. Our analysis reveals an explicit dependence of the excess risk on the set of effective edges utilized during the propagation process. 3) Based on the theoretical insights obtained, we further propose a novel regularization strategy based on structure entropy, which controls the structural complexity to strike a balance between underfitting and overfitting, thereby improving the generalization performance of the model.

\subsection{Edge-Driven Structural Complexity}

In message-passing GNNs, node representations are updated by aggregating information along graph edges. As a result, the effective degrees of freedom of the model are governed not only by the number of trainable parameters, but also by the set of edges through which information is allowed to propagate. Edges that do not participate in message passing cannot influence node representations or predictions. This observation motivates an edge-driven notion of structural complexity, which characterizes model capacity from a structural perspective. Formally, each message-passing layer induces a set of effective edges, corresponding to non-zero aggregation coefficients. Only these edges are capable of transmitting information and coupling the predictions of neighboring nodes. The number of such edges provides a natural measure of the structural degrees of freedom available to the model.

For GCNs with fixed aggregation rules, message passing is governed by the normalized adjacency matrix $\hat{A}$. Information flows exclusively along edges for which $\hat{A}_{ij} \neq 0$, and the set of structurally active edges is fully characterized by the support of $\hat{A}$. In this case, the structural complexity of the model is quantified by the entrywise $\ell_0$ norm:
\begin{equation}
\eta = \left\lVert \hat{A} \right\rVert_0,
\end{equation}
which counts the number of edges involved in neighborhood aggregation.

For attention-based architectures, such as GATs and GTs, aggregation weights are learned and depend on node features. At each layer $k$, message passing is governed by an attention matrix $A^{(k)} = (\alpha_{vu}^{(k)})$. Due to the softmax normalization, attention coefficients are typically dense, meaning that most entries are strictly positive. However, in practice, many attention weights are negligibly small and contribute minimally to representation updates. To characterize the effective structural complexity of attention-based models, we introduce a threshold parameter $\tau > 0$ and define the thresholded attention matrix as:

\begin{equation} \label{thre_att}
\tilde{A}^{(k)}_{vu} =
\begin{cases}
\alpha_{vu}^{(k)}, & \text{if } \alpha_{vu}^{(k)} \ge \tau, \\
0, & \text{otherwise}.
\end{cases}
\end{equation}

We then define the effective structural complexity at layer $k$ as:

\begin{equation}
\eta^{(k)}_\tau = \left\lVert \tilde{A}^{(k)} \right\rVert_0.
\end{equation}

This quantity measures the number of edges with sufficiently large attention weights that meaningfully participate in message passing. Smaller values of $\eta^{(k)}_\tau$ correspond to sparser effective attention patterns and thus reduced structural complexity, as fewer edges substantially influence node representations.

Overall, edge-driven structural complexity quantifies the extent to which graph edges are utilized by a message-passing architecture to propagate information. This notion provides a unified and architecture-aware characterization of model capacity under a fixed graph structure, serving as the foundation for our subsequent generalization analysis.

\subsection{Generalization Error Bound}

To connect edge-driven structural complexity with generalization performance, we analyze the complexity of message-passing hypothesis classes under a fixed graph structure. Since node predictions on a fixed graph are inherently dependent, classical i.i.d. generalization analysis is not applicable. Instead, we characterize generalization performance using transductive Rademacher complexity \citep{oono2020optimization, esser2021learning}, which measures the ability of a hypothesis class to fit random node-wise labelings under a fixed graph structure.

\begin{Definition}(Rademacher Complexity)
Let $\mathcal{F}$ be a class of node-level predictors on a fixed graph $G=(V, E)$. The empirical transductive Rademacher complexity of $\mathcal{F}$ is defined as:
\begin{equation} \label{rand}
\widehat{\mathfrak{R}}(\mathcal{F}) = \mathbb{E}_{\zeta}
\left[ \sup_{f \in \mathcal{F}} \left\vert \frac{1}{m}\sum_{j=1}^m \zeta_j f(x_j) \right\vert  \right],
\end{equation}
where $\zeta_j \in \{\pm 1\}$ are independent Rademacher variables.
\end{Definition}

Rademacher complexity provides a fundamental tool for connecting the capacity of a hypothesis class to its generalization performance. In particular, classical learning theory establishes that the expected risk of a predictor can be bounded by its empirical risk plus a complexity term measured by the Rademacher complexity of the hypothesis class. The following lemma formalizes this relationship.

\begin{Lemma} \citep{bartlett2002rademacher}
\label{lem:rad_bound}
Let $\mathcal{F}$ be a hypothesis class and assume that the loss function $\ell(y,\cdot)$ is $L$-Lipschitz. Then for $0 < \delta < 1$, with probability at least $1 - \delta$ for all $f \in \mathcal{F}$ the following holds:
\begin{equation} \label{th3_0}
\epsilon(f) \le \hat{\epsilon}(f)
+ 2L\widehat{\mathfrak{R}}(\mathcal{F})
+ \sqrt{\frac{2\log(2/\delta)}{n}}.
\end{equation}
where $\widehat{\mathfrak{R}}(\mathcal{F})$ denotes the empirical transductive Rademacher complexity.
\end{Lemma}

Lemma \ref{lem:rad_bound} shows that bounding the Rademacher complexity of a hypothesis class directly leads to a bound on its generalization error. Therefore, to understand how graph structure affects the generalization performance of message-passing GNNs, it suffices to analyze the transductive Rademacher complexity of the corresponding hypothesis class.

We now explore the generalization properties of message-passing GNNs from the perspective of edge-driven structural complexity. To make the dependence on graph structure explicit, we derive generalization error bounds by directly analyzing the transductive Rademacher complexity of message-passing predictors under different aggregation mechanisms.

\subsubsection{Generalization under Fixed Aggregation}

We first analyze models with fixed aggregation mechanisms, exemplified by the GCN, where the propagation structure is predetermined by the graph topology. In these architectures, the structural complexity $\eta$ is static, acting as a rigid constraint on the model’s hypothesis capacity. To align with the empirical analysis in Section \ref{Difficulty} and maintain mathematical tractability, we derive the generalization bound for a two-layer GCN configuration. The resulting bound is formalized in the following theorem.

\begin{Theorem}
\label{thm:gcn}
Consider the hypothesis class $\mathcal{F}$ induced by GCNs under a fixed message-passing structure. Assume that the input node features satisfy $\left\lVert x_i \right\rVert_2 \le B$ for all nodes $i$, the weight matrix in the hidden layer satisfies $\left\lVert W^{(1)} \right\rVert_2 \le R_1$, the output weight vector satisfies $\left\lVert W^{(2)} \right\rVert_2 \le R_2$, and the loss function $\ell(y,\cdot)$ is $L$-Lipschitz continuous. Then for $0 < \delta < 1$, with probability at least $1 - \delta$ for all $f \in \mathcal{F}$ the following holds:
\begin{equation}
\epsilon(f) \le \hat{\epsilon}(f) + \frac{2LR_1R_2B \eta}{\sqrt{m}}  + \sqrt{\frac{2\log(2/\delta)}{n}},
\end{equation}
where $\eta = \left\lVert \hat{A} \right\rVert_0$ denotes the structural complexity of the model.
\end{Theorem}

\begin{proof}
Substituting the GCN second-layer output in Eq.(\ref{gcn}) into Eq.(\ref{rand}), we obtain:

\begin{align} \label{th3_1}
\widehat{\mathfrak{R}}(\mathcal{F}) &= \mathbb{E}_{\zeta}
\left[ \sup_{f \in \mathcal{F}} \left\vert \frac{1}{m}\sum_{j=1}^m \zeta_j f(x_j) \right\vert  \right] \notag \\
&= \mathbb{E}_{\zeta}
\left[ \frac{1}{m}\sup_{f \in \mathcal{F}} \left\vert \sum_{j=1}^m \zeta_j \sigma\left(\sum_{v \in \mathcal{N}(j)} [\hat{A}]_{jv} W^{(2)}h_v^{(1)} \right) \right\vert  \right] \notag \\
&\le \frac{1}{m} \mathbb{E}_{\zeta}
\left[ \sup_{f \in \mathcal{F}} \left\vert \sum_{j=1}^m \zeta_j \sum_{v \in \mathcal{N}(j)} [\hat{A}]_{jv} W^{(2)}h_v^{(1)}  \right\vert  \right] \notag \\
&\le \frac{1}{m} \mathbb{E}_{\zeta}
\left[ \sup_{f \in \mathcal{F}} \left\lVert \sum_{j=1}^m \zeta_j \sum_{v \in \mathcal{N}(j)} [\hat{A}]_{jv} h_v^{(1)} \right\rVert_2 \left\lVert W^{(2)} \right\rVert_2 \right] \notag \\
&\le \frac{R_2}{m} \mathbb{E}_{\zeta}
\left[ \sup_{f \in \mathcal{F}} \left\lVert \sum_{j=1}^m \zeta_j \sum_{v \in \mathcal{N}(j)} [\hat{A}]_{jv} h_v^{(1)} \right\rVert_2  \right],
\end{align}
where the inequality in the third step follows from the $1$-Lipschitz contraction property of the activation function $\sigma$, the fourth-step inequality is obtained by expressing the output as an inner product between the aggregated hidden representations and the output weight vector and applying the Cauchy–Schwarz inequality, and the final inequality follows from the norm constraint $\|W^{(2)}\|_2 \le R_2$. Then, by substituting the first-layer update formula into Eq.(\ref{th3_1}), we obtain:

\begin{align} \label{th3_2}
& \mathbb{E}_{\zeta}
\left[ \sup_{f \in \mathcal{F}} \left\lVert \sum_{j=1}^m \zeta_j \sum_{v \in \mathcal{N}(j)} [\hat{A}]_{jv} h_v^{(1)} \right\rVert_2  \right] \notag \\
=& \mathbb{E}_{\zeta}
\left[ \sup_{f \in \mathcal{F}} \left\lVert \sum_{j=1}^m \zeta_j \sum_{v \in \mathcal{N}(j)} [\hat{A}]_{jv} \sigma\left(\sum_{i \in \mathcal{N}(v)} [\hat{A}]_{vi} W^{(1)}x_i \right) \right\rVert_2  \right] \notag \\
\le& \mathbb{E}_{\zeta}
\left[ \sup_{f \in \mathcal{F}} \left\lVert \sum_{j=1}^m \zeta_j \sum_{v \in \mathcal{N}(j)} [\hat{A}]_{jv} \sum_{i \in \mathcal{N}(v)} [\hat{A}]_{vi} W^{(1)}x_i  \right\rVert_2  \right] \notag \\
\le& \mathbb{E}_{\zeta}
\left[ \sup_{f \in \mathcal{F}} \left\lVert \sum_{j=1}^m \zeta_j \sum_{v \in \mathcal{N}(j)} [\hat{A}]_{jv} \sum_{i \in \mathcal{N}(v)} [\hat{A}]_{vi} x_i \right\rVert_2 \left\lVert W^{(1)} \right\rVert_2  \right] \notag \\
\le& R_1 \mathbb{E}_{\zeta}
\left[ \sup_{f \in \mathcal{F}} \left( \left\lVert \sum_{j=1}^m \zeta_j \sum_{v \in \mathcal{N}(j)} [\hat{A}]_{jv} \sum_{i \in \mathcal{N}(v)} [\hat{A}]_{vi} x_i \right\rVert_2^2 \right)^{1/2} \right] \notag \\
\le& R_1 \left( \mathbb{E}_{\zeta}
 \left\lVert \sum_{j=1}^m \zeta_j \sum_{v \in \mathcal{N}(j)} [\hat{A}]_{jv} \sum_{i \in \mathcal{N}(v)} [\hat{A}]_{vi} x_i \right\rVert_2^2   \right)^{1/2},
\end{align}
where the third inequality uses the weight norm constraint $\|W^{(1)}\|_2 \le R_1$, and the last step follows from an application of Jensen's inequality. Since $\zeta$ is a zero-mean i.i.d. Rademacher sequence, we can write:

\begin{equation} \label{th3_3}
R_1 \left( \mathbb{E}_{\zeta}
 \left\lVert \sum_{j=1}^m \zeta_j \sum_{v \in \mathcal{N}(j)} [\hat{A}]_{jv} \sum_{i \in \mathcal{N}(v)} [\hat{A}]_{vi} x_i \right\rVert_2^2   \right)^{1/2} = R_1 \left( 
 \sum_{j=1}^m \left\lVert  \sum_{v \in \mathcal{N}(j)} [\hat{A}]_{jv} \sum_{i \in \mathcal{N}(v)} [\hat{A}]_{vi} x_i \right\rVert_2^2   \right)^{1/2}.
\end{equation}

Then, the double summation over neighbors can be expressed compactly as a matrix product, which can be written as:

\begin{equation}
\sum_{j=1}^m \left\lVert  \sum_{v \in \mathcal{N}(j)} [\hat{A}]_{jv} \sum_{i \in \mathcal{N}(v)} [\hat{A}]_{vi} x_i \right\rVert_2^2 = \sum_{j=1}^m \left\lVert   (\hat{A}^2 X)_j \right\rVert_2^2,
\end{equation}
where $(\hat{A}^2 X)_j$ denotes the $j$-th row of $\hat{A}^2 X$. Accordingly, the Rademacher complexity term in Eq.(\ref{th3_3}) can be upper bounded as:

\begin{align} \label{th3_4}
& R_1 \left( 
 \sum_{j=1}^m \left\lVert  \sum_{v \in \mathcal{N}(j)} [\hat{A}]_{jv} \sum_{i \in \mathcal{N}(v)} [\hat{A}]_{vi} x_i \right\rVert_2^2   \right)^{1/2} \notag \\
=& R_1 \left( 
 \sum_{j=1}^m \left\lVert   (\hat{A}^2 X)_j \right\rVert_2^2  \right)^{1/2} \notag \\
\le& R_1 \left( \left\lVert  \hat{A}^2 X \right\rVert_F^2  \right)^{1/2} \notag \\
\le& R_1 \left( \left\lVert  \hat{A} \right\rVert_2^4 \left\lVert  X\right\rVert_F^2  \right)^{1/2},
\end{align}
where the first inequality follows from the definition of the Frobenius norm, and the second inequality follows from the property $\left\lVert AB\right\rVert_F \le \left\lVert A \right\rVert_2 \left\lVert B \right\rVert_F$. Then, by the definition of the Frobenius norm and the assumption that the input node features satisfy $\left\lVert x_i \right\rVert_2 \le B$ for all nodes $i$, we have:

\begin{align} \label{th3_5}
\left\lVert  X\right\rVert_F^2 = \sum_{i=1}^m \left\lVert  x_i\right\rVert_2^2 \le \sum_{i=1}^m B^2 = m B^2.
\end{align}

Substituting this result into Eq.(\ref{th3_4}), we obtain the upper bound:

\begin{align} \label{th3_6}
& R_1 \left( \left\lVert  \hat{A} \right\rVert_2^4 \left\lVert  X\right\rVert_F^2  \right)^{1/2} \notag \\
\le& \sqrt{m} R_1 B \left\lVert  \hat{A} \right\rVert_2^2 \notag \\
\le& \sqrt{m} R_1 B \left\lVert  \hat{A} \right\rVert_\infty^2 \left\lVert  \hat{A} \right\rVert_0  \notag \\
\le& \sqrt{m} R_1 B \left\lVert  \hat{A} \right\rVert_0,
\end{align}
where the second inequality follows from the bound $\left\lVert  A \right\rVert_2 \le \left\lVert  A \right\rVert_\infty \sqrt{\left\lVert  A \right\rVert_0}$. Since the entries of $\hat{A}$ are bounded by $1$, the final inequality directly follows. Combining Eq.(\ref{th3_6}), Eq.(\ref{th3_1}) and Eq.(\ref{th3_0}), for $0 < \delta < 1$, with probability
at least $1 - \delta$ for all $f \in \mathcal{F}$ the following holds:

\begin{equation}
\epsilon(f) \le \hat{\epsilon}(f) + \frac{2LR_1R_2B \eta}{\sqrt{m}}  + \sqrt{\frac{2\log(2/\delta)}{n}}.
\end{equation}

This completes the proof.
\end{proof}

\begin{Remark}
Theorem \ref{thm:gcn} shows that for fixed aggregation GCNs, the generalization risk scales linearly with the structural complexity term $\eta = \left\lVert \hat{A} \right\rVert_0$. This indicates that for non-adaptive GNNs, the input graph topology acts as a hard structural constraint that intrinsically limits the model's hypothesis capacity, providing a natural protection against overfitting.
\end{Remark}

\subsubsection{Generalization under Attention-Based Aggregation}

Next, we extend our analysis to more flexible architectures that employ attention-based aggregation, such as GATs and GTs. Unlike the fixed case, these models utilize learned weights to adaptively select neighbors, implying that the attention-induced effective edges and the resulting structural complexity $\eta_\tau$ evolve during the training process. We first derive the generalization bound for the GAT, similarly focusing on a two-layer configuration to capture the multiplicative coupling of attention patterns across layers. The result is formalized in the following theorem.

\begin{Theorem}
\label{thm:gat}
Consider the hypothesis class $\mathcal{F}$ induced by GATs with attention-based message passing. Assume that the input node features satisfy $\left\lVert x_i \right\rVert_2 \le B$, the first-layer weight matrix satisfies $\left\lVert W^{(1)} \right\rVert_2 \le R_1$, the second-layer weight matrix satisfies $\left\lVert W^{(2)} \right\rVert_2 \le R_2$, and the loss function $\ell(y,\cdot)$ is $L$-Lipschitz continuous. Then for $0 < \delta < 1$, with probability at least $1 - \delta$ for all $f \in \mathcal{F}$ the following holds:
\begin{equation}
\epsilon(f) \le \hat{\epsilon}(f) + 2LR_1R_2B \sqrt{\frac{\eta_\tau^{(1)}\eta_\tau^{(2)}}{m}}  + \sqrt{\frac{2\log(2/\delta)}{n}},
\end{equation}
where $\eta^{(k)}_\tau = \left\lVert \tilde{A}^{(k)} \right\rVert_0$ for $k=1,2$ denotes the effective structural complexity of each layer in the model.
\end{Theorem}

\begin{proof}
Substituting the second-layer update of the GAT in Eq.(\ref{gat}) into the definition of the empirical transductive Rademacher complexity in Eq.(\ref{rand}), and applying the same Lipschitz contraction and Cauchy–Schwarz inequality as in the GCN case, we obtain:

\begin{align} \label{th4_1}
\widehat{\mathfrak{R}}(\mathcal{F}) &= \mathbb{E}_{\zeta}
\left[ \sup_{f \in \mathcal{F}} \left\vert \frac{1}{m}\sum_{j=1}^m \zeta_j f(x_j) \right\vert  \right] \notag \\
&= \mathbb{E}_{\zeta}
\left[ \frac{1}{m}\sup_{f \in \mathcal{F}} \left\vert \sum_{j=1}^m \zeta_j \sigma\left(\sum_{v \in \mathcal{N}(j)} \alpha_{jv}^{(2)}W^{(2)}h_v^{(1)} \right) \right\vert  \right] \notag \\
&\le \frac{1}{m} \mathbb{E}_{\zeta}
\left[ \sup_{f \in \mathcal{F}} \left\vert \sum_{j=1}^m \zeta_j \sum_{v \in \mathcal{N}(j)} \alpha_{jv}^{(2)} W^{(2)}h_v^{(1)}  \right\vert  \right] \notag \\
&\le \frac{1}{m} \mathbb{E}_{\zeta}
\left[ \sup_{f \in \mathcal{F}} \left\lVert \sum_{j=1}^m \zeta_j \sum_{v \in \mathcal{N}(j)} \alpha_{jv}^{(2)} h_v^{(1)} \right\rVert_2 \left\lVert W^{(2)} \right\rVert_2 \right] \notag \\
&\le \frac{R_2}{m} \mathbb{E}_{\zeta}
\left[ \sup_{f \in \mathcal{F}} \left\lVert \sum_{j=1}^m \zeta_j \sum_{v \in \mathcal{N}(j)} \alpha_{jv}^{(2)} h_v^{(1)} \right\rVert_2  \right],
\end{align}
where the final inequality follows from the norm constraint $\|W^{(2)}\|_2 \le R_2$. Applying the first-layer update into Eq.~(\ref{th4_1}) then yields:

\begin{align} \label{th4_2}
& \mathbb{E}_{\zeta}
\left[ \sup_{f \in \mathcal{F}} \left\lVert \sum_{j=1}^m \zeta_j \sum_{v \in \mathcal{N}(j)} \alpha_{jv}^{(2)} h_v^{(1)} \right\rVert_2  \right] \notag \\
=& \mathbb{E}_{\zeta}
\left[ \sup_{f \in \mathcal{F}} \left\lVert \sum_{j=1}^m \zeta_j \sum_{v \in \mathcal{N}(j)} \alpha_{jv}^{(2)} \sigma\left(\sum_{i \in \mathcal{N}(v)} \alpha_{vi}^{(1)} W^{(1)}x_i \right) \right\rVert_2  \right] \notag \\
\le& \mathbb{E}_{\zeta}
\left[ \sup_{f \in \mathcal{F}} \left\lVert \sum_{j=1}^m \zeta_j \sum_{v \in \mathcal{N}(j)} \alpha_{jv}^{(2)} \sum_{i \in \mathcal{N}(v)} \alpha_{vi}^{(1)} W^{(1)}x_i \right\rVert_2  \right] \notag \\
\le& \mathbb{E}_{\zeta}
\left[ \sup_{f \in \mathcal{F}} \left\lVert \sum_{j=1}^m \zeta_j \sum_{v \in \mathcal{N}(j)} \alpha_{jv}^{(2)} \sum_{i \in \mathcal{N}(v)} \alpha_{vi}^{(1)}x_i \right\rVert_2 \left\lVert W^{(1)} \right\rVert_2  \right] \notag \\
\le& R_1 \mathbb{E}_{\zeta}
\left[ \sup_{f \in \mathcal{F}} \left( \left\lVert \sum_{j=1}^m \zeta_j \sum_{v \in \mathcal{N}(j)} \alpha_{jv}^{(2)} \sum_{i \in \mathcal{N}(v)} \alpha_{vi}^{(1)} x_i \right\rVert_2^2 \right)^{1/2} \right] \notag \\
\le& R_1 \left( \mathbb{E}_{\zeta}
 \left\lVert \sum_{j=1}^m \zeta_j \sum_{v \in \mathcal{N}(j)} \alpha_{jv}^{(2)} \sum_{i \in \mathcal{N}(v)} \alpha_{vi}^{(1)} x_i \right\rVert_2^2   \right)^{1/2},
\end{align}
where the third inequality follows from the norm constraint $\|W^{(1)}\|_2 \le R_1$, and the last inequality is due to Jensen's inequality. Since $\zeta$ is a zero-mean i.i.d. Rademacher sequence, the expectation over $\zeta$ eliminates the cross terms, yielding:
\begin{equation} \label{th4_3}
R_1 \left( \mathbb{E}_{\zeta}
 \left\lVert \sum_{j=1}^m \zeta_j \sum_{v \in \mathcal{N}(j)} \alpha_{jv}^{(2)} \sum_{i \in \mathcal{N}(v)} \alpha_{vi}^{(1)} x_i \right\rVert_2^2   \right)^{1/2} = R_1 \left( 
 \sum_{j=1}^m \left\lVert  \sum_{v \in \mathcal{N}(j)} \alpha_{jv}^{(2)} \sum_{i \in \mathcal{N}(v)} \alpha_{vi}^{(1)} x_i \right\rVert_2^2   \right)^{1/2}.
\end{equation}

As in the GCN case, the nested summation can be rewritten in matrix form. Let $A^{(k)}$ denote the attention matrix at layer $k$, whose entries are given by $\alpha_{ij}^{(k)}$. Then we have:
\begin{equation}
\sum_{j=1}^m \left\lVert  \sum_{v \in \mathcal{N}(j)} \alpha_{jv}^{(2)} \sum_{i \in \mathcal{N}(v)} \alpha_{vi}^{(1)} x_i \right\rVert_2^2 = \sum_{j=1}^m \left\lVert   (A^{(2)} A^{(1)} X)_j \right\rVert_2^2.
\end{equation}

Therefore, the Rademacher complexity term can be bounded as:

\begin{align} \label{th4_4}
& R_1 \left( 
 \sum_{j=1}^m \left\lVert  \sum_{v \in \mathcal{N}(j)} \alpha_{jv}^{(2)} \sum_{i \in \mathcal{N}(v)} \alpha_{vi}^{(1)} x_i \right\rVert_2^2   \right)^{1/2} \notag \\
=& R_1 \left( 
 \sum_{j=1}^m \left\lVert   (A^{(2)} A^{(1)} X)_j \right\rVert_2^2  \right)^{1/2} \notag \\
\le& R_1 \left( \left\lVert  A^{(2)} A^{(1)} X \right\rVert_F^2  \right)^{1/2} \notag \\
=& R_1 \left\lVert  A^{(2)} A^{(1)} X \right\rVert_F  \notag \\
\le& R_1 \left\lVert  A^{(2)}\right\rVert_2 \left\lVert A^{(1)} \right\rVert_2 \left\lVert  X\right\rVert_F,
\end{align}
where the last inequality is obtained by applying $\left\lVert AB\right\rVert_F \le \left\lVert A \right\rVert_2 \left\lVert B \right\rVert_F$ and $\left\lVert AB\right\rVert_2 \le \left\lVert A \right\rVert_2 \left\lVert B \right\rVert_2$ sequentially. Using the assumption $\|x_i\|_2 \le B$ for all nodes $i$, we again have $\left\lVert  X\right\rVert_F \le \sqrt{m} B$. 

To express the bound in terms of effective structural complexity, we replace the attention matrices $A^{(k)}$ by their thresholded versions $\tilde{A}^{(k)}$ defined in Eq.(\ref{thre_att}). 
Applying the inequality $\|A\|_2 \le \|A\|_\infty \sqrt{\|A\|_0}$ to $\tilde{A}^{(k)}$, and noting that
$\|\tilde{A}^{(k)}\|_\infty \le 1$, we obtain

\begin{equation}
\left\lVert \tilde{A}^{(k)} \right\rVert_2 \le \sqrt{\|\tilde{A}^{(k)}\|_0}
= \sqrt{\eta_\tau^{(k)}},
\qquad k=1,2.
\end{equation}

Combining the above inequalities yields:

\begin{equation} \label{th4_5}
R_1 \left\lVert  \tilde{A}^{(2)}\right\rVert_2 \left\lVert \tilde{A}^{(1)} \right\rVert_2 \left\lVert  X\right\rVert_F
\le
\sqrt{m} R_1 B \sqrt{\eta_\tau^{(1)} \eta_\tau^{(2)}}.
\end{equation}

Substituting Eq.(\ref{th4_5}) into Eq.(\ref{th4_1}) and then into the standard transductive Rademacher bound in Eq.(\ref{th3_0}), we conclude that for $0<\delta<1$, with probability at least $1 - \delta$ for all $f \in \mathcal{F}$ the following holds:

\begin{equation}
\epsilon(f) \le \hat{\epsilon}(f) + 2LR_1R_2B \sqrt{\frac{ \eta_\tau^{(1)}\eta_\tau^{(2)}}{m}}  + \sqrt{\frac{2\log(2/\delta)}{n}}.
\end{equation}

This completes the proof.
\end{proof}

Having established the generalization bound for single-head attention-based GATs, we extend the analysis to the multi-head setting, which constitutes the canonical and widely adopted form of GAT architectures. Consider a two-layer GAT with $\Delta$ attention heads per layer. For each layer $k \in {1,2}$ and each head $\delta \in {1,\dots,\Delta}$, let $\tilde{A}^{(k,\delta_k)}$ denote the thresholded attention matrix of the $\delta_k$-th head at layer $k$. Message passing within each head is governed independently by its corresponding attention matrix, and the outputs of different heads are aggregated through concatenation or averaging.

From the perspective of edge-driven structural complexity, each attention head induces its own set of effective edges. Accordingly, we define the aggregated effective structural complexity at layer $k$ as:

\begin{equation}
\eta_{\mathrm{MH},\tau}^{(k)} = \sum_{\delta_k=1}^{\Delta} \left\lVert \tilde{A}^{(k,\delta_k)} \right\rVert_0,
\end{equation}
which counts the total number of edges with sufficiently large attention weights across all heads at that layer. This definition naturally extends the single-head notion of structural complexity and allows us to characterize the capacity of multi-head GATs in a unified manner. The following corollary establishes a generalization error bound for multi-head GATs, revealing how the excess risk scales with the aggregated edge-driven structural complexity across layers.

\begin{Corollary}
\label{cor:mh_gat}
Consider the hypothesis class $\mathcal{F}$ induced by a two-layer multi-head GAT with $\Delta$ attention heads per layer. 
Assume that the input node features satisfy $\left\lVert x_i \right\rVert_2 \le B$, the weight matrices satisfy $\left\lVert W^{(1,\delta_1)} \right\rVert_2 \le R_1$ and $\left\lVert W^{(2,\delta_2)} \right\rVert_2 \le R_2$, and the loss function $\ell(y,\cdot)$ is $L$-Lipschitz continuous.
Then for any $0 < \delta < 1$, with probability at least $1-\delta$ for all $f \in \mathcal{F}$ the following holds:

\begin{equation}
\epsilon(f) \le \hat{\epsilon}(f)
+ 2 L R_1 R_2 B 
\sqrt{
\frac{
\eta_{\mathrm{MH},\tau}^{(1)} \eta_{\mathrm{MH},\tau}^{(2)}
}{m \Delta}} + \sqrt{\frac{2\log(2/\delta)}{n}},
\end{equation}
where $\eta_{\mathrm{MH},\tau}^{(k)} = \sum_{\delta_k=1}^{\Delta} \left\lVert \tilde{A}^{(k,\delta_k)} \right\rVert_0$ for $k = 1, 2$ denotes the aggregated effective structural complexity at layer $k$ in the multi-head attention model.
\end{Corollary}

\begin{proof}
We sketch the proof by extending the single-head GAT analysis to the multi-head setting. For a two-layer GAT with $\Delta$ attention heads per layer, the hidden representation produced by the second layer can be written as an aggregation over all heads. The prediction function then takes the following form:
\begin{equation}
f(x_j) = \sigma \left( \frac{1}{\Delta} \sum_{\delta_2=1}^{\Delta} \sum_{v \in \mathcal{N}(j)} \alpha_{jv}^{(2,\delta_2)} W^{(2,\delta_2)} h_v^{(1)} \right),
\end{equation}
where the first-layer representations are given by:
\begin{equation}
h_v^{(1)} = \parallel_{\delta_1=1}^\Delta \sigma\!\left( \sum_{i \in \mathcal{N}(v)} \alpha_{vi}^{(1,\delta_1)} W^{(1,\delta_1)} x_i \right),
\qquad \delta_1 = 1,\dots,\Delta .
\end{equation}

Substituting the above multi-head expression into the definition of the empirical transductive Rademacher complexity in Eq.(\ref{rand}), and applying the Cauchy–Schwarz inequality together with the $1$-Lipschitz property of the activation function, we obtain:

\begin{align} \label{co1_1}
\widehat{\mathfrak{R}}(\mathcal{F}) &= \mathbb{E}_{\zeta}
\left[ \sup_{f \in \mathcal{F}} \left\vert \frac{1}{m}\sum_{j=1}^m \zeta_j f(x_j) \right\vert  \right] \notag \\
&= \mathbb{E}_{\zeta}
\left[ \frac{1}{m}\sup_{f \in \mathcal{F}} \left\vert \sum_{j=1}^m \zeta_j \sigma \left( \frac{1}{\Delta} \sum_{\delta_2=1}^{\Delta} \sum_{v \in \mathcal{N}(j)} \alpha_{jv}^{(2,\delta_2)} W^{(2,\delta_2)} h_v^{(1)} \right) \right\vert  \right] \notag \\
&\le \frac{1}{m \Delta} \mathbb{E}_{\zeta}
\left[ \sup_{f \in \mathcal{F}} \left\vert \sum_{j=1}^m \zeta_j \sum_{\delta_2=1}^{\Delta} \sum_{v \in \mathcal{N}(j)} \alpha_{jv}^{(2,\delta_2)} W^{(2,\delta_2)} h_v^{(1)}  \right\vert  \right] \notag \\
&\le \frac{1}{m \Delta} \mathbb{E}_{\zeta}
\left[ \sup_{f \in \mathcal{F}} \sum_{\delta_2=1}^{\Delta} \left\lVert \sum_{j=1}^m \zeta_j \sum_{v \in \mathcal{N}(j)} \alpha_{jv}^{(2,\delta_2)} h_v^{(1)} \right\rVert_2 \left\lVert W^{(2,\delta_2)} \right\rVert_2 \right] \notag \\
&\le \frac{R_2}{m \Delta} \mathbb{E}_{\zeta}
\left[ \sup_{f \in \mathcal{F}} \sum_{\delta_2=1}^{\Delta} \left\lVert \sum_{j=1}^m \zeta_j \sum_{v \in \mathcal{N}(j)} \alpha_{jv}^{(2,\delta_2)} h_v^{(1)} \right\rVert_2  \right] \notag \\
&\le \frac{R_2}{m \Delta} \mathbb{E}_{\zeta}
\left[ \sup_{f \in \mathcal{F}} \sqrt{\Delta} \left(  \sum_{\delta_2=1}^{\Delta} \left\lVert \sum_{j=1}^m \zeta_j \sum_{v \in \mathcal{N}(j)} \alpha_{jv}^{(2,\delta_2)} h_v^{(1)} \right\rVert_2^2 \right)^{1/2} \right] \notag \\
&= \frac{R_2}{m \sqrt{\Delta}} \mathbb{E}_{\zeta}
\left[ \sup_{f \in \mathcal{F}}  \left(  \sum_{\delta_2=1}^{\Delta} \left\lVert \sum_{j=1}^m \zeta_j \sum_{v \in \mathcal{N}(j)} \alpha_{jv}^{(2,\delta_2)} h_v^{(1)} \right\rVert_2^2 \right)^{1/2} \right],
\end{align}
where the third inequality follows from the norm constraint $\|W^{(2,\delta_2)}\|_2 \le R_2$, and the fourth inequality follows from the Cauchy–Schwarz inequality applied over the head dimension. Applying the first-layer update expression to Eq.(\ref{co1_1}), we obtain:
\begin{align} \label{co1_2}
&\mathbb{E}_{\zeta}
\left[ \sup_{f \in \mathcal{F}} \left( \sum_{\delta_2=1}^{\Delta} \left\lVert \sum_{j=1}^m \zeta_j \sum_{v \in \mathcal{N}(j)} \alpha_{jv}^{(2,\delta_2)} h_v^{(1)} \right\rVert_2^2 \right)^{1/2} \right] \notag \\
=& \mathbb{E}_{\zeta}
\left[ \sup_{f \in \mathcal{F}} \left( \sum_{\delta_2=1}^{\Delta} \left\lVert \sum_{j=1}^m \zeta_j \sum_{v \in \mathcal{N}(j)} \alpha_{jv}^{(2,\delta_2)} \parallel_{\delta_1=1}^\Delta \sigma \left( \sum_{i \in \mathcal{N}(v)} \alpha_{vi}^{(1,\delta_1)} W^{(1,\delta_1)} x_i \right) \right\rVert_2^2 \right)^{1/2}  \right] \notag \\
\le& \mathbb{E}_{\zeta}
\left[ \sup_{f \in \mathcal{F}} \left( \sum_{\delta_2=1}^{\Delta} \left\lVert \sum_{j=1}^m \zeta_j \sum_{v \in \mathcal{N}(j)} \alpha_{jv}^{(2,\delta_2)} \parallel_{\delta_1=1}^\Delta \sum_{i \in \mathcal{N}(v)} \alpha_{vi}^{(1,\delta_1)} W^{(1,\delta_1)} x_i \right\rVert_2^2 \right)^{1/2} \right] \notag \\
=& \mathbb{E}_{\zeta}
\left[ \sup_{f \in \mathcal{F}} \left( \sum_{\delta_2=1}^{\Delta} \sum_{\delta_1=1}^{\Delta} \left\lVert \sum_{j=1}^m \zeta_j \sum_{v \in \mathcal{N}(j)} \alpha_{jv}^{(2,\delta_2)} \sum_{i \in \mathcal{N}(v)} \alpha_{vi}^{(1,\delta_1)} W^{(1,\delta_1)} x_i \right\rVert_2^2 \right)^{1/2} \right] \notag \\
\le& \mathbb{E}_{\zeta}
\left[ \sup_{f \in \mathcal{F}} \left( \sum_{\delta_2=1}^{\Delta} \sum_{\delta_1=1}^{\Delta} \left\lVert \sum_{j=1}^m \zeta_j \sum_{v \in \mathcal{N}(j)} \alpha_{jv}^{(2,\delta_2)} \sum_{i \in \mathcal{N}(v)} \alpha_{vi}^{(1,\delta_1)} x_i \right\rVert_2^2 \left\lVert W^{(1,\delta_1)} \right\rVert_2^2 \right)^{1/2} \right] \notag \\
\le& R_1 \mathbb{E}_{\zeta}
\left[ \left( \sum_{\delta_2=1}^{\Delta} \sum_{\delta_1=1}^{\Delta} \left\lVert \sum_{j=1}^m \zeta_j \sum_{v \in \mathcal{N}(j)} \alpha_{jv}^{(2,\delta_2)} \sum_{i \in \mathcal{N}(v)} \alpha_{vi}^{(1,\delta_1)} x_i \right\rVert_2^2 \right)^{1/2} \right],
\end{align}
where the last inequality is obtained by applying the spectral norm constraint $\|W^{(1,\delta_1)}\|_2 \le R_1$. Since $\{\zeta_j\}_{j=1}^m$ are independent Rademacher variables with zero mean, the expectation over $\zeta$ eliminates all cross terms across different nodes. Therefore, we have:
\begin{align} \label{co1_3}
& R_1 \mathbb{E}_{\zeta}
\left[
 \left( \sum_{\delta_2=1}^{\Delta} \sum_{\delta_1=1}^{\Delta}
 \left\lVert
 \sum_{j=1}^m \zeta_j \sum_{v \in \mathcal{N}(j)}
 \alpha_{jv}^{(2,\delta_2)}
 \sum_{i \in \mathcal{N}(v)}
 \alpha_{vi}^{(1,\delta_1)} x_i
 \right\rVert_2^2
 \right)^{1/2}
\right] \notag \\
=&
R_1
\left(
 \sum_{\delta_2=1}^{\Delta} \sum_{\delta_1=1}^{\Delta}
 \sum_{j=1}^m
 \left\lVert
 \sum_{v \in \mathcal{N}(j)}
 \alpha_{jv}^{(2,\delta_2)}
 \sum_{i \in \mathcal{N}(v)}
 \alpha_{vi}^{(1,\delta_1)} x_i
 \right\rVert_2^2
\right)^{1/2}.
\end{align}

As in the single-head case, the nested summation over neighbors can be expressed in matrix form. We can write:
\begin{equation}
\sum_{j=1}^m \left\lVert  \sum_{v \in \mathcal{N}(j)} \alpha_{jv}^{(2,\delta_2)}  \sum_{i \in \mathcal{N}(v)} \alpha_{vi}^{(1,\delta_1)} x_i \right\rVert_2^2
= \sum_{j=1}^m \left\lVert   (A^{(2,\delta_2)} A^{(1,\delta_1)} X)_j \right\rVert_2^2.
\end{equation}

Substituting this expression into Eq.(\ref{co1_3}) yields:
\begin{align} \label{co1_4}
& R_1
\left(
 \sum_{\delta_2=1}^{\Delta} \sum_{\delta_1=1}^{\Delta}
 \sum_{j=1}^m
 \left\lVert
 \sum_{v \in \mathcal{N}(j)}
 \alpha_{jv}^{(2,\delta_2)}
 \sum_{i \in \mathcal{N}(v)}
 \alpha_{vi}^{(1,\delta_1)} x_i
 \right\rVert_2^2
\right)^{1/2} \notag \\
=& R_1
\left(
 \sum_{\delta_2=1}^{\Delta} \sum_{\delta_1=1}^{\Delta} \sum_{j=1}^m
 \left\lVert
  (A^{(2,\delta_2)} A^{(1,\delta_1)} X)_j
 \right\rVert_2^2
\right)^{1/2} \notag \\
\le&
R_1 \left( \sum_{\delta_2=1}^{\Delta} \sum_{\delta_1=1}^{\Delta}
 \left\lVert  A^{(2,\delta_2)} A^{(1,\delta_1)} X \right\rVert_F^2\right)^{1/2} \notag \\
\le&
R_1
\left(
 \sum_{\delta_2=1}^{\Delta} \sum_{\delta_1=1}^{\Delta}
 \left\lVert
 A^{(2,\delta_2)}
 \right\rVert_2^2
 \left\lVert
 A^{(1,\delta_1)}
 \right\rVert_2^2
 \left\lVert
 X
 \right\rVert_F^2
\right)^{1/2},
\end{align}
where we used the inequalities $\left\lVert AB \right\rVert_F \le \left\lVert A \right\rVert_2 \left\lVert B \right\rVert_F$ and $\left\lVert AB\right\rVert_2 \le \left\lVert A \right\rVert_2 \left\lVert B \right\rVert_2$. Using the assumption $\|x_i\|_2 \le B$ for all nodes $i$, we have $\|X\|_F \le \sqrt{m} B$. 

To express the bound in terms of effective structural complexity, we replace each attention matrix $A^{(k,\delta_k)}$
by its thresholded counterpart $\tilde{A}^{(k,\delta_k)}$ defined in Eq.~(\ref{thre_att}). Applying the inequality $\|A\|_2 \le \|A\|_\infty \sqrt{\|A\|_0}$ to $\tilde{A}^{(k,\delta_k)}$, and noting that
$\|\tilde{A}^{(k,\delta_k)}\|_\infty \le 1$, we obtain:
\begin{equation}
\left\lVert \tilde{A}^{(k,\delta_k)} \right\rVert_2
\le \sqrt{\left\lVert \tilde{A}^{(k,\delta_k)} \right\rVert_0},
\qquad k=1,2.
\end{equation}

Substituting these bounds into Eq.~(\ref{co1_4}) gives:
\begin{align} \label{co1_5}
& R_1
\left(
 \sum_{\delta_2=1}^{\Delta} \sum_{\delta_1=1}^{\Delta}
 \left\lVert
 \tilde{A}^{(2,\delta_2)}
 \right\rVert_2^2
 \left\lVert
 \tilde{A}^{(1,\delta_1)}
 \right\rVert_2^2
 \left\lVert
 X
 \right\rVert_F^2
\right)^{1/2} \notag \\
\le&
\sqrt{m} R_1 B
\left(
 \sum_{\delta_2=1}^{\Delta} \sum_{\delta_1=1}^{\Delta}
 \left\lVert
 \tilde{A}^{(2,\delta_2)}
 \right\rVert_0
 \left\lVert
 \tilde{A}^{(1,\delta_1)}
 \right\rVert_0
\right)^{1/2}.
\end{align}

Noting that $\sum_{\delta_2=1}^{\Delta} \sum_{\delta_1=1}^{\Delta}
\left\lVert \tilde{A}^{(2,\delta_2)} \right\rVert_0 \left\lVert \tilde{A}^{(1,\delta_1)} \right\rVert_0 = \eta_{\mathrm{MH},\tau}^{(2)} \eta_{\mathrm{MH},\tau}^{(1)}$, we obtain:

\begin{equation} \label{co1_6}
\widehat{\mathfrak{R}}(\mathcal{F})
\le
\frac{R_1 R_2 B}{\sqrt{m \Delta}}
\sqrt{
\eta_{\mathrm{MH},\tau}^{(1)} \eta_{\mathrm{MH},\tau}^{(2)}
}.
\end{equation}

Finally, substituting Eq.(\ref{co1_6}) into the standard transductive Rademacher bound in Eq.(\ref{th3_0}), we conclude that for $0<\delta<1$, with probability at least $1 - \delta$ for all $f \in \mathcal{F}$ the following holds:
\begin{equation}
\epsilon(f) \le \hat{\epsilon}(f)
+ 2 L R_1 R_2 B
\sqrt{
\frac{
\eta_{\mathrm{MH},\tau}^{(1)} \eta_{\mathrm{MH},\tau}^{(2)}
}{m \Delta}
}
+ \sqrt{\frac{2\log(2/\delta)}{n}} .
\end{equation}

This completes the proof.
\end{proof}

\begin{Remark}
For attention-based GATs, Theorem \ref{thm:gat} shows that the generalization bound depends on the geometric interaction $\sqrt{\eta_\tau^{(1)}\eta_\tau^{(2)}}$, which reflects a multiplicative coupling of effective edges across layers. This suggests that the structural complexity of GATs can be amplified through successive message-passing steps. Furthermore, as shown in Corollary \ref{cor:mh_gat}, in the multi-head setting, this structural contribution is normalized by the number of heads $\Delta$. This indicates that while the total number of active edges across all heads still governs the model's capacity, the multi-head mechanism effectively distributes and averages this complexity, leading to more robust generalization.
\end{Remark}

Building upon the analysis of attention-based GATs, we proceed to extend the edge-driven structural complexity framework to GTs. Due to the increased architectural complexity introduced by global self-attention, residual connections, and normalization layers, we focus on a single-layer GT for clarity of exposition. This simplification allows us to isolate the impact of attention-driven message passing on generalization performance, while preserving the essential structural characteristics of GTs. The extension to multi-layer architectures follows analogously by iterating the same arguments across layers. To control the effect of normalization in the generalization analysis, we impose the following assumption.

\begin{Assumption}
\label{ass:norm}
The normalization operator $\mathrm{Norm}(\cdot)$ is $\Pi_{Norm}$-Lipschitz, such that for all $a,b \in \mathbb{R}^{\widetilde{d}}$, we have:
\begin{equation}
\|\mathrm{Norm}(a) - \mathrm{Norm}(b)\|_2 \le \Pi_{Norm} \|a - b\|_2 .
\end{equation}
\end{Assumption}

Under this setting, we state the generalization error bound for single-layer GTs, which reveals an explicit dependence on attention-driven structural complexity, as formalized in the following theorem.

\begin{Theorem}
\label{thm:gt}
Consider the hypothesis class $\mathcal{F}$ induced by a single-layer GT. Assume that the input node features satisfy $\left\lVert x_i \right\rVert_2 \le B$, the projection matrices satisfy $\left\lVert V^{(1,\delta_1)} \right\rVert_2 \le R_0$, $\left\lVert W_1^{(1)} \right\rVert_2 \le R_1$ and $\left\lVert W_2^{(1)} \right\rVert_2 \le R_2$, the output projection matrix satisfies $\left\lVert O^{1} \right\rVert_2 \le \widetilde{R}$, the loss function $\ell(y,\cdot)$ is $L$-Lipschitz continuous, and Assumption \ref{ass:norm} holds.
Then for any $0<\delta<1$, with probability at least $1-\delta$ for all $f \in \mathcal{F}$ the following holds:

\begin{equation}  \label{gtth_0}
\epsilon(f) \le \hat{\epsilon}(f)
+ \frac{8 \Pi_{Norm}^2 L B (1+R_2R_1)}{\sqrt{m}}
 \left( 1+R_0 \widetilde{R} \sqrt{\Delta \eta_{\mathrm{MHT},\tau}^{(1)}} \right)
+ \sqrt{\frac{2\log(2/\delta)}{n}},
\end{equation}
where $\eta_{\mathrm{MHT},\tau}^{(1)} = \sum_{\delta_1=1}^{\Delta} \left\lVert \tilde{A}^{(1,\delta_1)} \right\rVert_0$ denotes the aggregated attention-driven structural complexity of the single-layer multi-head GT.
\end{Theorem}

\begin{proof}
Substituting the GT update rule in Eq.(\ref{tra3}) into the definition of the empirical transductive Rademacher complexity in Eq.~(\ref{rand}), we obtain:

\begin{align} \label{gtth_1}
\widehat{\mathfrak{R}}(\mathcal{F}) &= \mathbb{E}_{\zeta}
\left[ \sup_{f \in \mathcal{F}} \left\vert \frac{1}{m}\sum_{j=1}^m \zeta_j f(x_j) \right\vert  \right] \notag \\
&= \mathbb{E}_{\zeta}
\left[ \frac{1}{m}\sup_{f \in \mathcal{F}} \left\vert \sum_{j=1}^m \zeta_j \mathrm{Norm} \left(\widetilde{h}_j^{(1)} + \hat{\hat{h}}_j^{(1)} \right) \right\vert  \right] \notag \\
&\le \frac{2 \Pi_{Norm}}{m} \mathbb{E}_{\zeta}
\left[ \sup_{f \in \mathcal{F}} \left\lVert \sum_{j=1}^m \zeta_j \left(\widetilde{h}_j^{(1)} + \hat{\hat{h}}_j^{(1)} \right)  \right\rVert_2  \right],
\end{align}
where the last inequality follows from the $\Pi_{Norm}$-Lipschitz property of the normalization operator in Assumption \ref{ass:norm}, together with the vector contraction inequality. Then, substituting the feed-forward update in Eq.(\ref{tra2}) into Eq.(\ref{gtth_1}) yields:

\begin{align} \label{gtth_2}
&\mathbb{E}_{\zeta}
\left[ \sup_{f \in \mathcal{F}} \left\lVert \sum_{j=1}^m \zeta_j \left(\widetilde{h}_j^{(1)} + \hat{\hat{h}}_j^{(1)} \right)  \right\rVert_2  \right] \notag \\
=& \mathbb{E}_{\zeta}
\left[ \sup_{f \in \mathcal{F}} \left\lVert \sum_{j=1}^m \zeta_j \left(\widetilde{h}_j^{(1)} +  W_2^{(1)} \sigma \left( W_1^{(1)} \widetilde{h}_j^{(1)} \right) \right)  \right\rVert_2  \right] \notag \\
\le& \mathbb{E}_{\zeta}
\left[ \sup_{f \in \mathcal{F}} \left\lVert \sum_{j=1}^m \zeta_j \widetilde{h}_j^{(1)}  \right\rVert_2  \right] + \mathbb{E}_{\zeta}
\left[ \sup_{f \in \mathcal{F}} \left\lVert \sum_{j=1}^m \zeta_j  W_2^{(1)} \sigma \left( W_1^{(1)} \widetilde{h}_j^{(1)} \right)  \right\rVert_2  \right] \notag \\
\le& \mathbb{E}_{\zeta}
\left[ \sup_{f \in \mathcal{F}} \left\lVert \sum_{j=1}^m \zeta_j \widetilde{h}_j^{(1)}  \right\rVert_2  \right] + \mathbb{E}_{\zeta}
\left[ \sup_{f \in \mathcal{F}} \left\lVert \sum_{j=1}^m \zeta_j W_1^{(1)} \widetilde{h}_j^{(1)} \right\rVert_2 \left\lVert W_2^{(1)} \right\rVert_2  \right] \notag \\
\le& \mathbb{E}_{\zeta}
\left[ \sup_{f \in \mathcal{F}} \left\lVert \sum_{j=1}^m \zeta_j \widetilde{h}_j^{(1)}  \right\rVert_2  \right] + R_2 \mathbb{E}_{\zeta}
\left[ \sup_{f \in \mathcal{F}} \left\lVert \sum_{j=1}^m \zeta_j \widetilde{h}_j^{(1)} \right\rVert_2 \left\lVert W_1^{(1)} \right\rVert_2  \right] \notag \\
\le& \mathbb{E}_{\zeta}
\left[ \sup_{f \in \mathcal{F}} \left\lVert \sum_{j=1}^m \zeta_j \widetilde{h}_j^{(1)}  \right\rVert_2  \right] + R_2 R_1 \mathbb{E}_{\zeta}
\left[ \sup_{f \in \mathcal{F}} \left\lVert \sum_{j=1}^m \zeta_j \widetilde{h}_j^{(1)} \right\rVert_2  \right] \notag \\
=& (1+R_2 R_1) \mathbb{E}_{\zeta}
\left[ \sup_{f \in \mathcal{F}} \left\lVert \sum_{j=1}^m \zeta_j \widetilde{h}_j^{(1)}  \right\rVert_2  \right],
\end{align}
where the first inequality follows from the triangle inequality of the Euclidean norm, the third inequality applies the spectral norm constraint $\left\lVert W_2^{(1)} \right\rVert_2 \le R_2$, and the fourth inequality applies the sub-multiplicativity of the operator norm together with the parameter constraint $\left\lVert W_1^{(1)} \right\rVert_2 \le R_1$. Substituting the attention-based aggregation in Eq.(\ref{tra1}) into the above expression, while omitting the positional encoding term for clarity since it introduces only an additive bounded component and does not influence the structural complexity term, we obtain:

\begin{align} \label{gtth_3}
&\mathbb{E}_{\zeta}
\left[ \sup_{f \in \mathcal{F}} \left\lVert \sum_{j=1}^m \zeta_j \widetilde{h}_j^{(1)}  \right\rVert_2  \right] \notag \\
=& \mathbb{E}_{\zeta}
\left[ \sup_{f \in \mathcal{F}} \left\lVert \sum_{j=1}^m \zeta_j \mathrm{Norm} \left(x_j + \hat{h}_j^{(1)} \right)  \right\rVert_2  \right] \notag \\
\le& 2 \Pi_{Norm} \mathbb{E}_{\zeta}
\left[ \sup_{f \in \mathcal{F}} \left\lVert \sum_{j=1}^m \zeta_j \left(x_j + \hat{h}_j^{(1)} \right)  \right\rVert_2  \right] \notag \\
\le& 2 \Pi_{Norm} \left( \mathbb{E}_{\zeta}
\left[ \sup_{f \in \mathcal{F}} \left\lVert \sum_{j=1}^m \zeta_j x_j \right\rVert_2  \right] + \mathbb{E}_{\zeta}
\left[ \sup_{f \in \mathcal{F}} \left\lVert \sum_{j=1}^m \zeta_j \hat{h}_j^{(1)} \right\rVert_2  \right] \right).
\end{align}

For the second term, recalling the attention-based aggregation in Eq.(\ref{gt}), we have:

\begin{align} \label{gtth_4}
& \mathbb{E}_{\zeta}
\left[ \sup_{f \in \mathcal{F}} \left\lVert \sum_{j=1}^m \zeta_j \hat{h}_j^{(1)} \right\rVert_2  \right] \notag \\
=& \mathbb{E}_{\zeta}
\left[ \sup_{f \in \mathcal{F}} \left\lVert \sum_{j=1}^m \zeta_j \mathit{O}^{1} \mathop{\big\|}_{\delta_1=1}^{\Delta} \left(\sum_{v \in \mathcal{N}(j)} \alpha_{jv}^{(1, \delta_1)}V^{(1, \delta_1)}x_v \right) \right\rVert_2  \right] \notag \\
\le& \widetilde{R} \mathbb{E}_{\zeta}
\left[ \sup_{f \in \mathcal{F}} \left\lVert \sum_{j=1}^m \zeta_j \mathop{\big\|}_{\delta_1=1}^{\Delta} \left(\sum_{v \in \mathcal{N}(j)} \alpha_{jv}^{(1, \delta_1)}V^{(1, \delta_1)}x_v \right) \right\rVert_2  \right] \notag \\
\le& \widetilde{R} \mathbb{E}_{\zeta} 
\left[ \sup_{f \in \mathcal{F}} \sqrt{\Delta} \left( \sum_{\delta_1=1}^\Delta \left\lVert \sum_{j=1}^m \zeta_j \sum_{v \in \mathcal{N}(j)} \alpha_{jv}^{(1, \delta_1)}V^{(1, \delta_1)}x_v \right\rVert_2^2 \right)^{1/2}  \right] \notag \\
\le& \widetilde{R} \mathbb{E}_{\zeta} 
\left[ \sup_{f \in \mathcal{F}} \sqrt{\Delta} \left( \sum_{\delta_1=1}^\Delta \left\lVert \sum_{j=1}^m \zeta_j \sum_{v \in \mathcal{N}(j)} \alpha_{jv}^{(1, \delta_1)}x_v \right\rVert_2^2 \left\lVert V^{(1, \delta_1)} \right\rVert_2^2 \right)^{1/2}  \right] \notag \\
\le& \widetilde{R} R_0 \sqrt{\Delta} \mathbb{E}_{\zeta} 
\left[ \sup_{f \in \mathcal{F}} \left( \sum_{\delta_1=1}^\Delta \left\lVert \sum_{j=1}^m \zeta_j \sum_{v \in \mathcal{N}(j)} \alpha_{jv}^{(1, \delta_1)}x_v \right\rVert_2^2 \right)^{1/2}  \right],
\end{align}
where the first inequality follows from the spectral norm constraint $\left\lVert O^{1} \right\rVert_2 \le \widetilde{R}$. The last inequality uses the parameter constraint $\left\lVert V^{(1,\delta_1)} \right\rVert_2 \le R_0$. Since the Rademacher variables $\{\zeta_j\}_{j=1}^m$are independent with zero mean and unit variance, taking expectation over $\zeta$ eliminates all cross terms across different nodes. Consequently, we obtain:

\begin{align} \label{gtth_5}
& \widetilde{R} R_0 \sqrt{\Delta} \mathbb{E}_{\zeta} 
\left[ \sup_{f \in \mathcal{F}} \left( \sum_{\delta_1=1}^\Delta \left\lVert \sum_{j=1}^m \zeta_j \sum_{v \in \mathcal{N}(j)} \alpha_{jv}^{(1, \delta_1)}x_v \right\rVert_2^2 \right)^{1/2}  \right] \notag \\
=&
\widetilde{R} R_0 \sqrt{\Delta} \left( \sum_{\delta_1=1}^\Delta \sum_{j=1}^m \left\lVert  \sum_{v \in \mathcal{N}(j)} \alpha_{jv}^{(1, \delta_1)}x_v \right\rVert_2^2 \right)^{1/2}.
\end{align}

Similarly, for the first term in Eq.(\ref{gtth_3}), we can also obtain:

\begin{equation} \label{gtth_6}
\mathbb{E}_{\zeta}
\left[ \sup_{f \in \mathcal{F}} \left\lVert \sum_{j=1}^m \zeta_j x_j \right\rVert_2  \right] \le \left( \sum_{j=1}^m \left\lVert x_j \right\rVert_2^2 \right)^{1/2}.
\end{equation}

Combining Eq.(\ref{gtth_3}), Eq.(\ref{gtth_5}), and Eq.(\ref{gtth_6}), together with the assumption that $\left\lVert x_i \right\rVert_2 \le B$, yields:

\begin{align} \label{gtth_7}
&\mathbb{E}_{\zeta}
\left[ \sup_{f \in \mathcal{F}} \left\lVert \sum_{j=1}^m \zeta_j \widetilde{h}_j^{(1)}  \right\rVert_2  \right] \notag \\
\le& 2 \Pi_{Norm} \left(\left( \sum_{j=1}^m \left\lVert x_j \right\rVert_2^2 \right)^{1/2} + \widetilde{R} R_0 \sqrt{\Delta} \left( \sum_{\delta_1=1}^\Delta \sum_{j=1}^m \left\lVert  \sum_{v \in \mathcal{N}(j)} \alpha_{jv}^{(1, \delta_1)}x_v \right\rVert_2^2 \right)^{1/2} \right)  \notag \\
=& 2 \Pi_{Norm} \left( \left( \sum_{j=1}^m \left\lVert x_j \right\rVert_2^2 \right)^{1/2} + \widetilde{R} R_0 \sqrt{\Delta} \left( \sum_{\delta_1=1}^\Delta \sum_{j=1}^m \left\lVert  (A^{(1,\delta_1)} X)_j \right\rVert_2^2 \right)^{1/2} \right)   \notag \\
\le& 2 \Pi_{Norm} \left( \left( \sum_{j=1}^m \left\lVert x_j \right\rVert_2^2 \right)^{1/2} + \widetilde{R} R_0 \sqrt{\Delta} \left( \sum_{\delta_1=1}^\Delta  \left\lVert  A^{(1,\delta_1)} X \right\rVert_F^2 \right)^{1/2} \right) \notag \\
\le& 2 \Pi_{Norm} \left(  \left( \sum_{j=1}^m \left\lVert x_j \right\rVert_2^2 \right)^{1/2} + \widetilde{R} R_0 \sqrt{\Delta} \left( \sum_{\delta_1=1}^\Delta \left\lVert  A^{(1,\delta_1)} \right\rVert_2^2 \left\lVert  X \right\rVert_F^2 \right)^{1/2} \right)  \notag \\
\le& 2 \sqrt{m} \Pi_{Norm} B \left( 1 + \widetilde{R} R_0 \sqrt{\Delta} \left( \sum_{\delta_1=1}^\Delta \left\lVert  A^{(1,\delta_1)} \right\rVert_2^2 \right)^{1/2} \right).
\end{align}

Continuing from the last inequality, we further control the spectral norm of the thresholded attention matrices. To express the bound in terms of effective structural complexity, 
we replace each attention matrix $A^{(1,\delta_1)}$ 
by its thresholded counterpart $\tilde{A}^{(1,\delta_1)}$ 
defined in Eq.(\ref{thre_att}). Since the entries of $\tilde{A}^{(1,\delta_1)}$ are nonnegative and bounded by $1$, 
we apply the standard inequality $\|A\|_2 \le \|A\|_\infty \sqrt{\|A\|_0}$, which yields:

\begin{equation}
\left\lVert  \tilde{A}^{(1,\delta_1)} \right\rVert_2^2 
\le 
\left\lVert  \tilde{A}^{(1,\delta_1)} \right\rVert_0.
\end{equation}

Substituting this bound into the previous expression yields:

\begin{equation}
\mathbb{E}_{\zeta}
\left[ \sup_{f \in \mathcal{F}} \left\lVert \sum_{j=1}^m \zeta_j \widetilde{h}_j^{(1)}  \right\rVert_2  \right] \le 2 \sqrt{m} \Pi_{Norm} B \left( 1 + \widetilde{R} R_0 \sqrt{\Delta} \left( \sum_{\delta_1=1}^\Delta \left\lVert  \tilde{A}^{(1,\delta_1)} \right\rVert_0 \right)^{1/2} \right).
\nonumber
\end{equation}

Recalling Eq.~(\ref{gtth_2}), we have:

\begin{equation}
\mathbb{E}_{\zeta}
\left[ \sup_{f \in \mathcal{F}} \left\lVert \sum_{j=1}^m \zeta_j \left(\widetilde{h}_j^{(1)} + \hat{\hat{h}}_j^{(1)} \right)  \right\rVert_2  \right] \le 2 (1+R_2 R_1) \sqrt{m} \Pi_{Norm} B \left( 1 + \widetilde{R} R_0 \sqrt{\Delta} \left( \sum_{\delta_1=1}^\Delta \left\lVert  \tilde{A}^{(1,\delta_1)} \right\rVert_0 \right)^{1/2} \right).
\end{equation}

Substituting this bound into Eq.~(\ref{gtth_1}) gives:

\begin{equation}
\widehat{\mathfrak{R}}(\mathcal{F}) \le \frac{4 \Pi_{Norm}^2 B (1+R_2R_1)}{\sqrt{m}}
 \left( 1+R_0 \widetilde{R} \sqrt{\Delta \eta_{\mathrm{MHT},\tau}^{(1)}} \right).
\end{equation}

Finally, applying the standard transductive Rademacher generalization bound in Eq.~(\ref{th3_0}), and using the $L$-Lipschitz continuity of the loss function, we conclude that for $0<\delta<1$, with probability at least $1 - \delta$ for all $f \in \mathcal{F}$ the following holds:

\begin{equation}
\epsilon(f) \le \hat{\epsilon}(f)
+ \frac{8 \Pi_{Norm}^2 L B (1+R_2R_1)}{\sqrt{m}}
 \left( 1+R_0 \widetilde{R} \sqrt{\Delta \eta_{\mathrm{MHT},\tau}^{(1)}} \right)
+ \sqrt{\frac{2\log(2/\delta)}{n}}.
\end{equation}

This completes the proof.
\end{proof}

\begin{Remark}
For GTs, Theorem \ref{thm:gt} shows that although global self-attention and normalization layers introduce additional architectural components, the structural term still appears explicitly through expressions of the form $\sqrt{\Delta \eta_{\mathrm{MHT},\tau}^{(1)}}$. This confirms that the generalization performance of GTs is still fundamentally governed by the aggregated attention-driven structural complexity. While our formal derivation is presented for a single-layer GT for the sake of clarity, this result can be extended to multi-layer architectures by recursively applying the same bounding arguments across successive layers. This suggests that the observed structural dependency is a persistent characteristic of the GT architecture. Notably, unlike GCNs and GATs where the generalization bound typically involves the geometric mean of structural complexities across multiple layers, the bound for GTs depends on the square root of the structural complexity itself. This suggests that the normalization layers and feed-forward networks in GTs effectively redistribute structural reliance, making their generalization performance less sensitive to structural complexity than traditional GCNs and GATs.
\end{Remark}

\subsubsection{Implications: Structural Complexity Controls Generalization}

Synthesizing the theoretical results from Theorems \ref{thm:gcn}, \ref{thm:gat}, and \ref{thm:gt}, it is evident that the generalization performance of message-passing GNNs is jointly governed by three fundamental factors: the number of labeled nodes $m$, the parameter norms controlling model complexity, and the edge-driven structural complexity induced by message passing. While the $1/\sqrt{m}$ rate and the dependence on parameter norms are consistent with classical neural network analyses on Euclidean data, GNNs exhibit an additional and explicit dependence on structural complexity. This shared structural dependence across diverse architectures highlights a fundamental trade-off in graph representation learning. While expanding the set of effective edges can enhance the model's fitting capacity, it simultaneously inflates the transductive Rademacher complexity, thereby elevating the risk of overfitting to structural noise.

\begin{figure*}[t!]
\centering
\subfigbottomskip=0.05pt
\subfigcapskip=-3pt

    \rotatebox{90}{\scriptsize~~~~~~~~~~~~~~~~~~~GT~~~~~~~~~~~~~~~~~~~~~~~~~~~~~~~~~~GAT~~~~~~~~~~~~~~~~~~~~~~~~~~~~~~~~~GCN}
\subfigure[Cora]
{
 	\begin{minipage}[b]{.3\linewidth}
        \centering
        \includegraphics[width=1.07\linewidth,height=3.9cm]{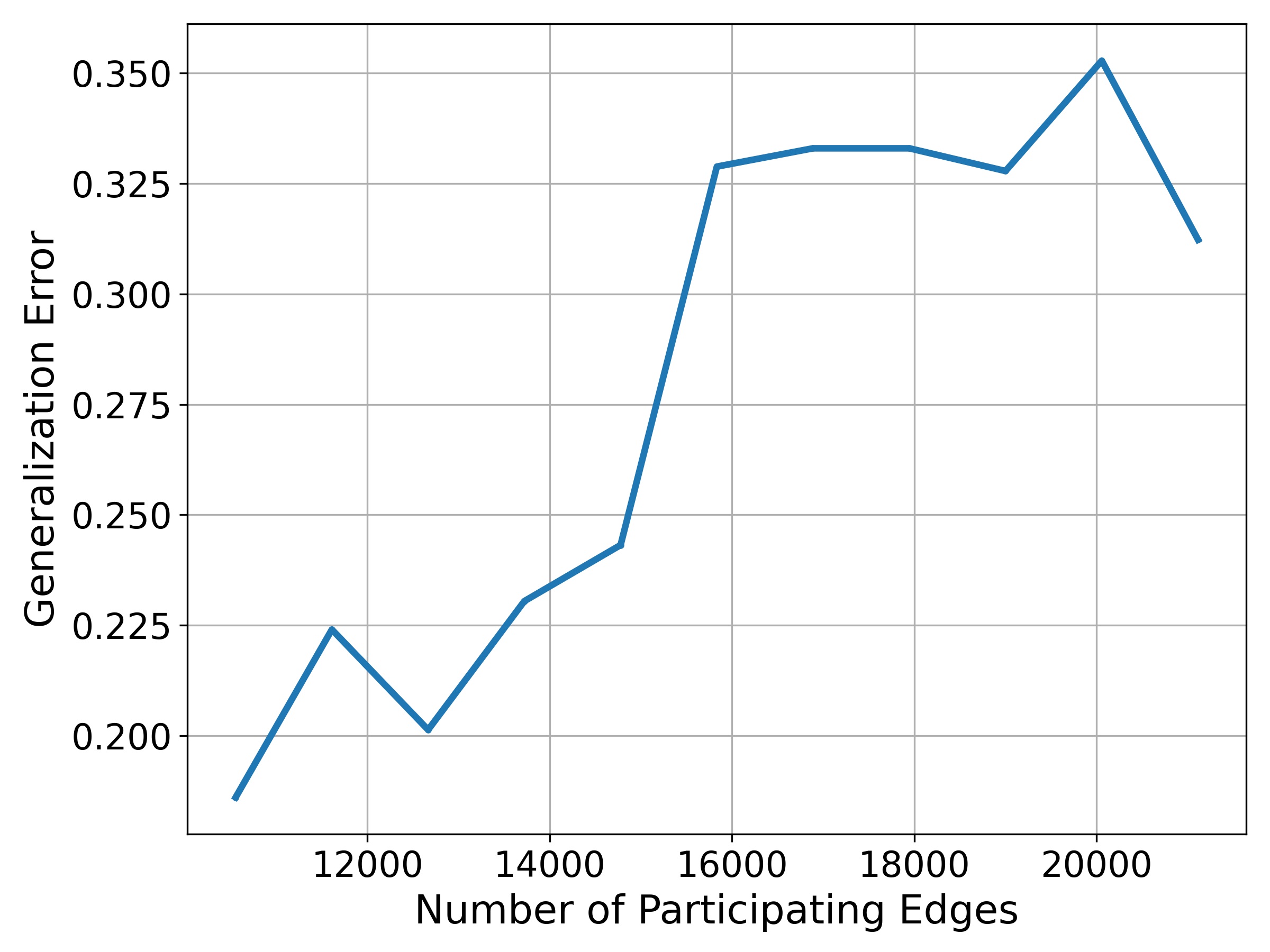}
        \includegraphics[width=1.07\linewidth,height=3.9cm]{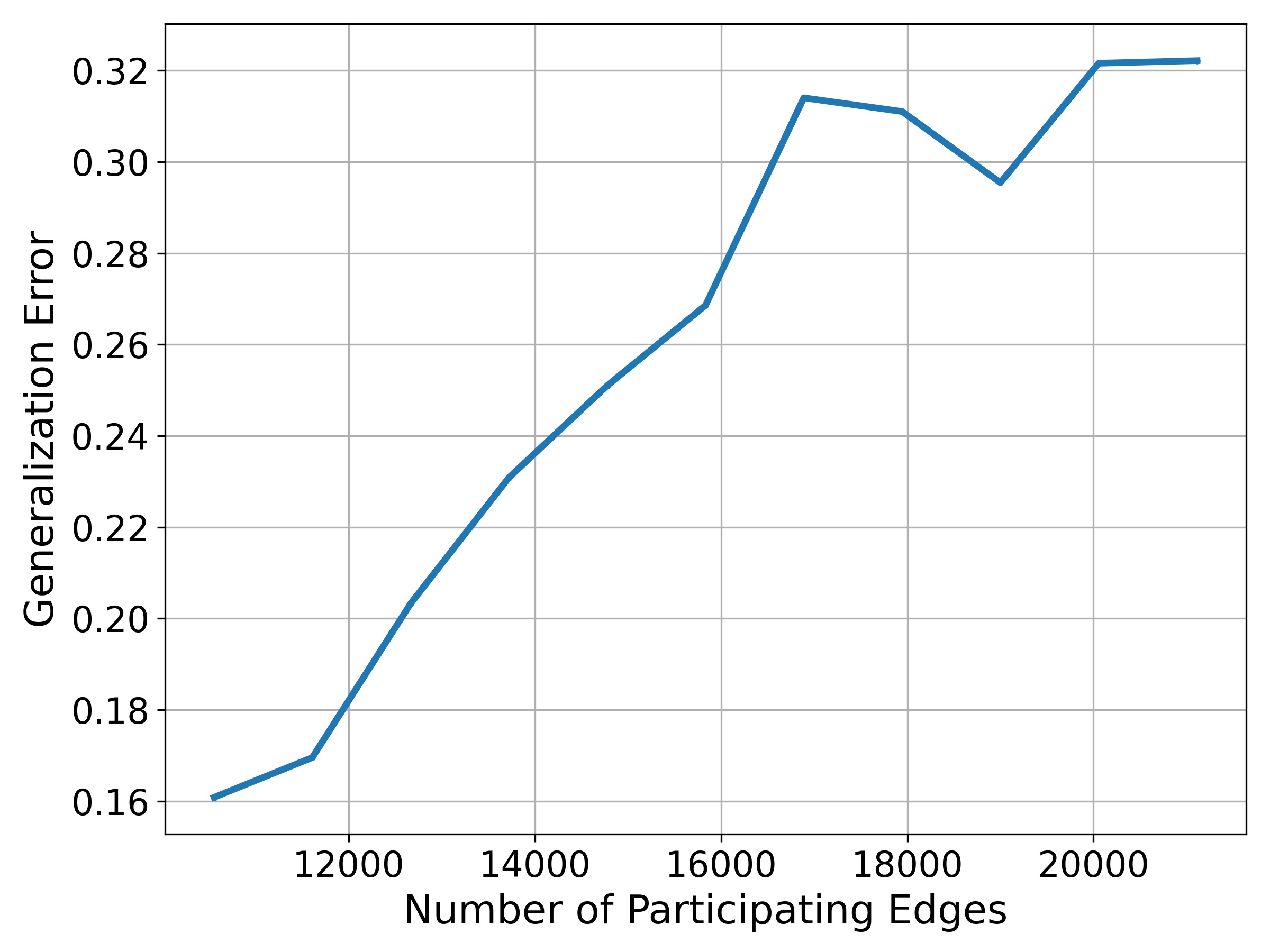}        \includegraphics[width=1.07\linewidth,height=3.9cm]{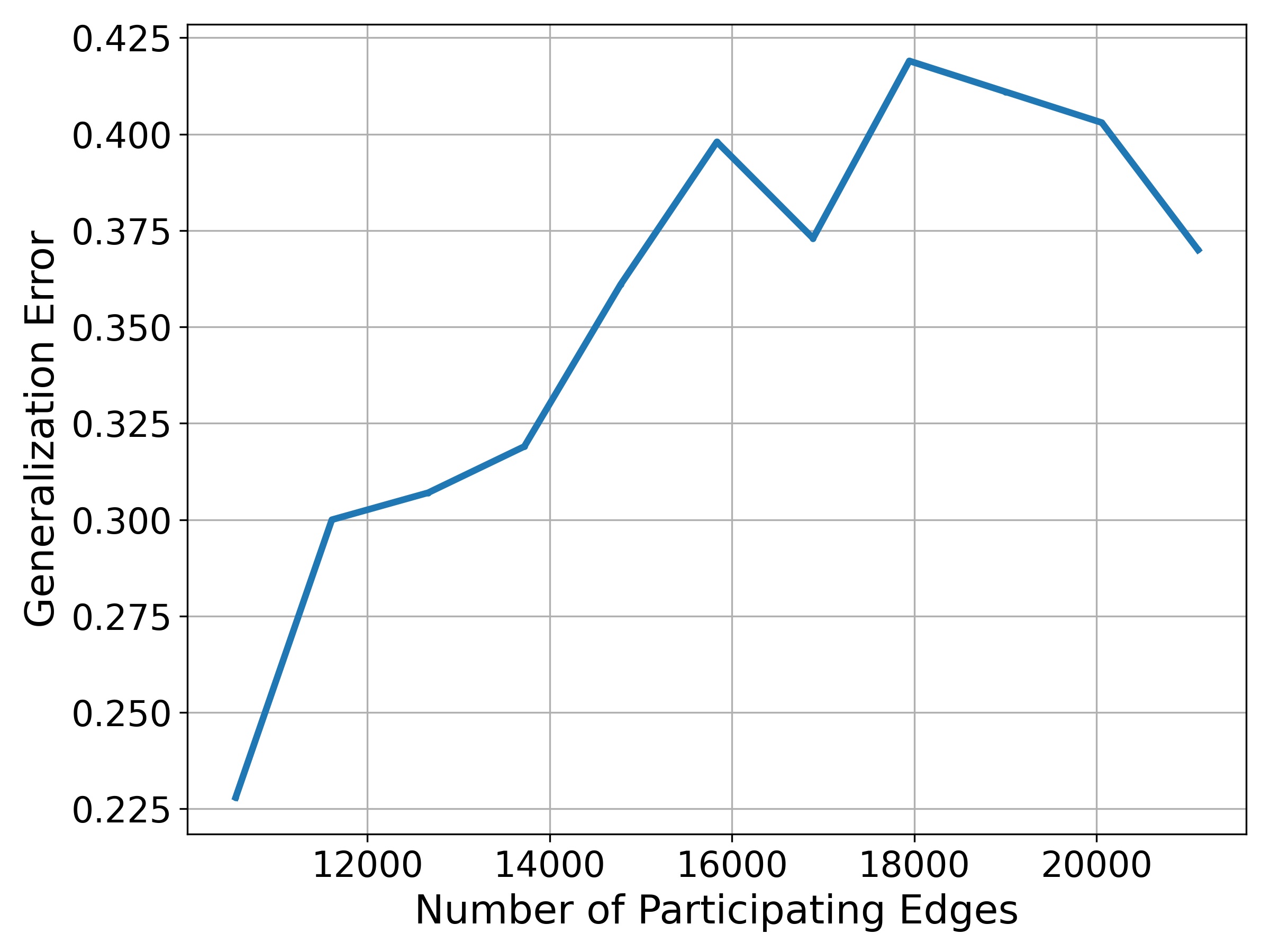}
    \end{minipage}
}
\subfigure[CiteSeer]
{
 	\begin{minipage}[b]{.3\linewidth}
        \centering
        \includegraphics[width=1.07\linewidth,height=3.9cm]{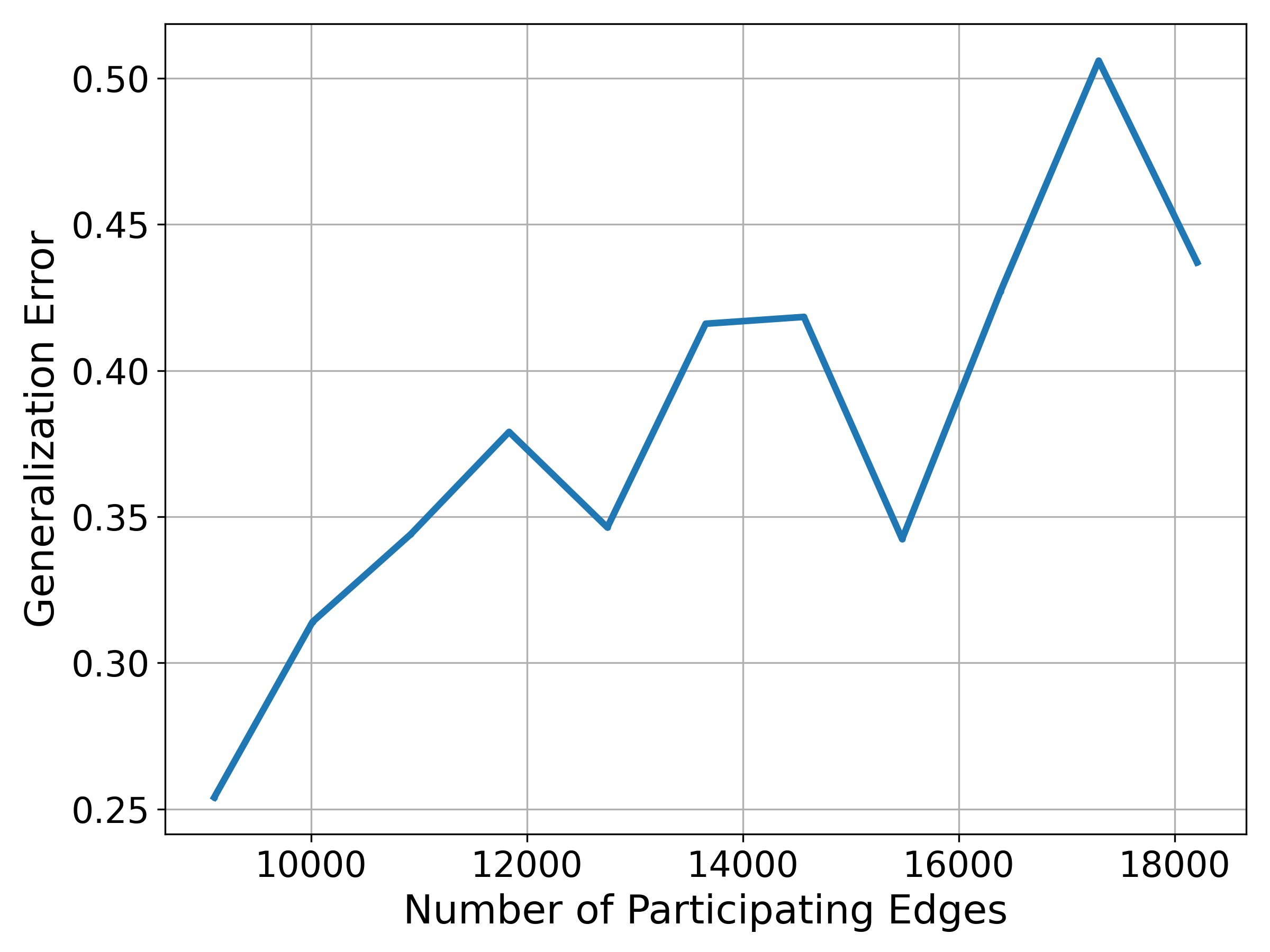}
        \includegraphics[width=1.07\linewidth,height=3.9cm]{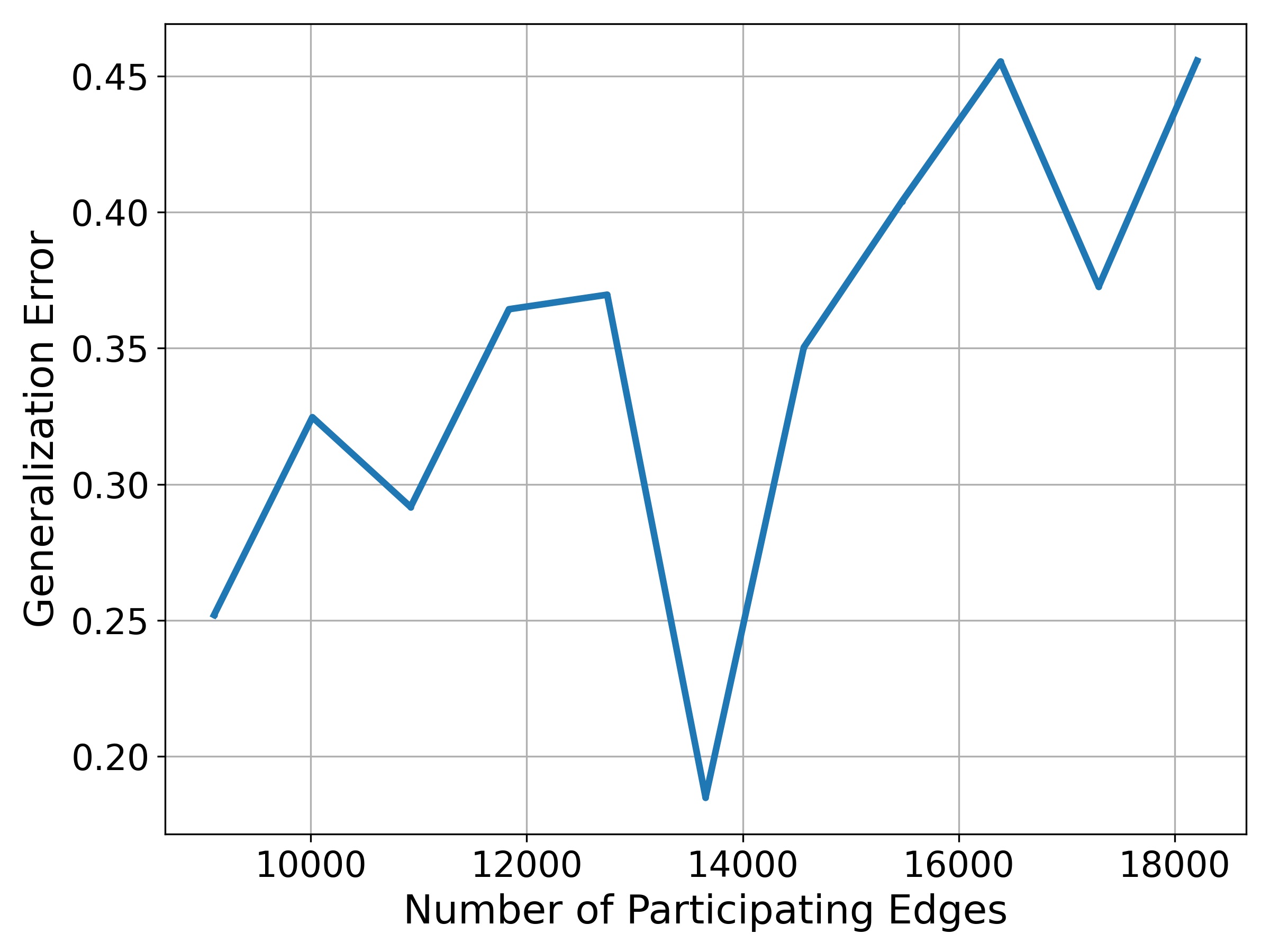}
        \includegraphics[width=1.07\linewidth,height=3.9cm]{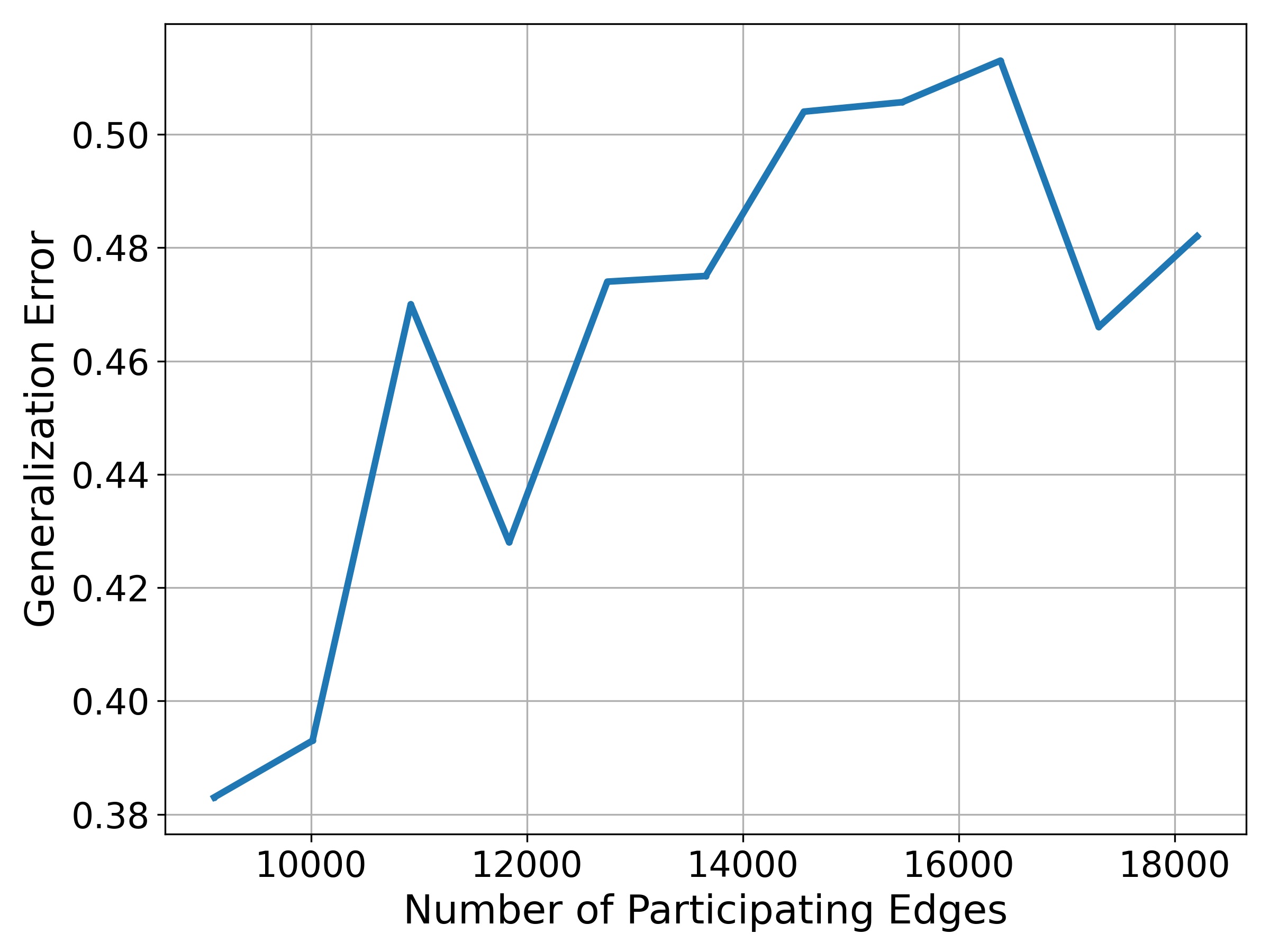}
    \end{minipage}
}
\subfigure[CS]
{
 	\begin{minipage}[b]{.3\linewidth}
        \centering
        \includegraphics[width=1.07\linewidth,height=3.9cm]{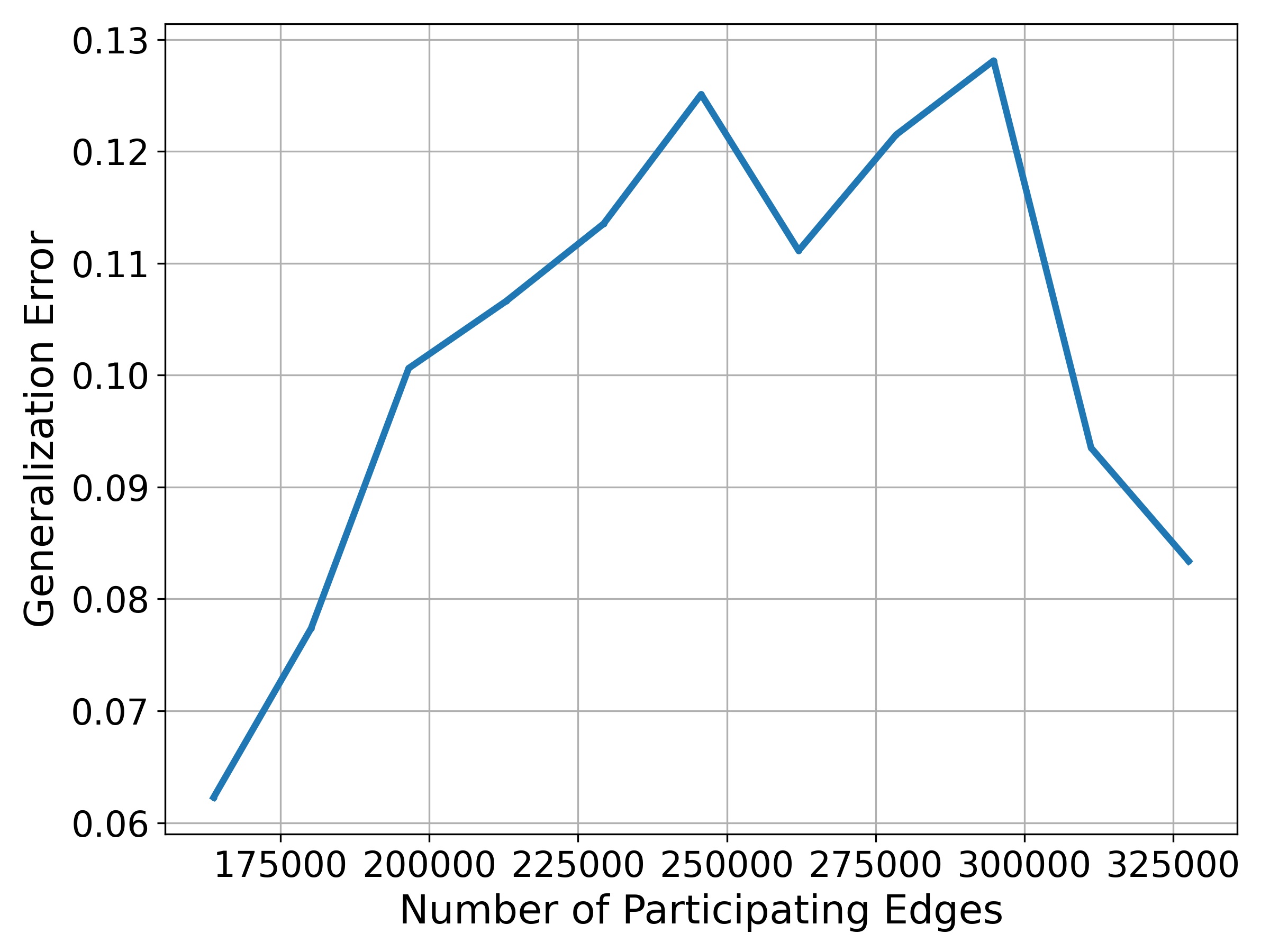}
        \includegraphics[width=1.07\linewidth,height=3.9cm]{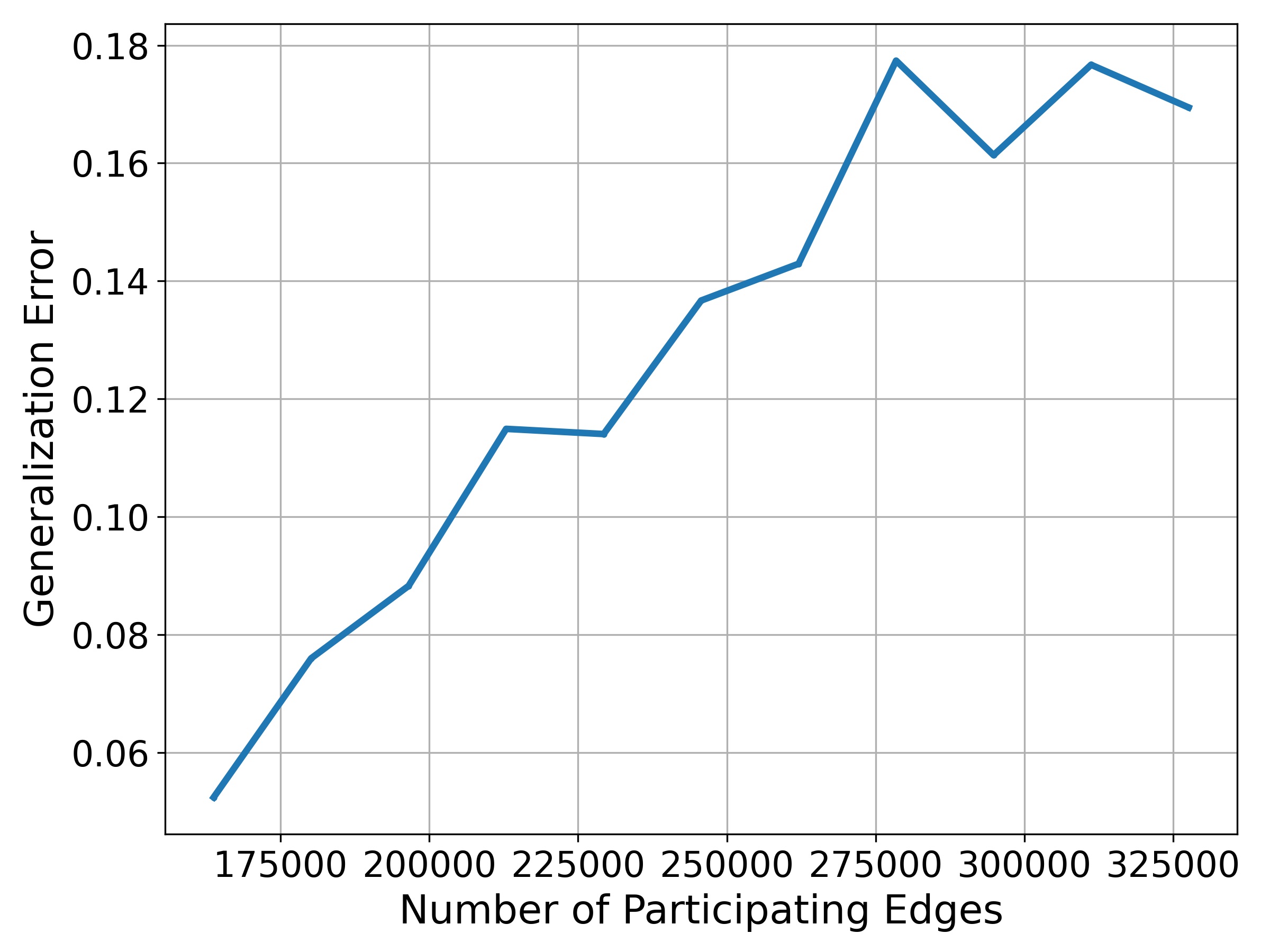}
        \includegraphics[width=1.07\linewidth,height=3.9cm]{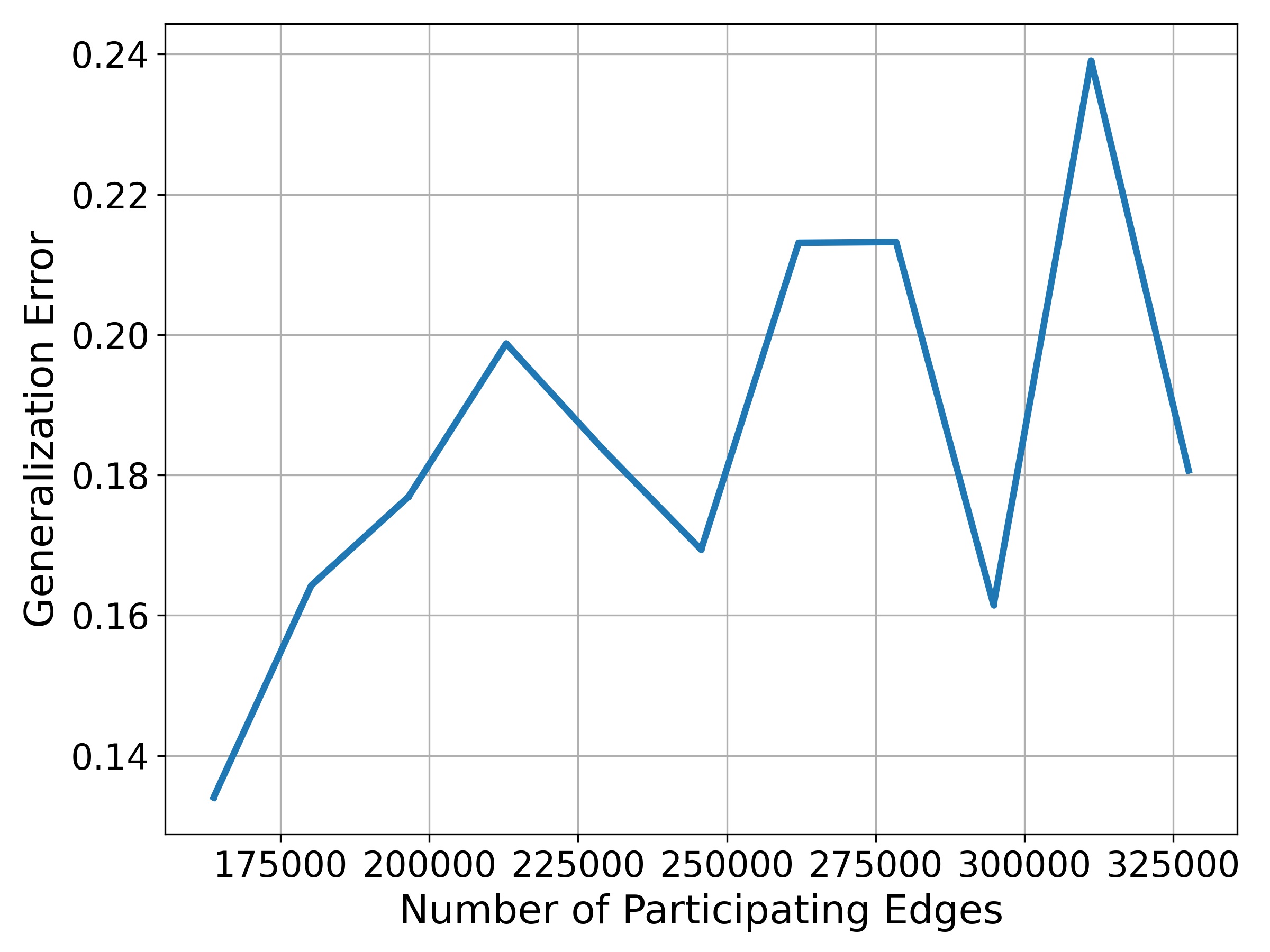}
    \end{minipage}
}
\caption{The relationship between generalization error and the number of effective edges in GCN, GAT and GT.}
\label{gene_err}
\end{figure*}

To empirically corroborate these theoretical bounds, we further conduct a series of validation experiments. We use the absolute difference between the empirical error on the test set and that on the training set as an estimate of the generalization error. We then record the generalization error of three representative GNN architectures, including GCN, GAT, and GT, under varying numbers of edges involved in the decision-making process, with results shown in Figure \ref{gene_err}. The first row corresponds to GCN across three datasets, the second row to GAT, and the third row to GT. In most settings, we observe an increasing trend in generalization error as the number of participating edges increases. These empirical observations corroborate our theoretical analysis, indicating that as the number of effective edges in GNNs increases, the accompanying noise intensifies, leading to overfitting and consequently degrading generalization performance.

In summary, unlike neural networks on independent Euclidean data, the generalization of GNNs is fundamentally governed by both parameter complexity and structural complexity. Motivated by these theoretical insights, we next develop a structure-aware regularization strategy that explicitly controls structural complexity by balancing the number of effective edges participating in message propagation.

\subsection{Structure-Aware Regularization}

The theoretical analysis shows that structural complexity is characterized by the $\ell_0$ norm of the aggregation or attention matrices, which counts the number of effective edges involved in message passing. Therefore, a direct way to improve generalization is to explicitly control this $\ell_0$ term during training. In optimization theory, the $\ell_0$ norm provides a natural way to regulate sparsity, as it directly determines how many elements are non-zero in a vector or matrix, thereby controlling the structural complexity induced by the edges participating in the decision process. However, due to its non-convex and combinatorial nature, optimizing the $\ell_0$ norm is computationally challenging in large-scale neural network training.

\begin{algorithm}[t]
\LinesNumbered
\KwIn{A graph $G=(V,E)$ with node feature matrix $X$ and adjacency matrix $A$, ground-truth labels $Y_S$ for labeled nodes, a randomly initialized GNN model $f_\theta$, regularization coefficient $\lambda$}
\KwOut{Optimized model parameters $\theta$}

\While{not converged}{
$H = f_\theta(X,A)$

Construct class assignment matrix $P$

Compute aggregation weights $\Omega^{(k)} = (\omega_{vu}^{(k)})$ for each layer

Compute class-level structural entropy $\mathcal{L}_{\mathrm{SE}}^{(k)}$ through Eq.(\ref{ser})

Compute total loss $\mathcal{L}_{\mathrm{total}}$ through Eq.(\ref{total})

Compute gradients of $\mathcal{L}_{\mathrm{total}}$ with respect to $\theta$

Update model parameters $\theta$ using gradient descent
}
\caption{Training algorithm with structure entropy regularization}
\label{alg1}
\end{algorithm}

To obtain a tractable alternative, we introduce a structure-aware regularization term based on structural entropy \citep{li2016structural, wang2023user, zou2023se}. This regularization strategy differs from existing graph regularization methods in both motivation and mechanism. Rather than randomly dropping edges \citep{bo2022regularizing} or imposing generic smoothness constraints \citep{ando2006learning, yang2021rethinking}, it is explicitly derived from our structural-complexity generalization analysis. By using structural entropy as a differentiable surrogate for the combinatorial effective-edge count, this approach enables the model to regulate the effective message-passing structure in a principled manner, suppressing structurally misaligned aggregation while preserving beneficial within-class smoothing. 

Specifically, for each layer $k$, we define a unified non-negative aggregation weight $\omega_{vu}^{(k)}$ representing the normalized contribution of node $u$ to node $v$. In attention-based models, $\omega_{vu}^{(k)}$ corresponds to the learned attention coefficient. In fixed aggregation models such as GCN, we introduce a learnable edge mask $M_{vu}\in[0,1]$ and define $\omega_{vu}^{(k)} = M_{vu}\hat{A}_{vu}$, so that the effective connectivity can be adaptively adjusted during training. The resulting matrix $\Omega^{(k)}=(\omega_{vu}^{(k)})$ therefore represents the effective message-passing structure induced by the model at layer $k$.

Based on $\Omega^{(k)}$, we construct a three-level encoding tree to measure structural entropy. The root node corresponds to the entire vertex set $V$, the intermediate level consists of class-level subsets, and the leaf nodes correspond to individual vertices. Let $P\in\mathbb{R}^{n\times c}$ denote the class assignment matrix, where $P_{ij}$ represents the probability that node $v_i$ belongs to class $j$. Each intermediate node in the encoding tree corresponds to a class subset, and leaf nodes are connected to intermediate nodes through the class assignment matrix $P$, whose rows are fixed to ground-truth one-hot labels for labeled data and given by the model posterior distribution for unlabeled data. For layer $k$, the weighted degree induced by the aggregation weights is defined as $d_v^{(k)}=\sum_{u} \omega_{vu}^{(k)}$, and the total volume is $\mathrm{vol}^{(k)}(V)=\sum_{v} d_v^{(k)}$. For each class subset $C_j$, its weighted volume is $\mathrm{vol}^{(k)}(C_j) 
= \sum_{v \in V} P_{vj} d_v^{(k)}$, and its cut weight is $g_{C_j}^{(k)} = \sum_{v \in V} P_{vj} \sum_{u \in \mathcal{N}(v)} \omega_{vu}^{(k)} (1 - P_{uj})$, which measures the total weight of edges connecting class $j$ to other classes. The structural entropy at layer $k$ is then defined as:

\begin{equation} \label{ser}
\mathcal{L}_{\mathrm{SE}}^{(k)}
=
- \sum_{j=1}^{c}
\frac{g_{C_j}^{(k)}}{\mathrm{vol}^{(k)}(V)}
\log
\frac{\mathrm{vol}^{(k)}(C_j)}{\mathrm{vol}^{(k)}(V)} .
\end{equation}

Since the root corresponds to the entire graph and the leaves correspond to individual nodes, the entropy term associated with the intermediate class layer captures how the effective connectivity aligns with the class partition in the three-tier encoding tree. Optimizing $\mathcal{L}_{\mathrm{SE}}^{(k)}$ suppresses cross-class aggregation weights and encourages connectivity to concentrate within class subsets. This leads the connectivity pattern encoded in $\Omega^{(k)}$ to become more consistent with the class partition, reducing structurally irrelevant or misaligned connections. Therefore, structural entropy acts as a continuous relaxation of the combinatorial $\ell_0$ structural complexity term. Rather than explicitly counting active edges, it penalizes dispersed and cross-partition connectivity, thereby regulating the effective edge set involved in message passing. Unlike directly optimizing the $\ell_0$ term, this entropy-based regularizer is fully differentiable and compatible with gradient-based optimization. We refer to this entropy-based structural regularization strategy as Structure Entropy Regularization (SER). 

Finally, the overall training objective is formulated as:

\begin{equation} \label{total}
\mathcal{L}_{\mathrm{total}}
= \mathcal{L}_{\mathrm{CLS}} + \lambda \sum_{k=1}^{K}
\mathcal{L}_{\mathrm{SE}}^{(k)},
\end{equation}
where $\mathcal{L}_{\mathrm{CLS}}$ denotes the standard cross-entropy loss evaluated on the labeled node set $\mathcal{S}$, defined as $\mathcal{L}_{\mathrm{CLS}} = - \sum_{v \in \mathcal{S}} \sum_{j=1}^{c} y_{vj}\log \hat{y}_{vj}$,
and $\lambda > 0$ controls the strength of the structure-aware regularization. The overall optimization procedure is summarized in Algorithm \ref{alg1}.

\section{Experiments}
\label{experiments}

In this section, we evaluate the impact of structure-aware regularization on three representative GNN models: GCN, GAT, and GT. We compare the performance of the regularized models with their baselines across various benchmark datasets and analyze the effect of the regularization strength on generalization.

\begin{table}[t]
\caption{Statistics of the datasets.}\label{tab:data}
\centering
\begin{tabular}{ccccc}
\hline
Dataset & \# Nodes & \# Edges & \# Classes & \# Features\\
\hline
Cora & 2,708 & 5,429 & 7 &  1,433\\
CiteSeer & 3,327 & 4,732 & 6 &  3,703\\
PubMed & 19,717 & 44,338 & 3 &  500\\
CS & 18,333 & 81,894 & 15 &  6,805\\
Physics & 34,493 & 247,962 & 5 &  8,415\\
Computers & 13,381 & 245,778 & 10 &  767\\
Photo & 7,487 & 119,043 & 8 &  745\\
\hline
\end{tabular}
\end{table}

\subsection{Datasets}

We evaluate our method on seven widely used benchmark datasets for semi-supervised node classification with random splits. These include three citation networks \citep{sen2008collective} (Cora, CiteSeer, and PubMed), two co-authorship networks \citep{shchur2018pitfalls} (CS and Physics), and two Amazon co-purchase networks \citep{shchur2018pitfalls} (Computers and Photo). The statistical summary of these datasets, including the numbers of nodes, edges, classes, and features, is reported in Table \ref{tab:data}, followed by detailed information about each dataset.

\begin{itemize}

\item[$\bullet$] \textbf{Citation networks.}
Cora, CiteSeer, and PubMed are citation graphs in which nodes represent scientific publications and edges denote citation relationships. Each node is associated with a sparse bag-of-words feature vector constructed from the document content, and class labels correspond to research topics.

\item[$\bullet$] \textbf{Co-author networks.}
CS and Physics are co-authorship graphs constructed from the Microsoft Academic database. Nodes represent authors, and edges indicate co-authorship relations. Node features are derived from the keywords of published papers, and labels correspond to the primary research area of each author.

\item[$\bullet$] \textbf{Amazon co-purchase networks.}
Computers and Photo are segments of the Amazon co-purchase network. Nodes correspond to products, and edges connect items that are frequently purchased together. Node features are bag-of-words representations extracted from product reviews, and labels indicate product categories.

\end{itemize}

All datasets are treated as undirected graphs. Following prior work, we adopt random splits on each graph, where a fixed number of labeled nodes per class are used for training and the remaining nodes are used for validation and testing.

\begin{table}[t]
\caption{Node classification accuracy (\%) of different regularization strategies on GCN.}\label{tab:main_results_GCN}
\centering
\setlength{\tabcolsep}{2pt}
\scalebox{0.87}{
\begin{tabular}{cccccccc}
\hline
 & Cora & CiteSeer & PubMed & CS & Physics & Computers & Photo\\
\hline
 Vanilla & 81.54±0.88
 & 71.40±0.69
 & 79.18±0.29
 &  91.25±0.71
 &  93.25±0.37
 &  69.51±3.40
 &  83.24±2.01
\\
 LapR & 82.13±0.49
 & 71.70±0.75
 & \textbf{79.62±0.32}
 & 91.50±0.61
 & 93.89±0.21
 & \uline{73.49±1.39}
 & 84.48±1.22
\\
 CP & \uline{82.44±2.97}
 & 74.53±0.28
 & 78.93±0.42
 & 92.26±0.21
 & 94.12±0.79
 & 67.64±8.59
 & \uline{85.76±2.01}
\\
 P-reg & 83.25±0.56
 & \uline{74.56±0.35}
 & 78.32±0.37
 & 92.35±0.19
 & 93.71±1.26
 & 72.22±3.35
 & \textbf{87.20±1.58}
\\
 NASA & 81.03±0.93
 & 71.78±0.41
 & 78.95±0.37
 & 91.55±0.60
 & 93.36±0.38
 & 71.96±2.54
 & 84.44±1.41
\\
 R-reg & 82.29±0.56
 & 73.87±0.39
 & 79.43±0.36
 & \uline{92.36±0.20}
 & \uline{94.24±0.49}
 & 69.10±6.15
 & 84.24±2.82
\\
 SER & \textbf{82.85±0.54}
 & \textbf{74.79±0.54}
 & \uline{79.50±0.29}
 & \textbf{92.57±0.29}
 & \textbf{94.57±0.30}
 & \textbf{73.58±2.57}
 & 85.11±2.43
\\
\hline
\end{tabular}}
\end{table}

\begin{table}[t]
\caption{Node classification accuracy (\%) of different regularization strategies on GAT.}\label{tab:main_results_GAT}
\centering
\setlength{\tabcolsep}{3pt}
\scalebox{0.87}{
\begin{tabular}{cccccccc}
\hline
 & Cora & CiteSeer & PubMed & CS & Physics & Computers & Photo\\
\hline
 Vanilla & 82.90±0.66
 & 71.88±0.86
 & 77.99±0.55
 & 91.22±0.45
 & 93.04±0.51
 & 73.85±2.75
 & 87.04±2.03
\\
 LapR & 82.81±0.93
 & 71.67±0.82
 & \textbf{79.63±0.37}
 & 91.54±0.55
 & 93.83±0.23
 & 73.59±1.21
 & 84.39±1.11
\\
 CP & \uline{83.30±0.41}
 & \uline{72.75±0.81}
 & 78.28±0.41
 & 91.77±0.18
 & 93.89±0.39
 & 74.14±2.67
 & \uline{88.62±2.13}
\\
 P-reg & 82.45±0.80
 & 72.60±0.52
 & 78.51±0.68
 & \uline{91.78±0.28}
 & 94.00±0.34
 & \uline{74.41±3.51}
 & 88.56±2.27
\\
 NASA & 83.07±0.64
 & 71.87±1.27
 & 77.71±0.69
 & 90.80±0.75
 & 93.05±0.79
 & 72.84±2.91
 & 86.54±2.06
\\
R-reg & 83.17±0.71
 & 72.33±0.74
 & \uline{78.60±0.33}
 & 91.77±0.22
 & \uline{94.02±0.35}
 & 62.80±3.51
 & 62.59±11.97
\\
 SER & \textbf{83.44±0.41}
 & \textbf{73.59±0.57}
 & 78.44±0.63
 & \textbf{91.84±0.22}
 & \textbf{94.18±0.34}
 & \textbf{75.27±2.84}
 & \textbf{88.98±2.22}
\\
\hline
\end{tabular}}
\end{table}

\begin{table}[t]
\caption{Node classification accuracy (\%) of different regularization strategies on GT.}\label{tab:main_results_GT}
\centering
\setlength{\tabcolsep}{3pt}
\scalebox{0.87}{
\begin{tabular}{cccccccc}
\hline
 & Cora & CiteSeer & PubMed & CS & Physics & Computers & Photo\\
\hline
Vanilla & 76.60±1.55
 & 65.09±1.27
 & 75.26±1.32
 & 86.79±1.57
 & 90.70±1.53
 & 77.12±3.50
 & 87.02±2.10
\\
 LapR & \uline{78.69±1.76}
 & 68.83±1.89
 & \textbf{76.71±0.80}
 & 87.06±6.65
 & 92.66±0.49
 & 78.03±1.90
 & 88.35±1.34
\\
 CP & 77.43±1.38
 & 64.78±3.84
 & 75.74±0.77
 & 89.3.±1.03
 & 91.44±0.96
 & 80.42±2.40
 & 88.46±1.36
\\
 P-reg & 75.800±3.22
 & 62.76±5.24
 & 76.41±1.54
 & \uline{89.72±2.23}
 & \textbf{93.08±0.77}
 & \textbf{82.56±1.88}
 & \uline{88.81±1.96}
\\
 NASA & 64.11±9.75
 & 60.04±7.05
 & 73.23±3.47
 & 86.88±10.98
 & 81.85±13.74
 & 78.08±13.94
 & 80.35±4.91
\\
 R-reg & 78.57±1.49
 & \uline{69.05±2.03}
 & 76.60±1.39
 & 87.43±6.27
 & \uline{93.17±0.67}
 & 54.69±14.88
 & 70.58±4.80
\\
 SER & \textbf{78.96±0.95}
 & \textbf{69.58±1.60}
 & \uline{76.68±0.86}
 & \textbf{90.17±1.53}
 & 93.02±0.59
 & \uline{81.19±2.22}
 & \textbf{89.13±2.03}
\\
\hline
\end{tabular}}
\end{table}

\subsection{Experimental Setup}

We evaluate the proposed structure-aware regularization on three representative message-passing architectures: GCN, GAT, and GT. For each backbone, we compare the vanilla model (without structural regularization) and the regularized version obtained by adding the proposed structural entropy term.

For GCN, we adopt the standard two-layer architecture with symmetric normalization. For GAT, we use multi-head attention in the first layer and a single head in the output layer. For GT, we implement a two-layer architecture with multi-head self-attention and layer normalization. All models use ReLU activation and dropout for regularization.

For fair comparison, the hidden dimensions, learning rates, weight decay, and dropout rates are tuned individually for each backbone and kept identical between the baseline and the regularized variant. The regularization coefficient $\lambda$ is selected from the range $[0,1]$ based on validation performance. To mitigate the potential instability introduced by the self-referential nature of the proposed regularization, we incorporate a warm-up schedule during training. Specifically, the structural entropy term is gradually introduced in the early training stages, allowing the model to first learn reliable node representations before enforcing structural constraints. This progressive scheduling stabilizes the interaction between representation learning and structure induction, and ensures that the regularization signal becomes more informative as training progresses. Each experiment is repeated over multiple random splits, and we report the average classification accuracy with standard deviation.

To further verify the effectiveness of our approach, we also incorporate several established regularization techniques for comparison. Detailed information about the compared methods is provided as follows:

\begin{itemize}

\item[$\bullet$] Laplacian regularization (LapR) \citep{ando2006learning} enforces spatial smoothness by penalizing the distance between representations of adjacent nodes to prevent over-fitting and improve semi-supervised generalization.

\item[$\bullet$] Confidence penalty (CP) \citep{pereyra2017regularizing} adds the negative entropy of the network outputs to the classification loss to penalize overconfident predictions and improve generalization.

\item[$\bullet$] P-reg \citep{yang2021rethinking} imposes a variation of Laplacian-based regularization that possesses the equivalent power of an infinite-depth GCN to capture long-range structural information.

\item[$\bullet$] NASA \citep{bo2022regularizing} replaces immediate neighbors with remote neighbors to generate graph augmentations that balance consistency and diversity via neighbor-constrained regularization.

\item[$\bullet$] R-reg \citep{ji2025graph} introduces the concept of graph distributional signals to encode smoothness and non-uniformity of model outputs as a plug-in regularization term.

\end{itemize}

\begin{figure*}[!t]

\centerline{\includegraphics[width=25pc]{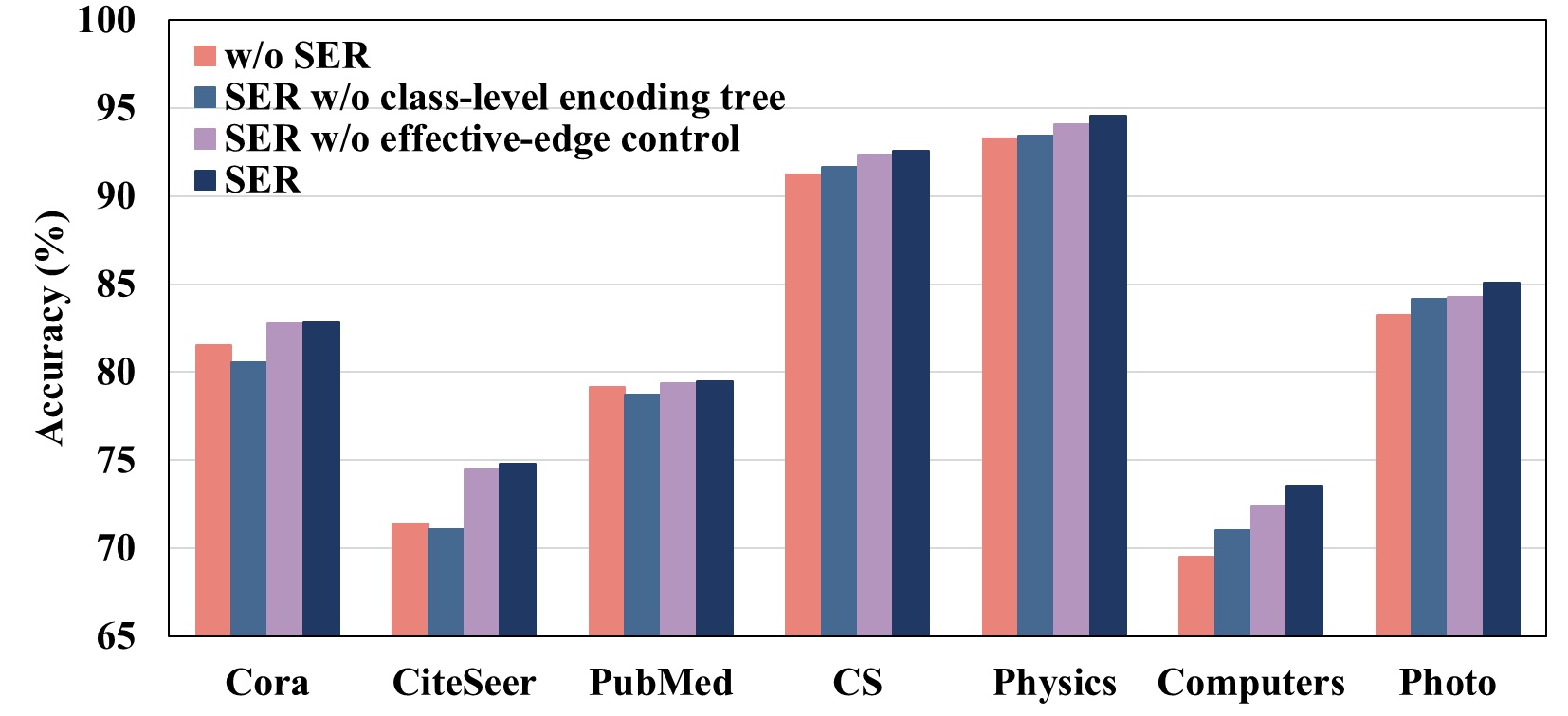}}
\caption{Ablation study of SER on seven benchmark datasets with GCN.}
\label{abl_GCN}
\end{figure*}

\begin{figure*}[!t]

\centerline{\includegraphics[width=25pc]{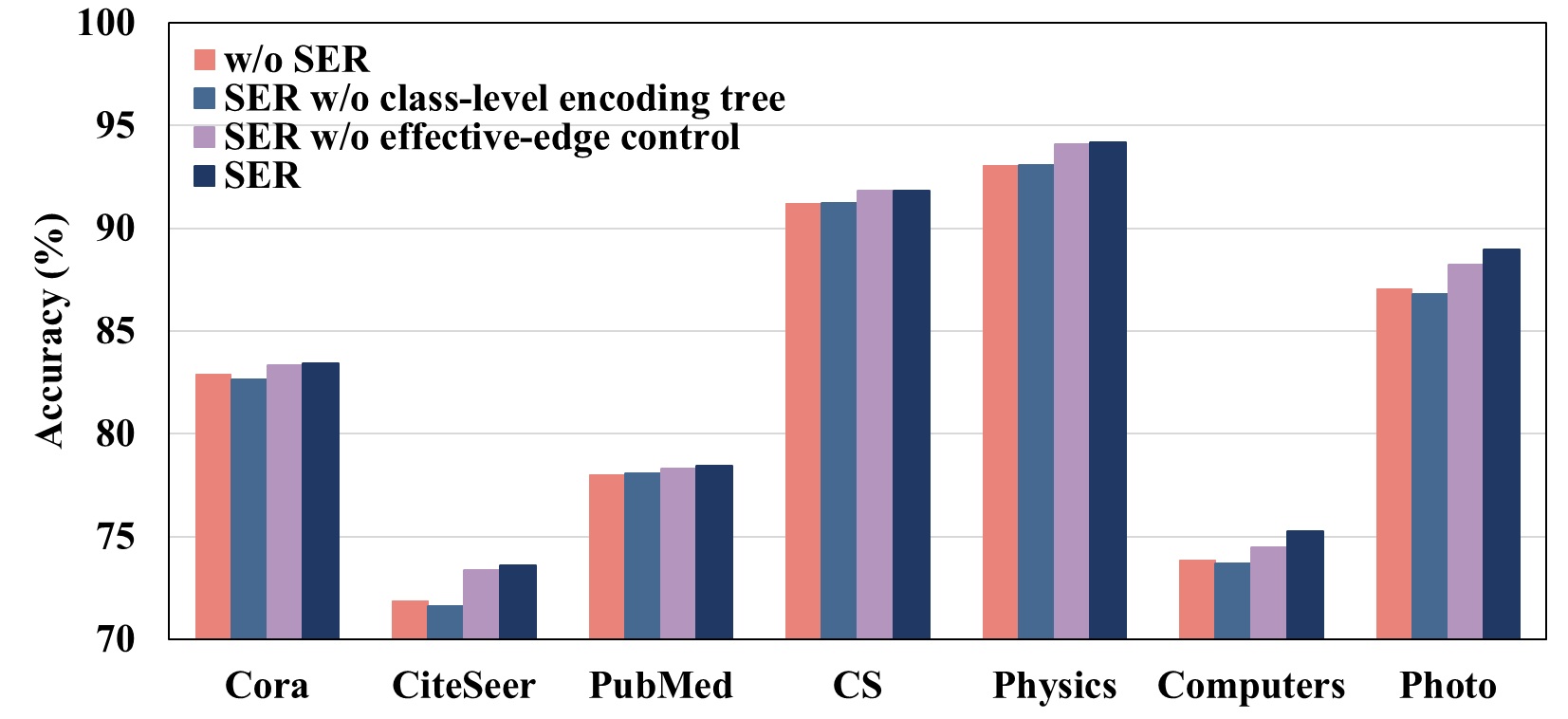}}
\caption{Ablation study of SER on seven benchmark datasets with GAT.}
\label{abl_GAT}
\end{figure*}

\begin{figure*}[!t]

\centerline{\includegraphics[width=25pc]{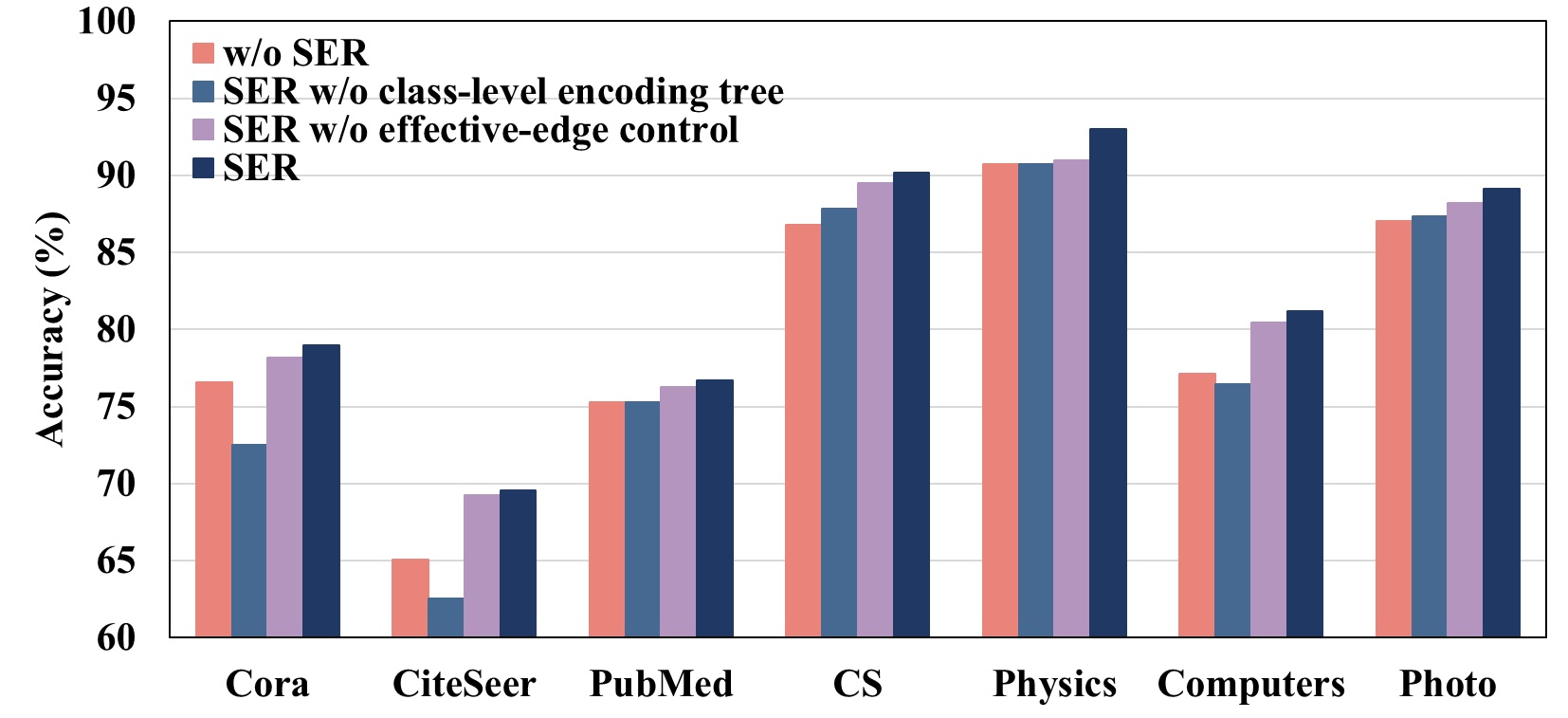}}
\caption{Ablation study of SER on seven benchmark datasets with GT.}
\label{abl_GT}
\end{figure*}

\subsection{Overall Performance Comparison}

Tables \ref{tab:main_results_GCN}, \ref{tab:main_results_GAT}, and \ref{tab:main_results_GT} report the node classification accuracy of all models under different regularization strategies, where the best results are highlighted in \textbf{bold} and the second-best results are \uline{underlined}.

Across all seven datasets, incorporating regularization generally improves performance over the corresponding vanilla models, indicating that appropriate regularization is beneficial for graph neural networks. However, the effectiveness of different regularization strategies varies significantly across datasets and architectures, reflecting their distinct mechanisms and inductive biases.

Methods such as CP and P-reg provide moderate improvements on several datasets but show limited robustness, especially on more complex graphs. Their inconsistent performance across different architectures suggests that prediction-level or fixed propagation constraints are insufficient to manage the dynamic structural complexity of GNNs. The LapR achieves consistent gains on relatively homophilous datasets by enforcing local smoothness, but its performance tends to degrade on denser graphs due to over-smoothing effects. NASA, which leverages neighborhood-based augmentation, exhibits competitive performance in some cases but suffers from instability, particularly on transformer-based architectures such as GT. The recently proposed R-reg, which models graph distributional signals, shows competitive results but exhibits instability on attention-based models like GAT and GT. This suggests that enforcing distributional smoothness can conflict with the anisotropic nature of attention mechanisms, leading to inconsistent generalization.

In contrast, the proposed SER consistently outperforms both the vanilla baselines and competing regularization methods across most datasets and model architectures. The improvements are particularly pronounced on datasets with denser connectivity, such as Physics and Computers, where excessive cross-class message passing is more likely to occur. From a model-specific perspective, for GCN, SER achieves the best or near-best performance across all datasets, demonstrating that explicitly controlling effective connectivity effectively mitigates over-smoothing and structural overfitting. For GAT, the improvements indicate that suppressing unnecessary cross-class attention weights enhances the discriminative power of learned attention coefficients. For GT, despite its global self-attention mechanism, SER still yields consistent gains over both the baseline and other regularization methods, confirming that even globally connected architectures benefit from explicitly regulating effective edge participation.

Overall, the empirical results strongly support our theoretical analysis. As predicted by the generalization bounds, structural complexity plays a crucial role in determining model performance. By optimizing structural entropy to regulate the complexity of the effective edges involved in message passing, the proposed regularizer improves generalization across different GNN architectures.

\begin{figure*}[t!]
\centering
\subfigbottomskip=0.05pt
\subfigcapskip=-6pt

\subfigure[Cora]
{
 	\begin{minipage}[b]{.3\linewidth}
        \centering
        \includegraphics[scale=0.22]{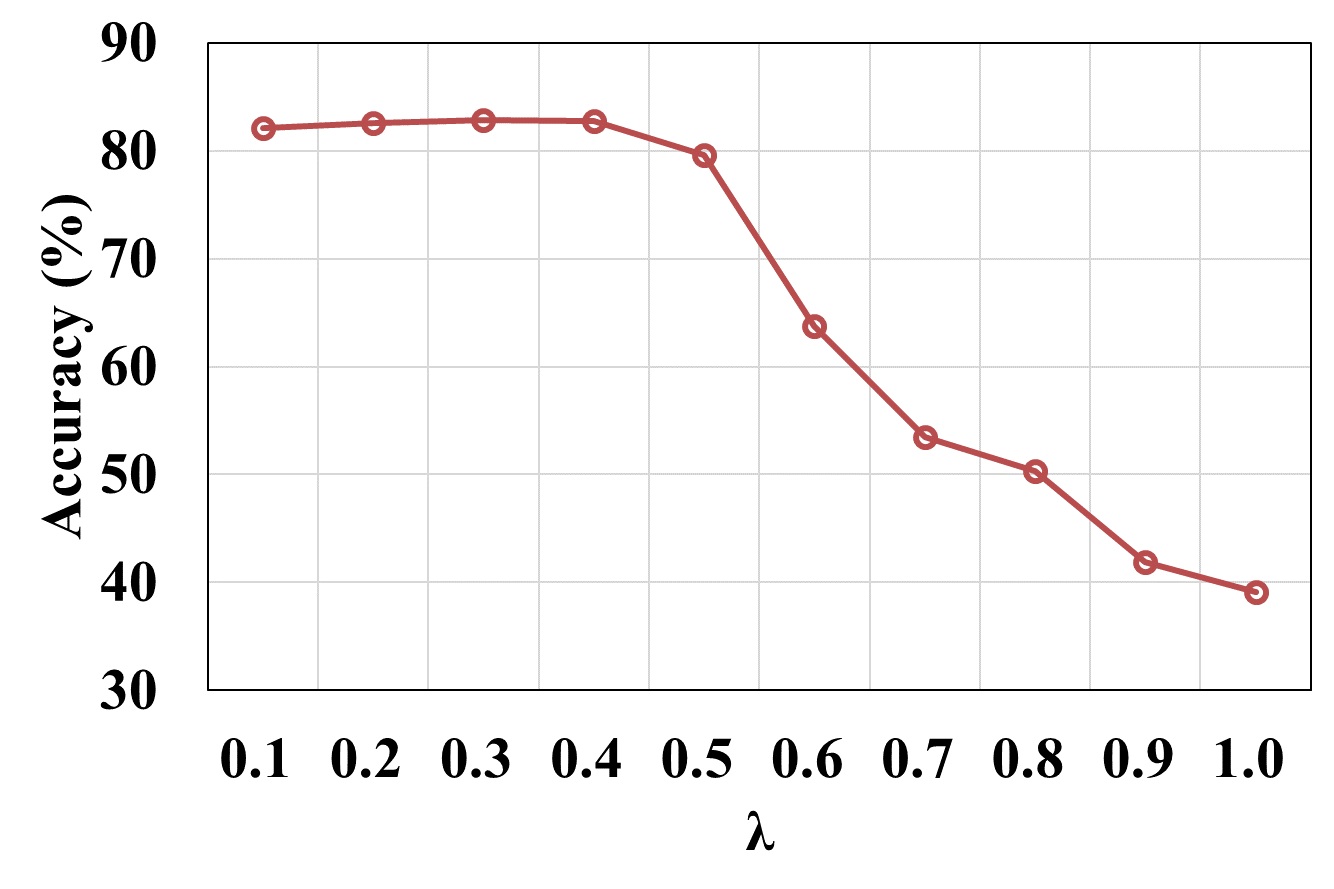}
    \end{minipage}
}
\subfigure[CiteSeer]
{
 	\begin{minipage}[b]{.3\linewidth}
        \centering
        \includegraphics[scale=0.22]{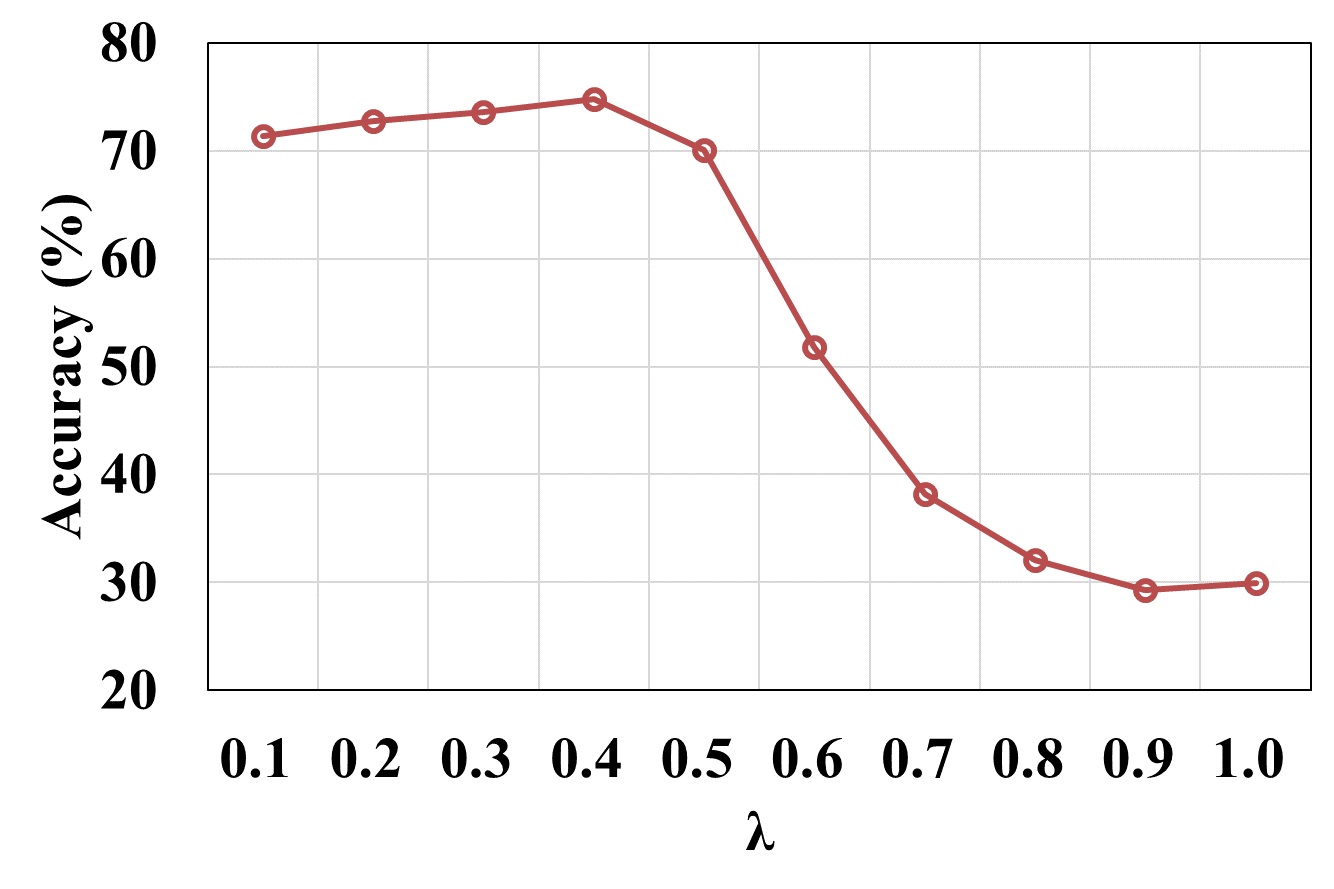}
    \end{minipage}
}
\subfigure[PubMed]
{
 	\begin{minipage}[b]{.3\linewidth}
        \centering
        \includegraphics[scale=0.22]{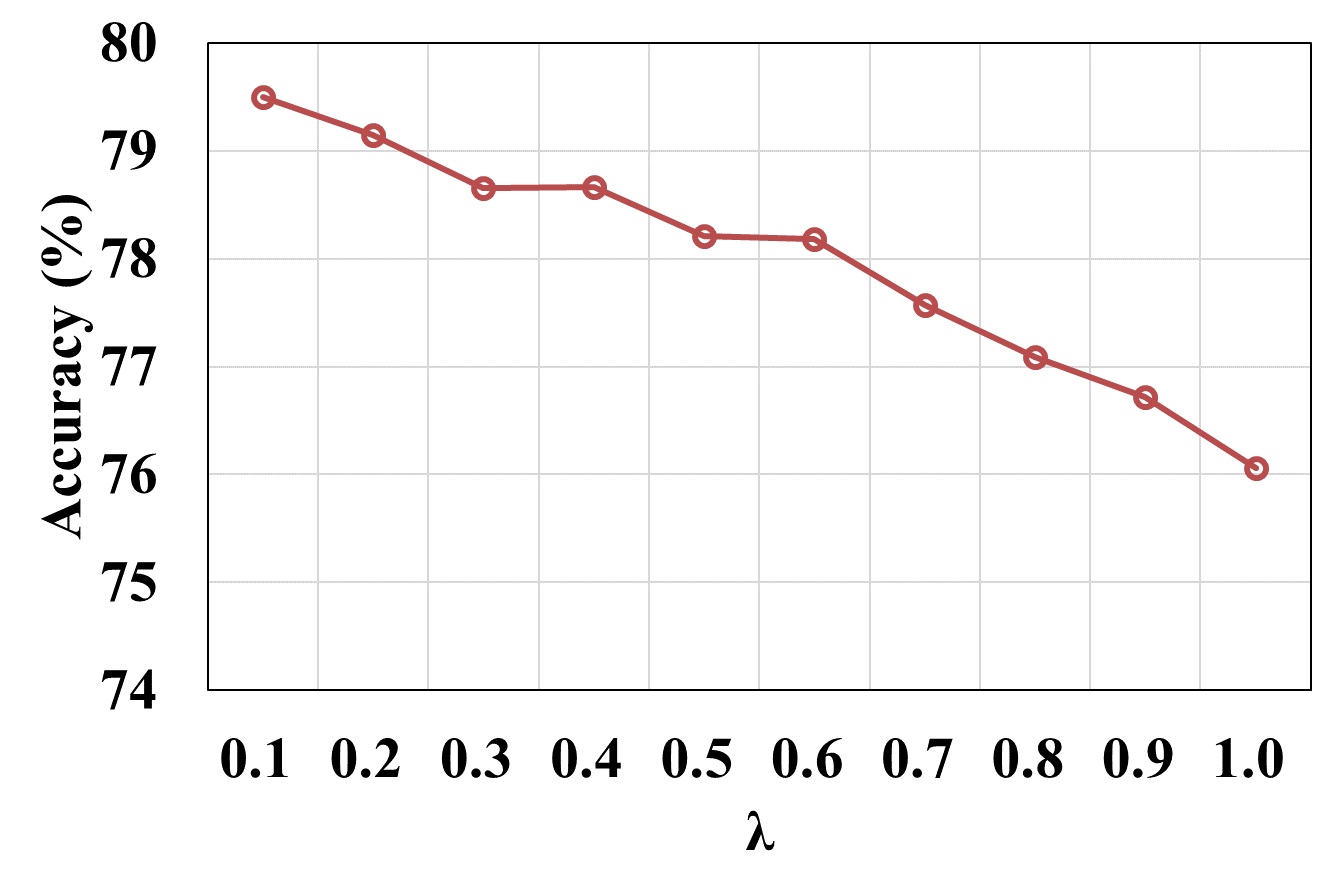}
    \end{minipage}
}
\quad
\subfigure[CS]
{
 	\begin{minipage}[b]{.3\linewidth}
        \centering
        \includegraphics[scale=0.22]{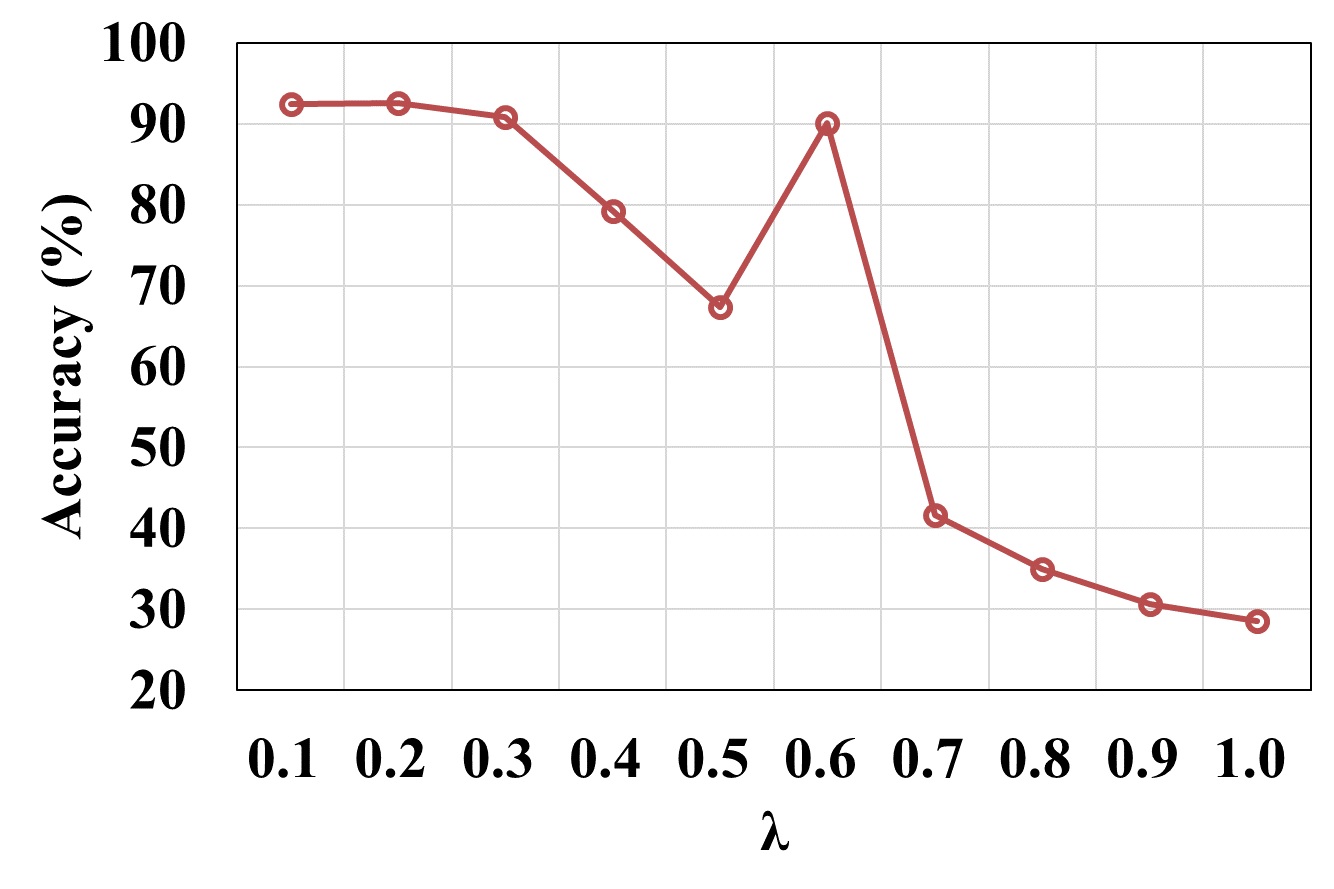}
    \end{minipage}
}
\subfigure[Physics]
{
 	\begin{minipage}[b]{.3\linewidth}
        \centering
        \includegraphics[scale=0.22]{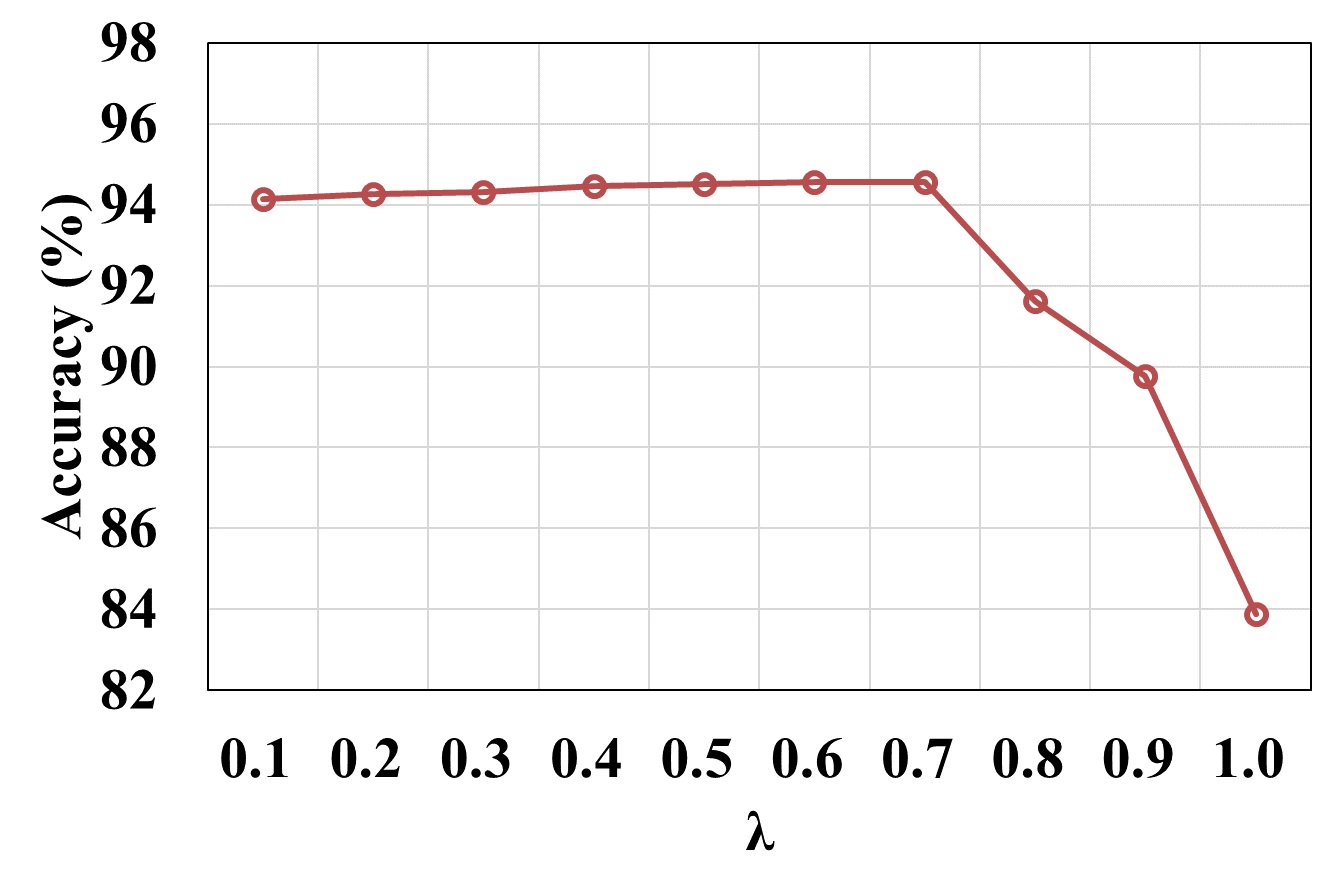}
    \end{minipage}
}
\subfigure[Computers]
{
 	\begin{minipage}[b]{.3\linewidth}
        \centering
        \includegraphics[scale=0.22]{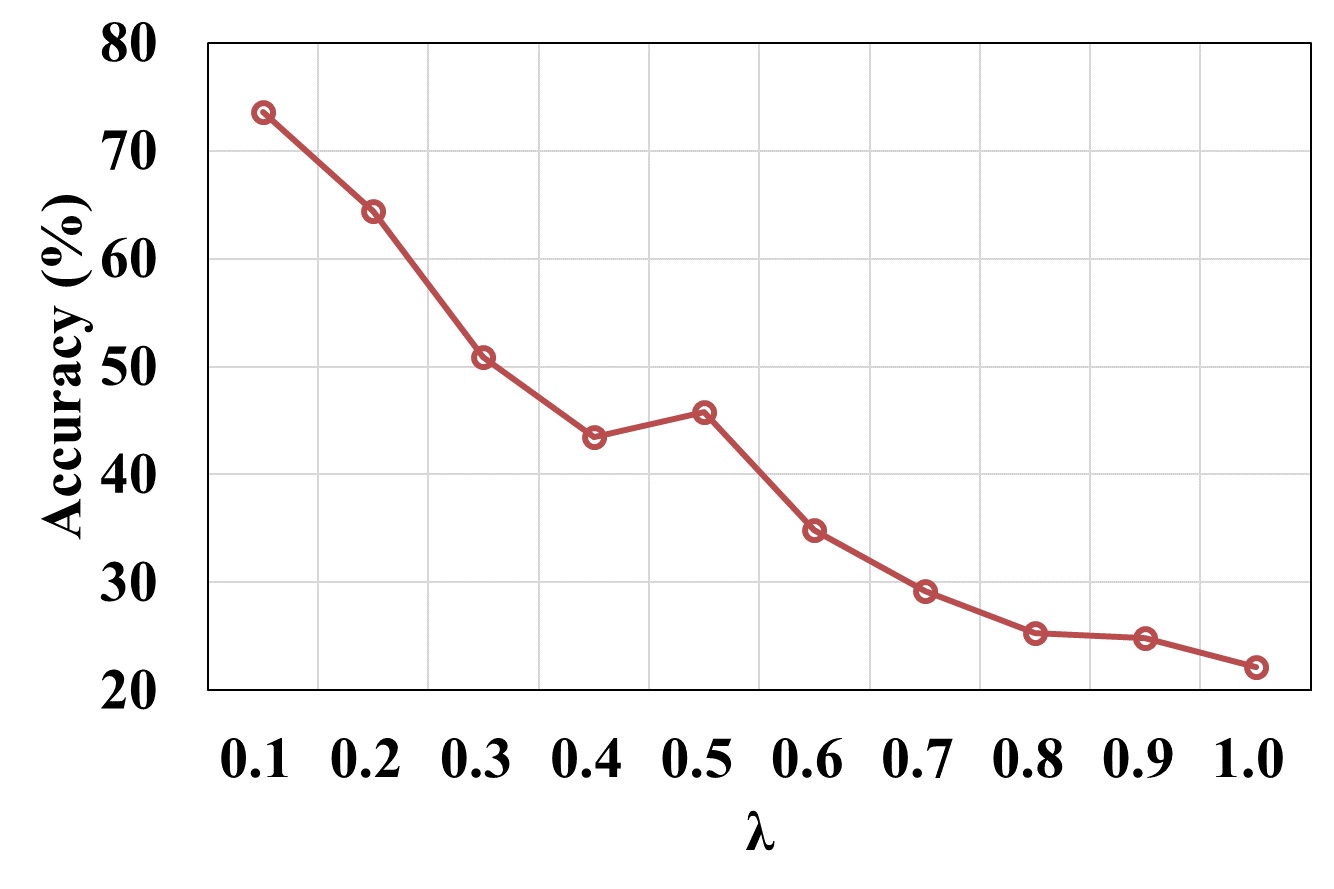}
    \end{minipage}
}
\caption{The relationship between test accuracy and the regularization strength $\lambda$ for GCN across six benchmark datasets.}
\label{GCN_lam}
\end{figure*}

\begin{figure*}[t!]
\centering
\subfigbottomskip=0.05pt
\subfigcapskip=-6pt

\subfigure[Cora]
{
 	\begin{minipage}[b]{.3\linewidth}
        \centering
        \includegraphics[scale=0.22]{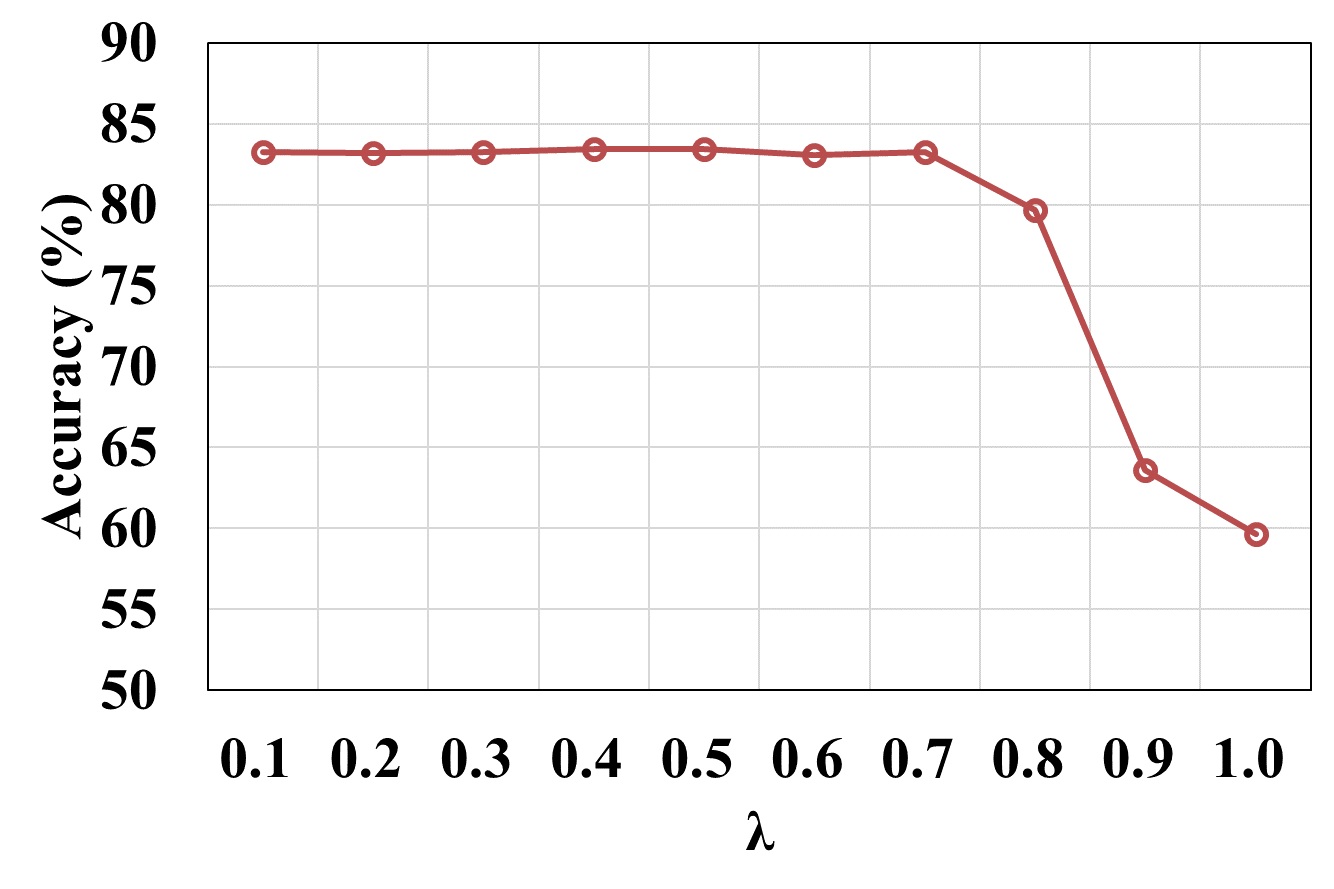}
    \end{minipage}
}
\subfigure[CiteSeer]
{
 	\begin{minipage}[b]{.3\linewidth}
        \centering
        \includegraphics[scale=0.22]{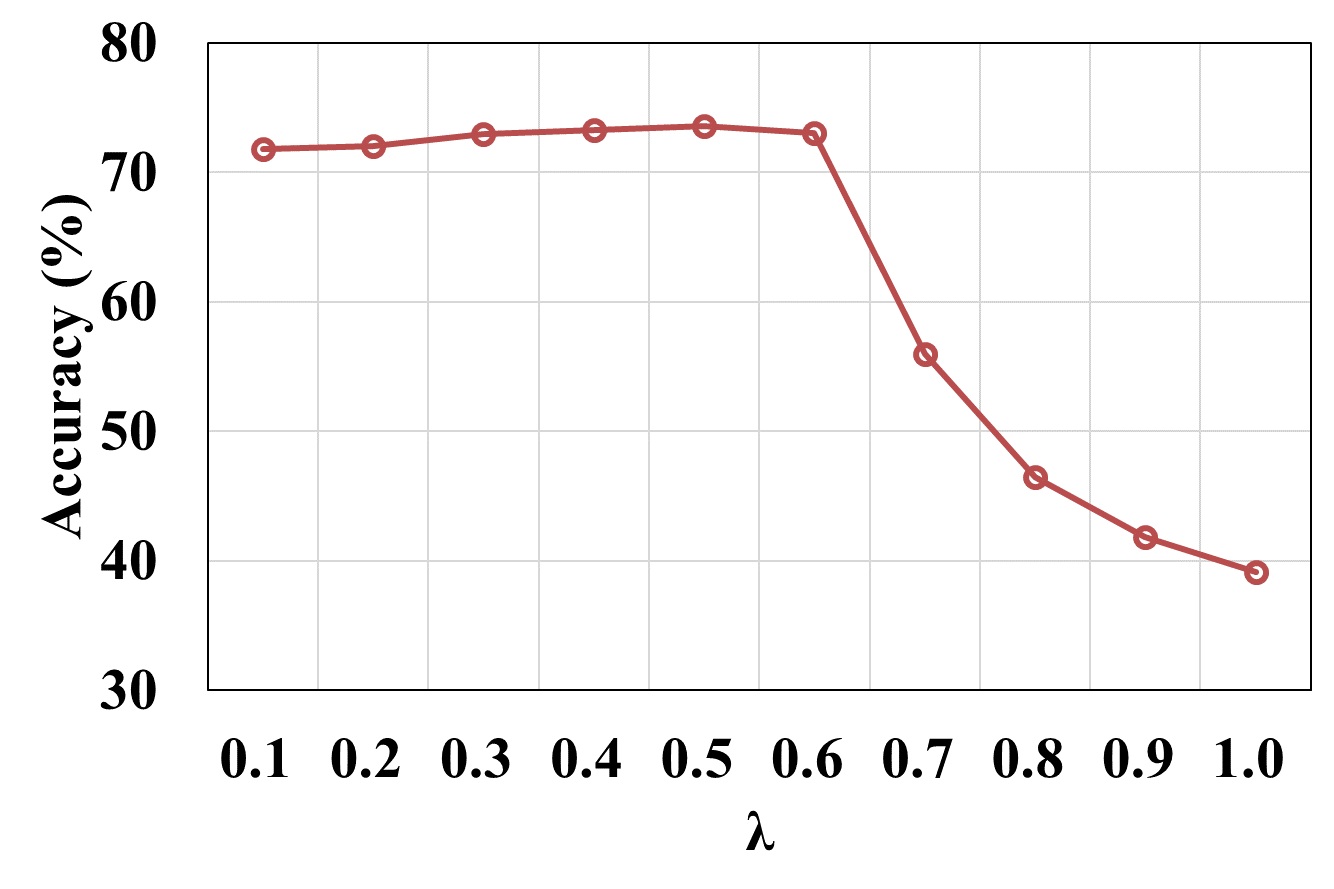}
    \end{minipage}
}
\subfigure[PubMed]
{
 	\begin{minipage}[b]{.3\linewidth}
        \centering
        \includegraphics[scale=0.22]{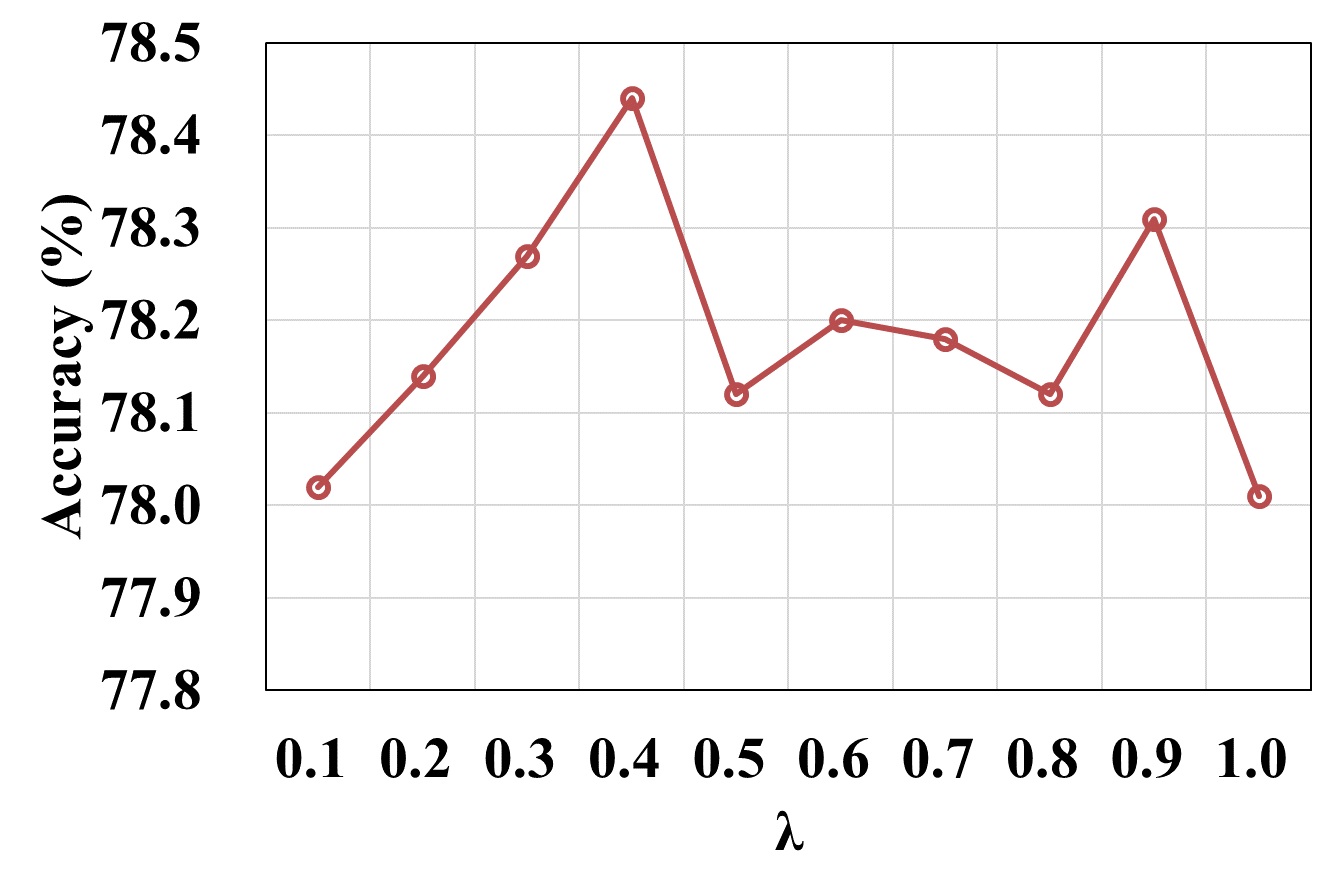}
    \end{minipage}
}
\quad
\subfigure[CS]
{
 	\begin{minipage}[b]{.3\linewidth}
        \centering
        \includegraphics[scale=0.22]{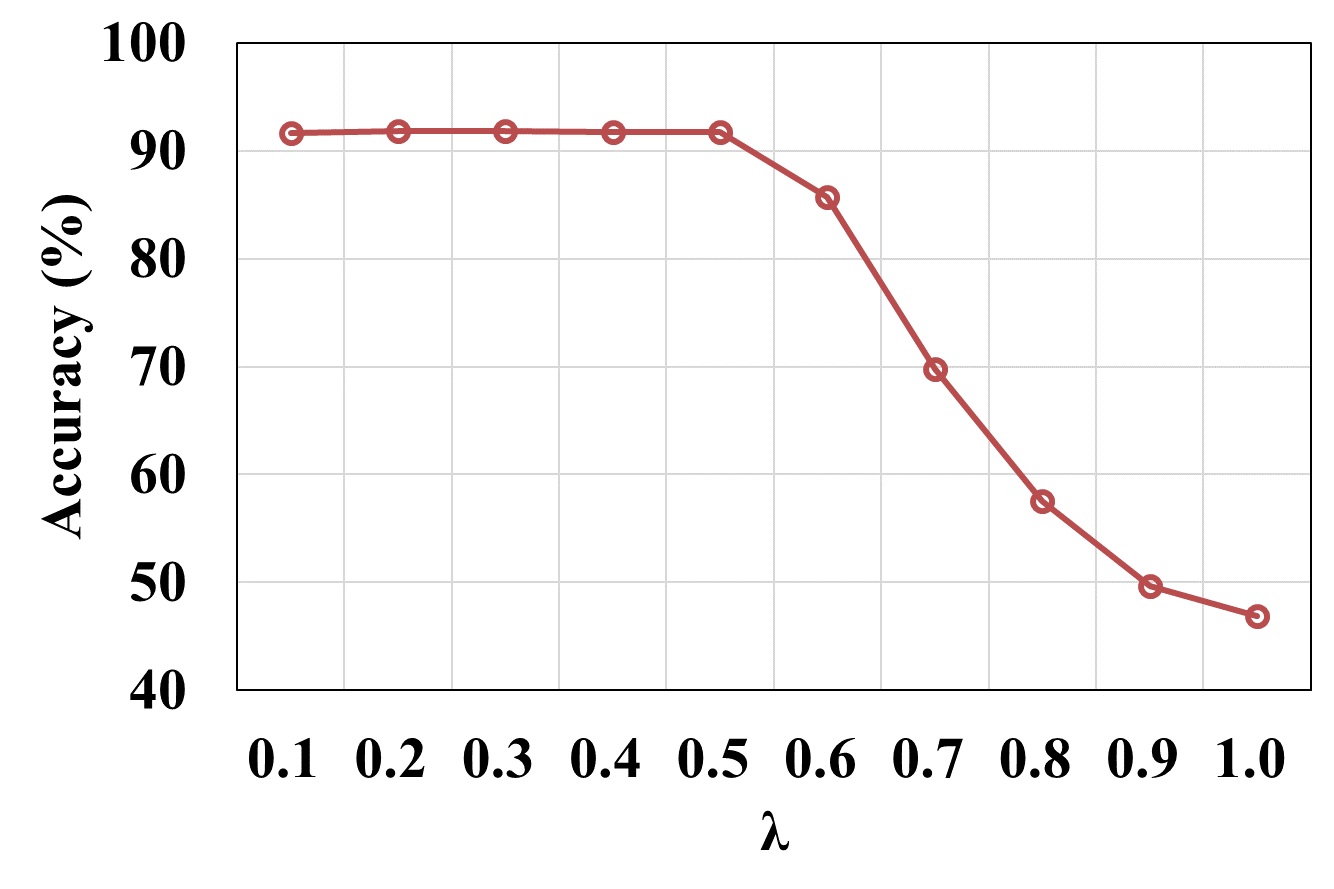}
    \end{minipage}
}
\subfigure[Physics]
{
 	\begin{minipage}[b]{.3\linewidth}
        \centering
        \includegraphics[scale=0.22]{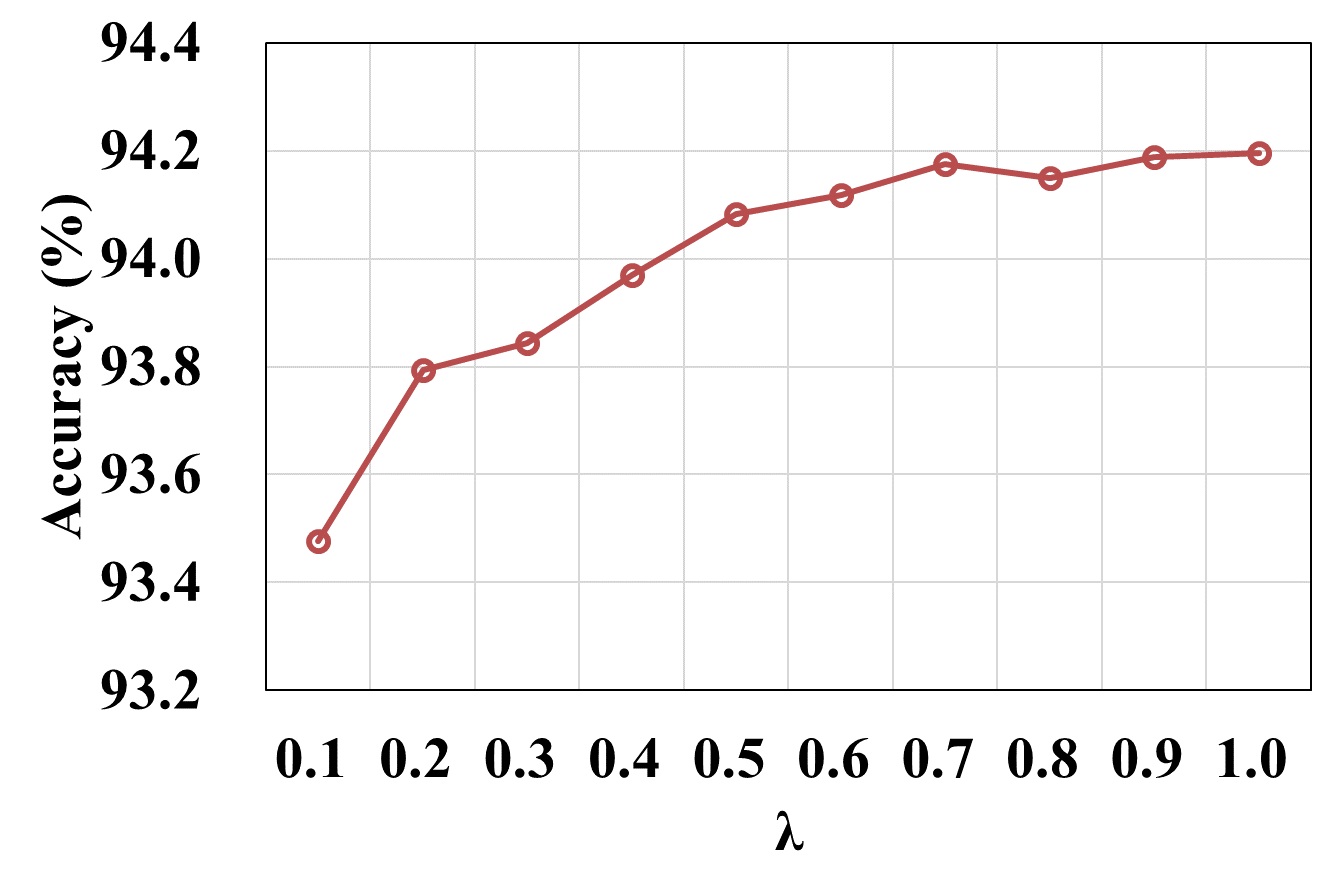}
    \end{minipage}
}
\subfigure[Computers]
{
 	\begin{minipage}[b]{.3\linewidth}
        \centering
        \includegraphics[scale=0.22]{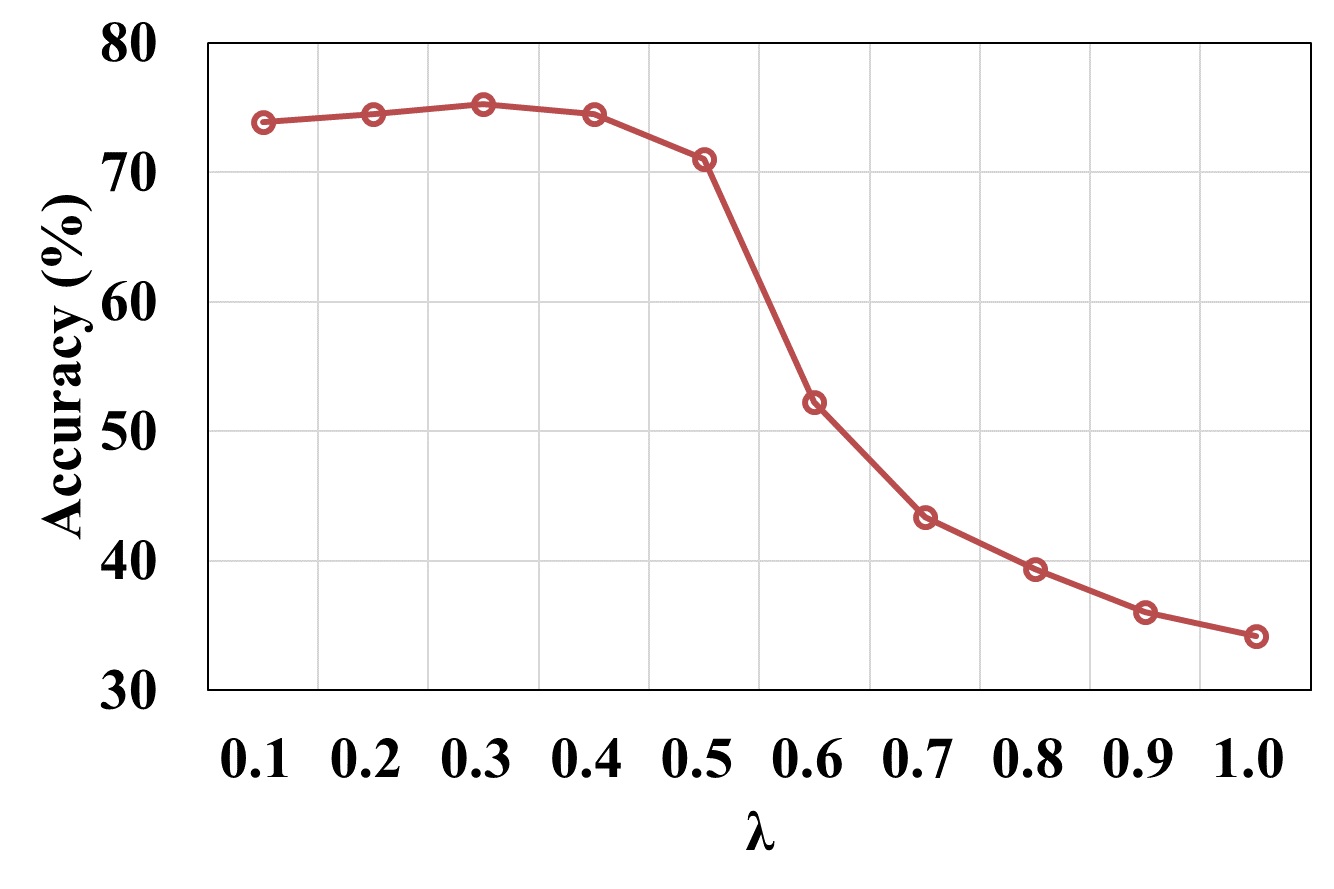}
    \end{minipage}
}
\caption{The relationship between test accuracy and the regularization strength $\lambda$ for GAT across six benchmark datasets.}
\label{GAT_lam}
\end{figure*}

\begin{figure*}[t!]
\centering
\subfigbottomskip=0.05pt
\subfigcapskip=-6pt

\subfigure[Cora]
{
 	\begin{minipage}[b]{.3\linewidth}
        \centering
        \includegraphics[scale=0.22]{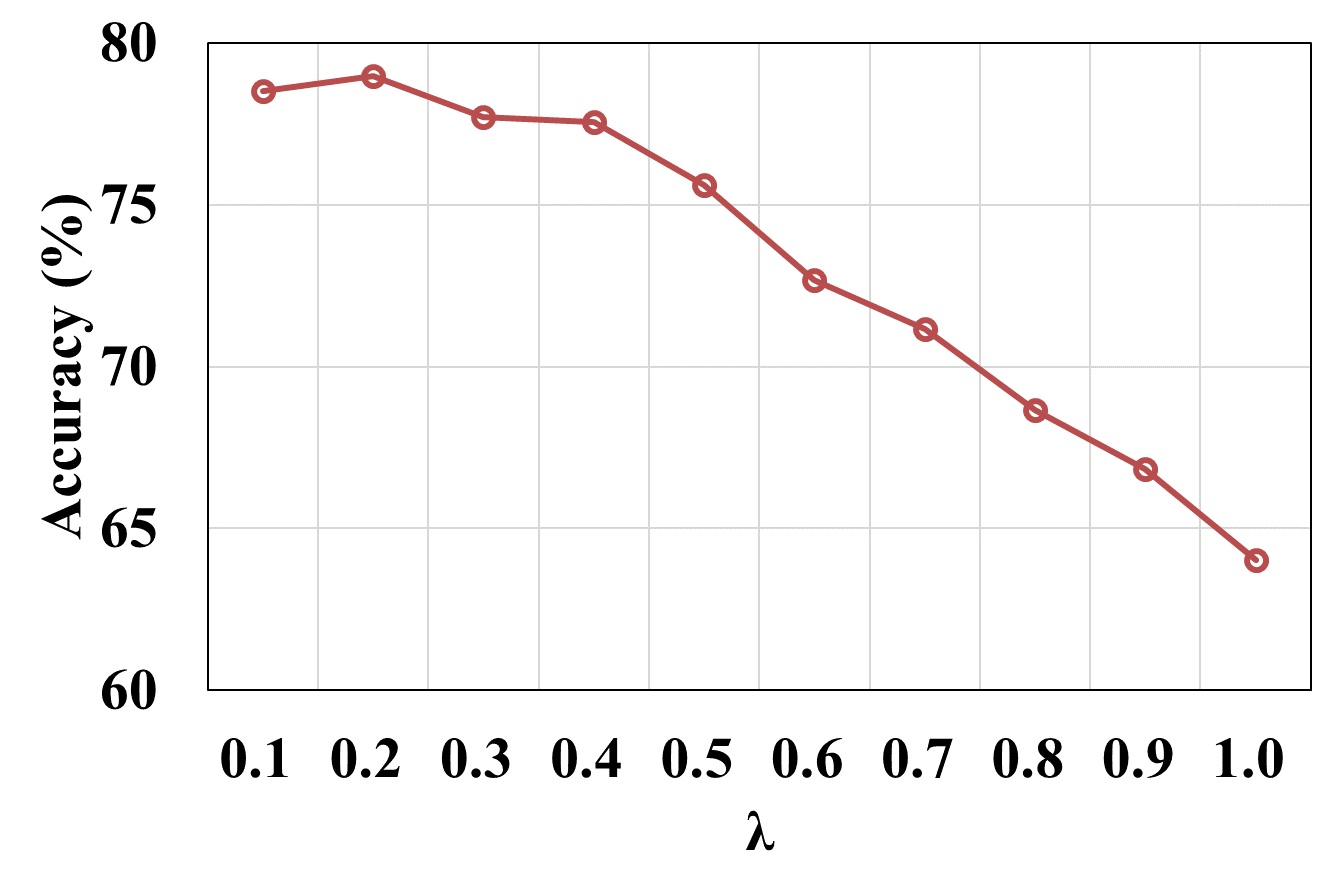}
    \end{minipage}
}
\subfigure[CiteSeer]
{
 	\begin{minipage}[b]{.3\linewidth}
        \centering
        \includegraphics[scale=0.22]{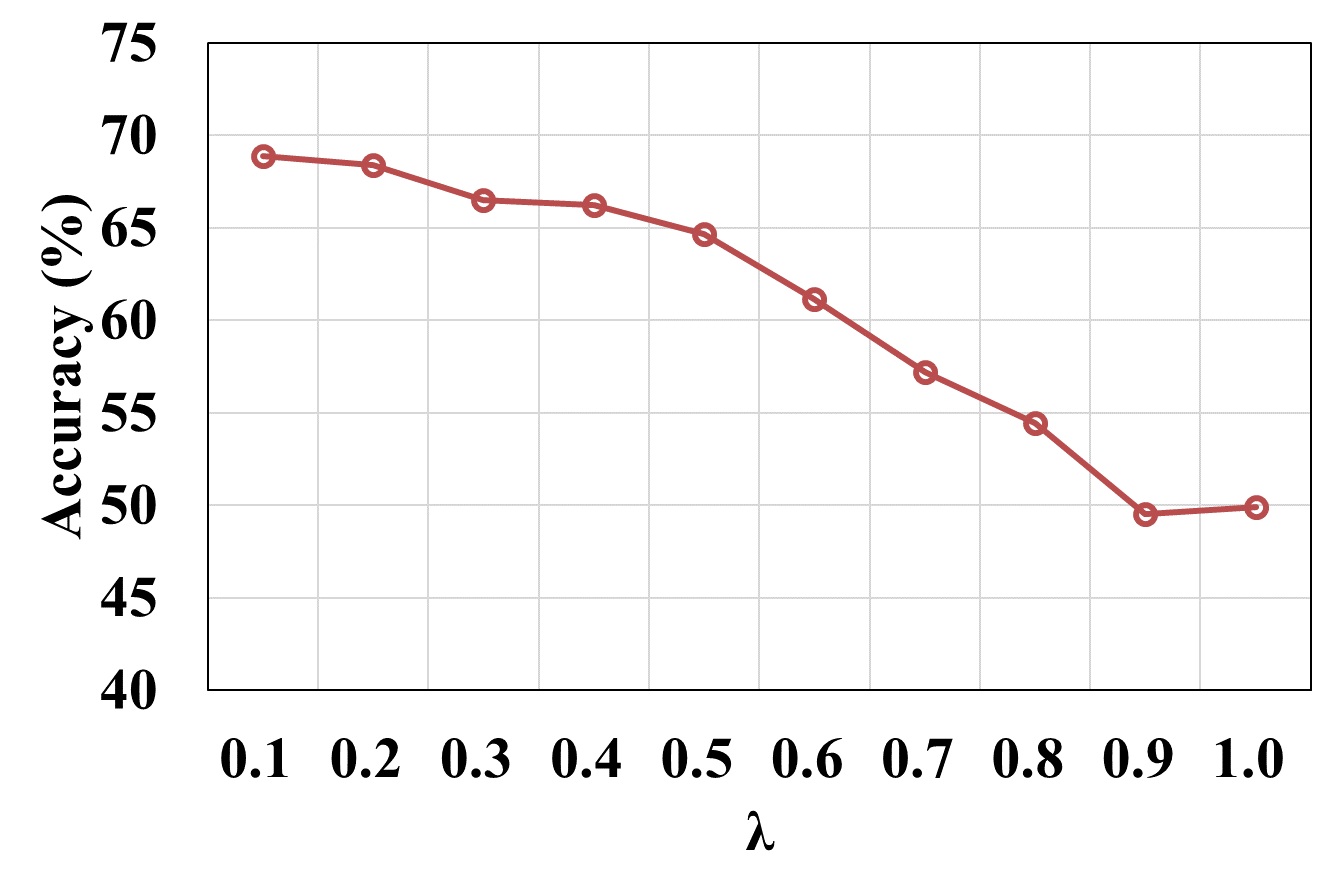}
    \end{minipage}
}
\subfigure[PubMed]
{
 	\begin{minipage}[b]{.3\linewidth}
        \centering
        \includegraphics[scale=0.22]{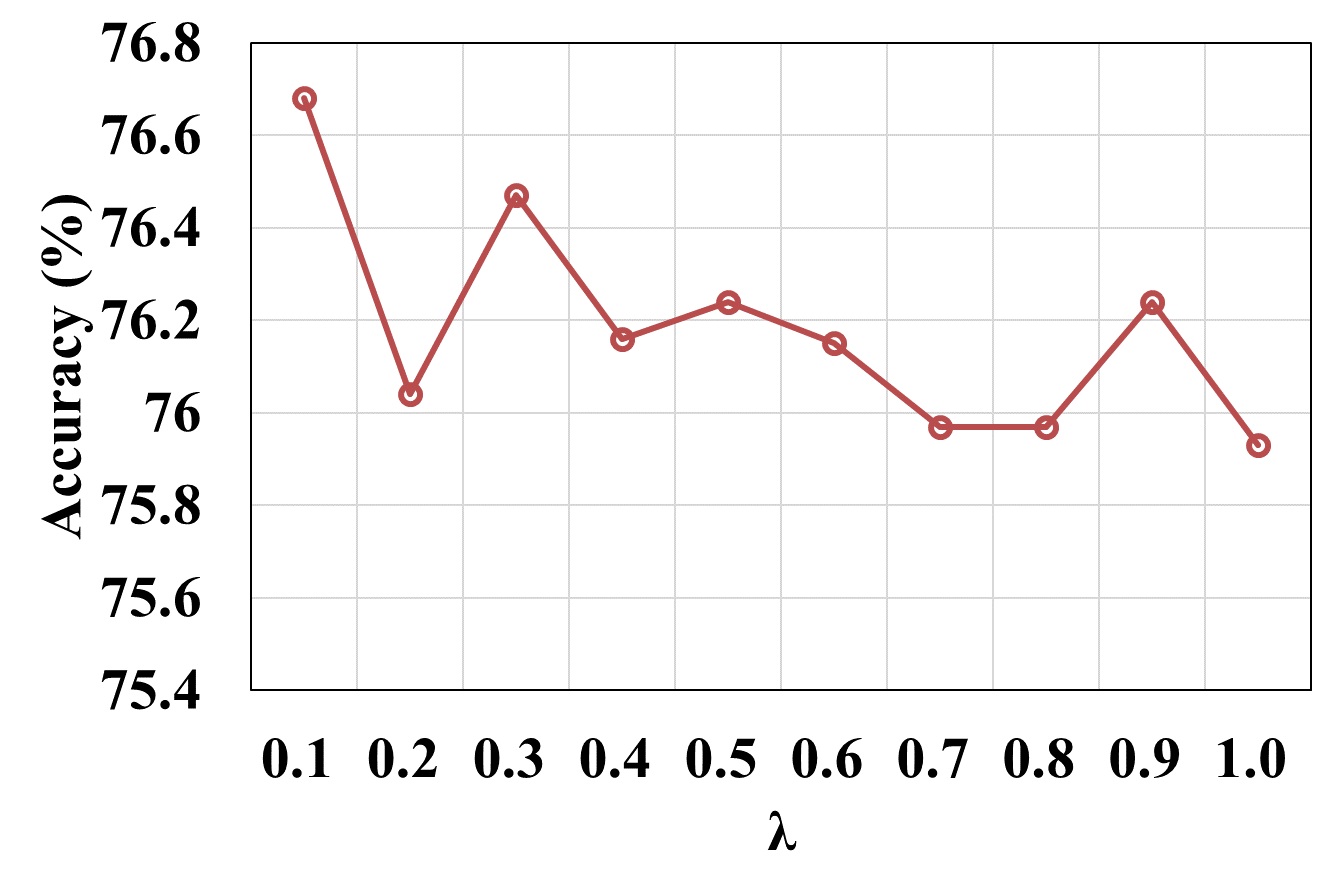}
    \end{minipage}
}
\quad
\subfigure[CS]
{
 	\begin{minipage}[b]{.3\linewidth}
        \centering
        \includegraphics[scale=0.22]{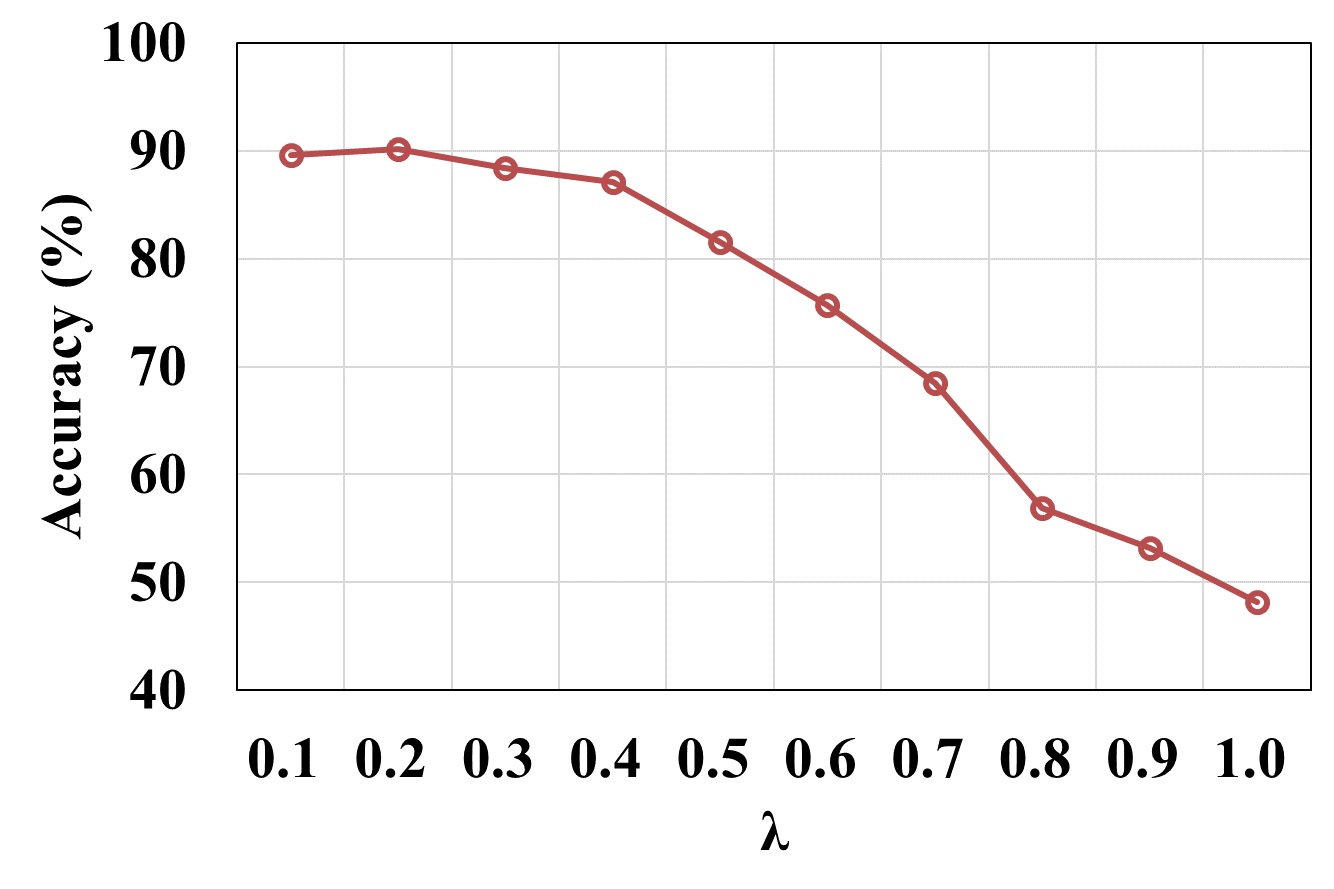}
    \end{minipage}
}
\subfigure[Physics]
{
 	\begin{minipage}[b]{.3\linewidth}
        \centering
        \includegraphics[scale=0.22]{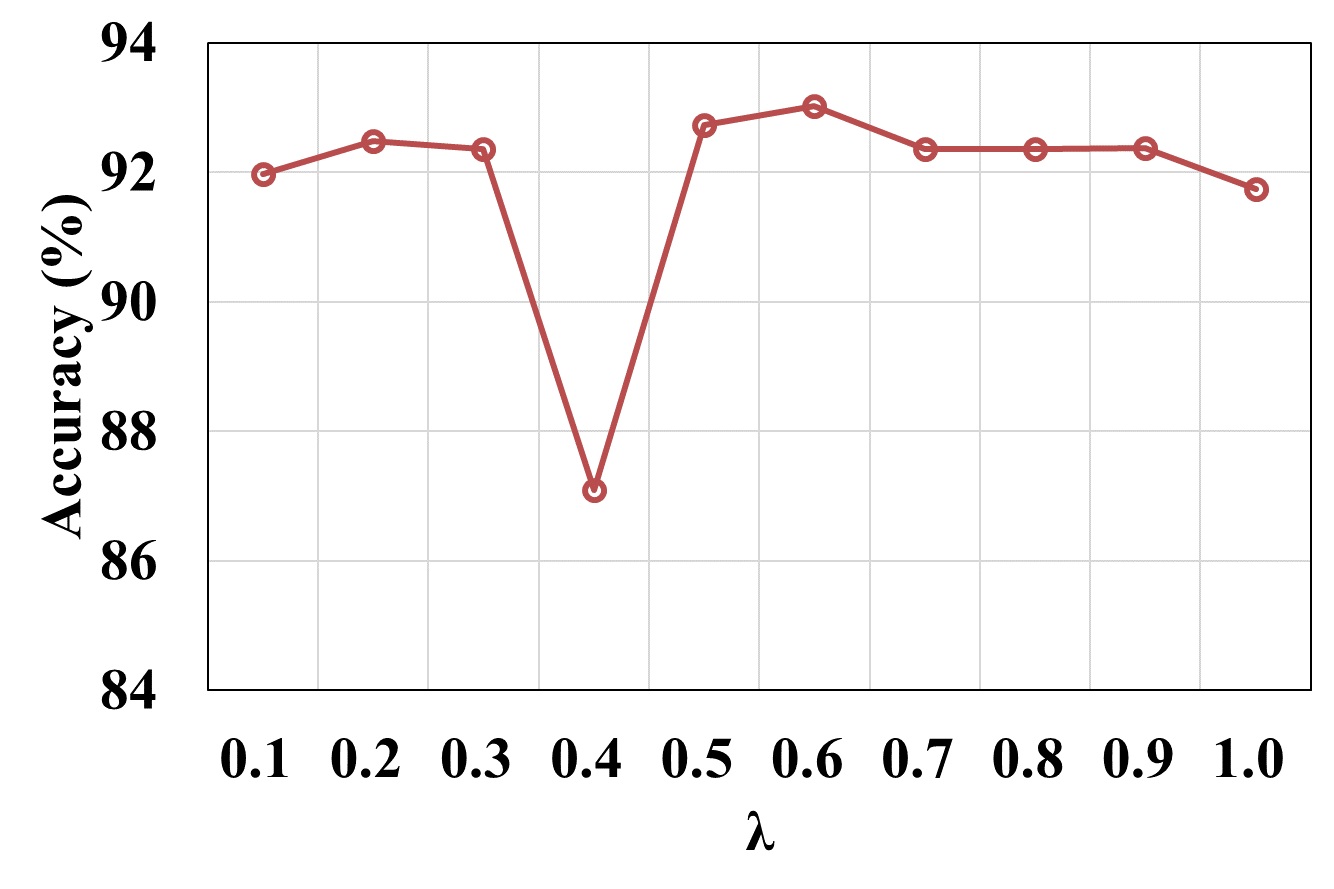}
    \end{minipage}
}
\subfigure[Computers]
{
 	\begin{minipage}[b]{.3\linewidth}
        \centering
        \includegraphics[scale=0.22]{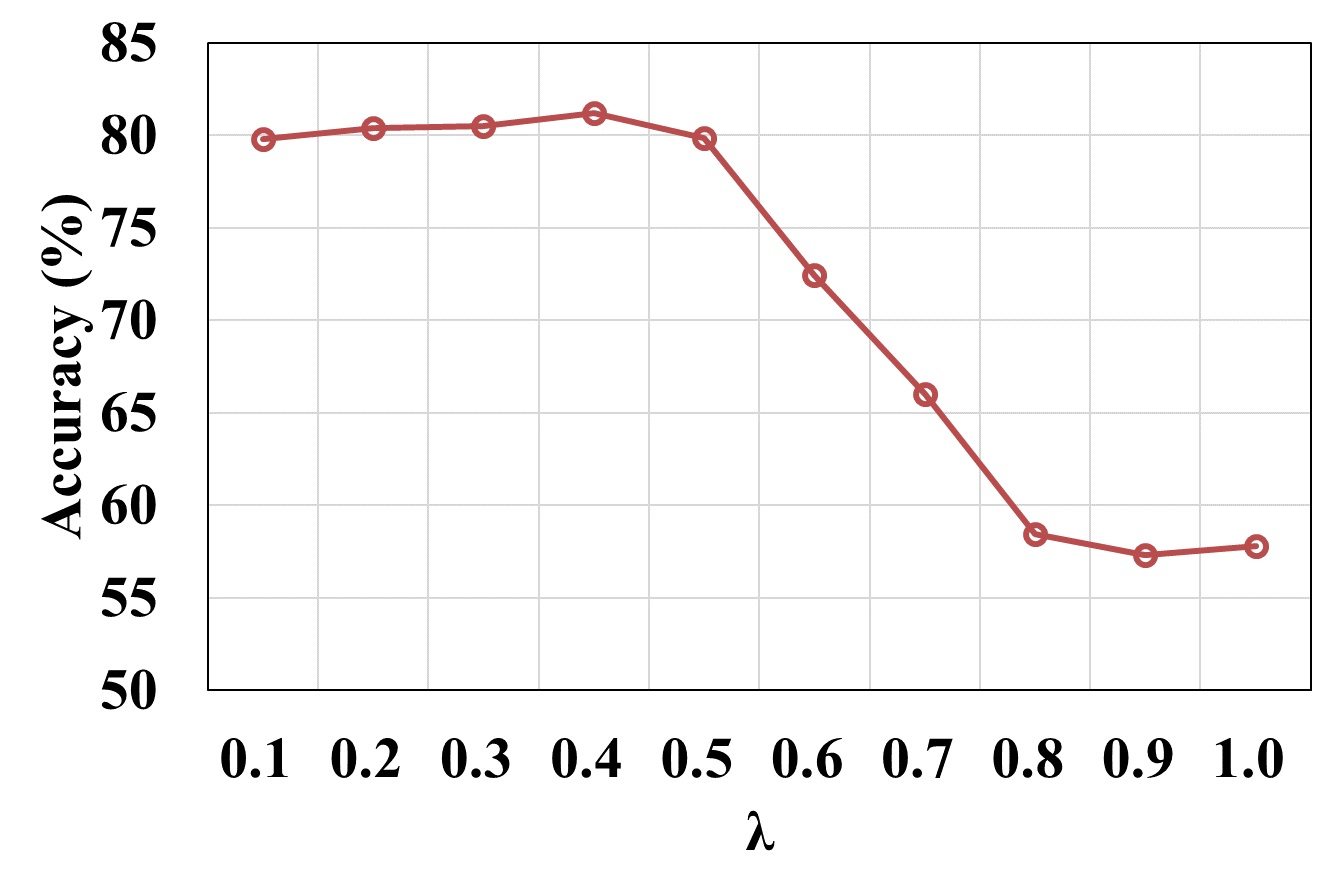}
    \end{minipage}
}
\caption{The relationship between test accuracy and the regularization strength $\lambda$ for GT across six benchmark datasets.}
\label{GT_lam}
\end{figure*}

\subsection{Ablation Study}

To further understand the contribution of each component in the proposed SER, we conduct an ablation study by progressively removing key design elements. The results are summarized in Figure \ref{abl_GCN}, Figure \ref{abl_GAT} and Figure \ref{abl_GT}, respectively. (1) w/o SER. This variant corresponds to the original backbone models without any structural entropy regularization. It serves as the baseline to evaluate the overall effectiveness of SER. The results show that removing SER consistently degrades performance across all architectures, confirming that explicitly regulating structural complexity is beneficial for generalization. (2) SER w/o class-level encoding tree. In this variant, we remove the class-level partition in the three-level encoding tree and apply entropy regularization directly on the graph structure without distinguishing intra-class and inter-class connections. This reduces the regularizer to a generic structural entropy term. The results show a clear performance drop compared to the full SER, indicating that class-level information is essential for effectively suppressing cross-class interactions and improving generalization. (3) SER w/o effective-edge control. In this setting, we retain the entropy-based regularization but disable its effect on updating the aggregation structure, so that it no longer influences the learned message-passing weights. The results show consistent performance degradation compared to full SER, indicating that its effectiveness does not come from introducing an additional entropy term alone, but from its interaction with the learnable aggregation process. In other words, SER improves generalization by regulating the effective edges involved in message passing, rather than merely adding an additional entropy-based penalty. Overall, the ablation results verify that each component of SER is essential. The class-aware encoding enables semantic alignment, while the effective-edge control mechanism allows the model to adaptively refine its connectivity structure.

\begin{figure*}[t!]
\centering
\subfigbottomskip=0.05pt
\subfigcapskip=-6pt

    \rotatebox{90}{\scriptsize~~~~~~~~~~~~~~~~~~~GT~~~~~~~~~~~~~~~~~~~~~~~~~~~~~~~~~~GAT~~~~~~~~~~~~~~~~~~~~~~~~~~~~~~~~~GCN}
\subfigure[Cora]
{
 	\begin{minipage}[b]{.3\linewidth}
        \centering
        \includegraphics[width=1.04\linewidth,height=3.9cm]{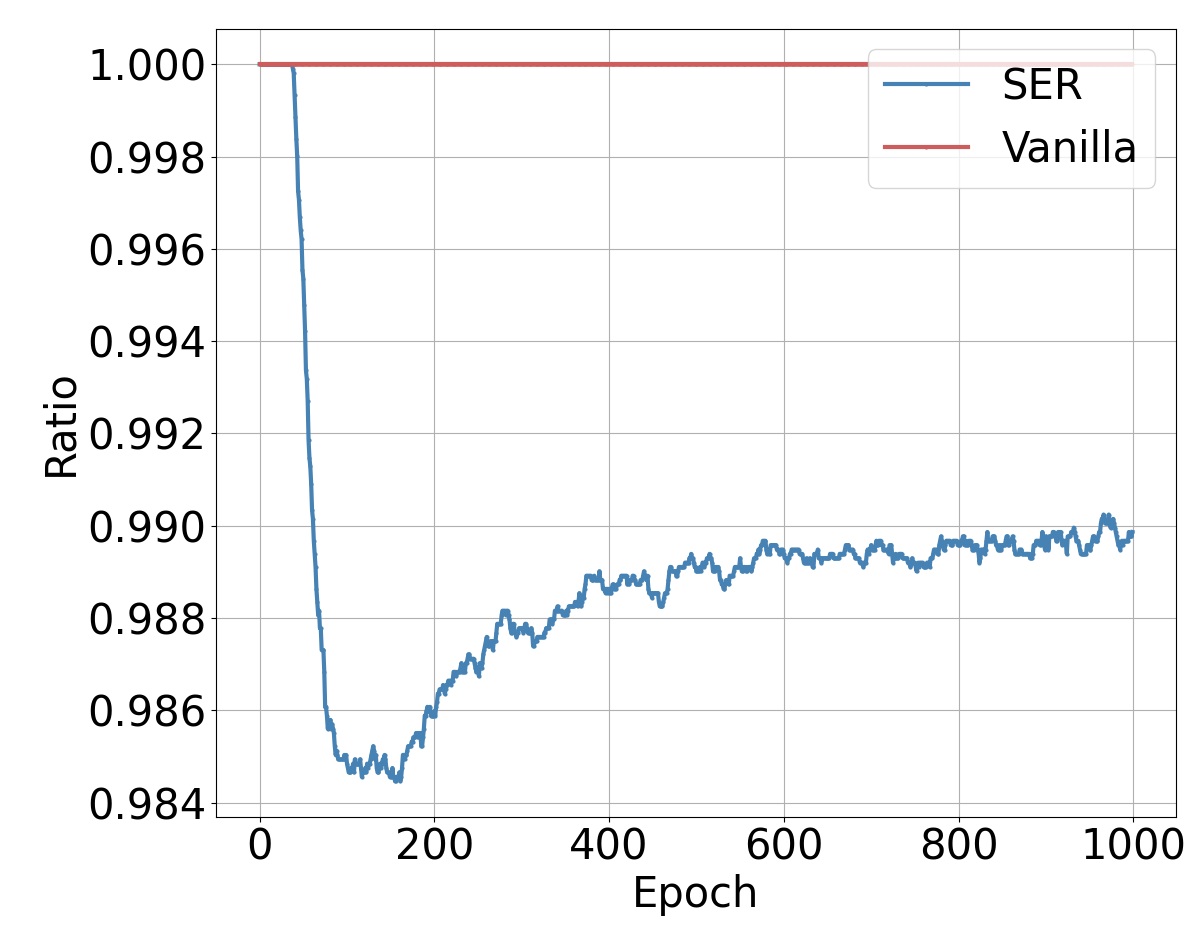}
        \includegraphics[width=1.04\linewidth,height=3.9cm]{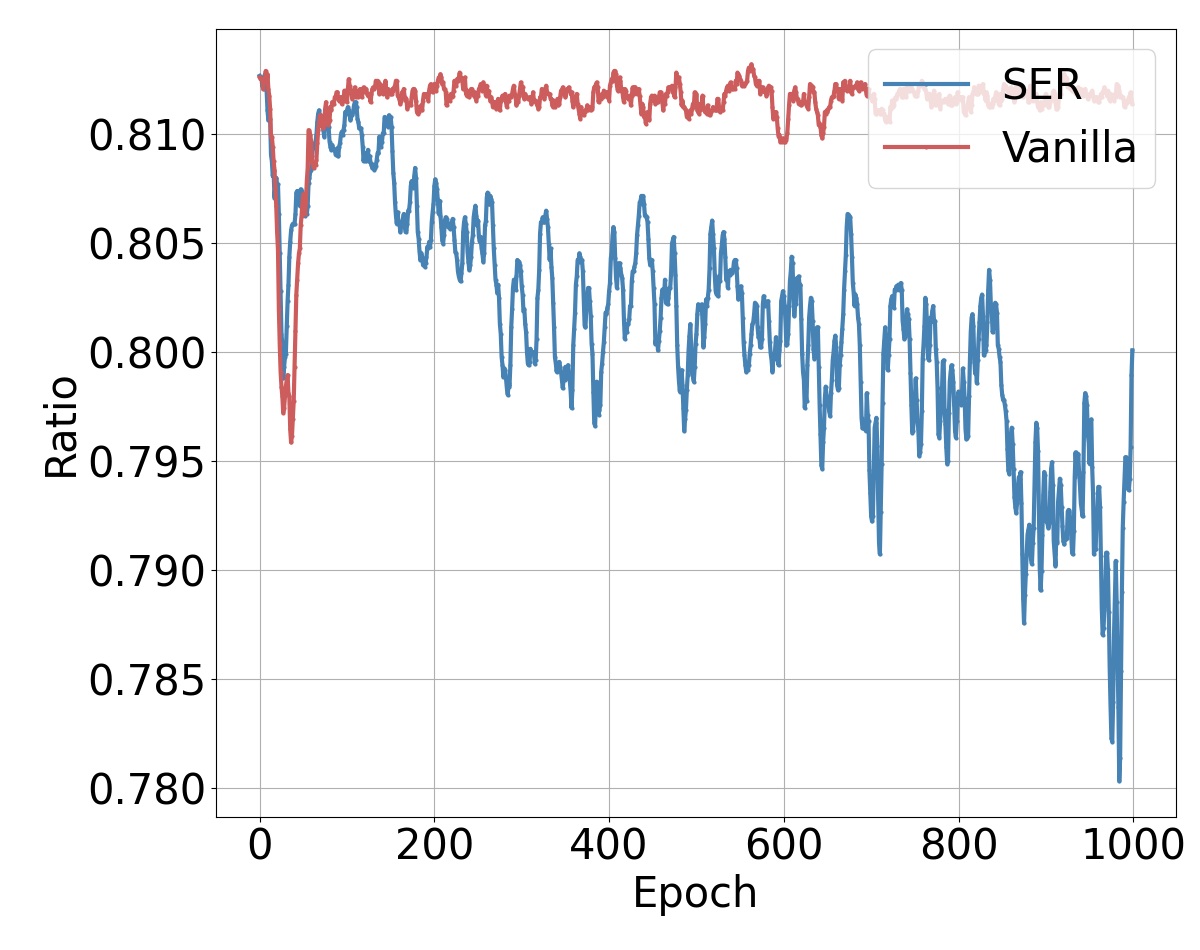}        \includegraphics[width=1.04\linewidth,height=3.9cm]{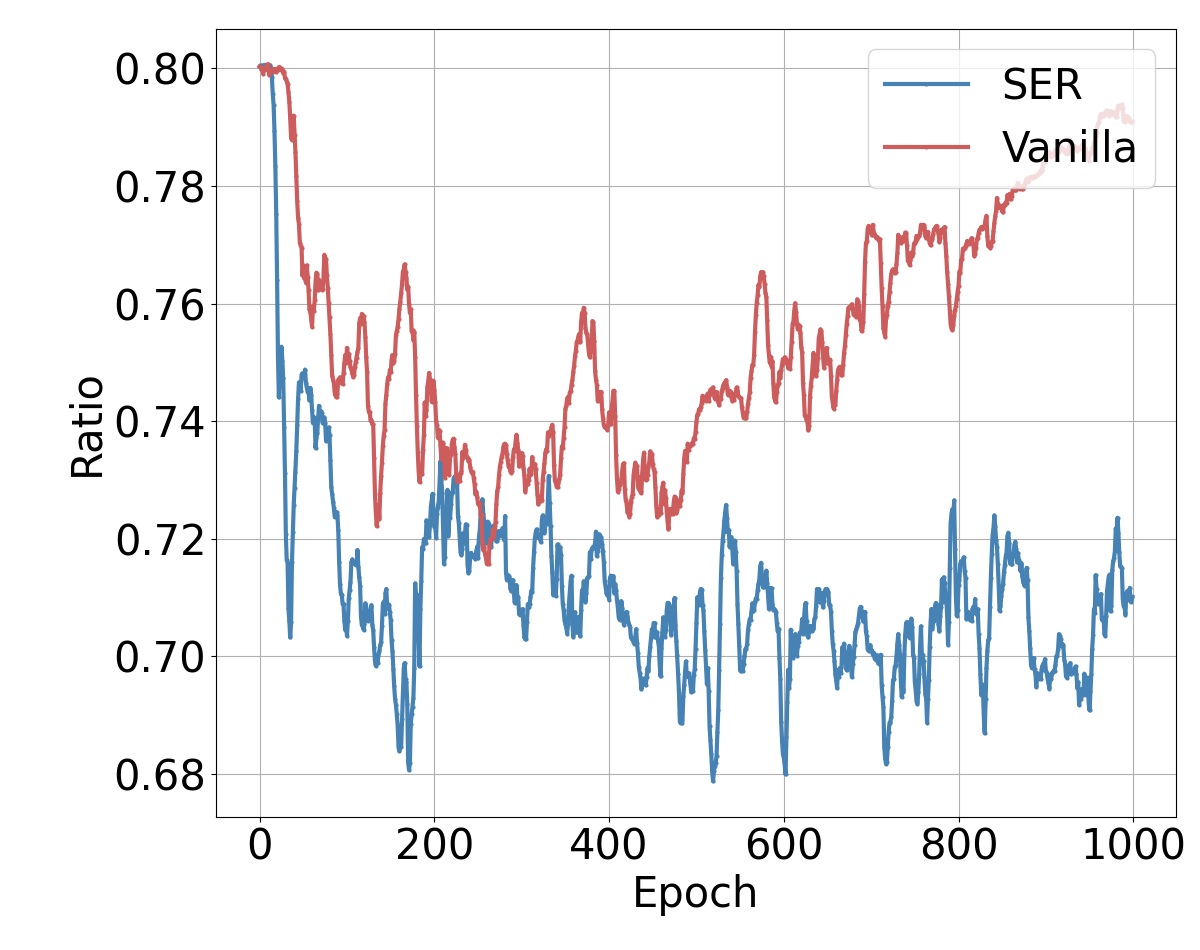}
    \end{minipage}
}
\subfigure[CiteSeer]
{
 	\begin{minipage}[b]{.3\linewidth}
        \centering
        \includegraphics[width=1.04\linewidth,height=3.9cm]{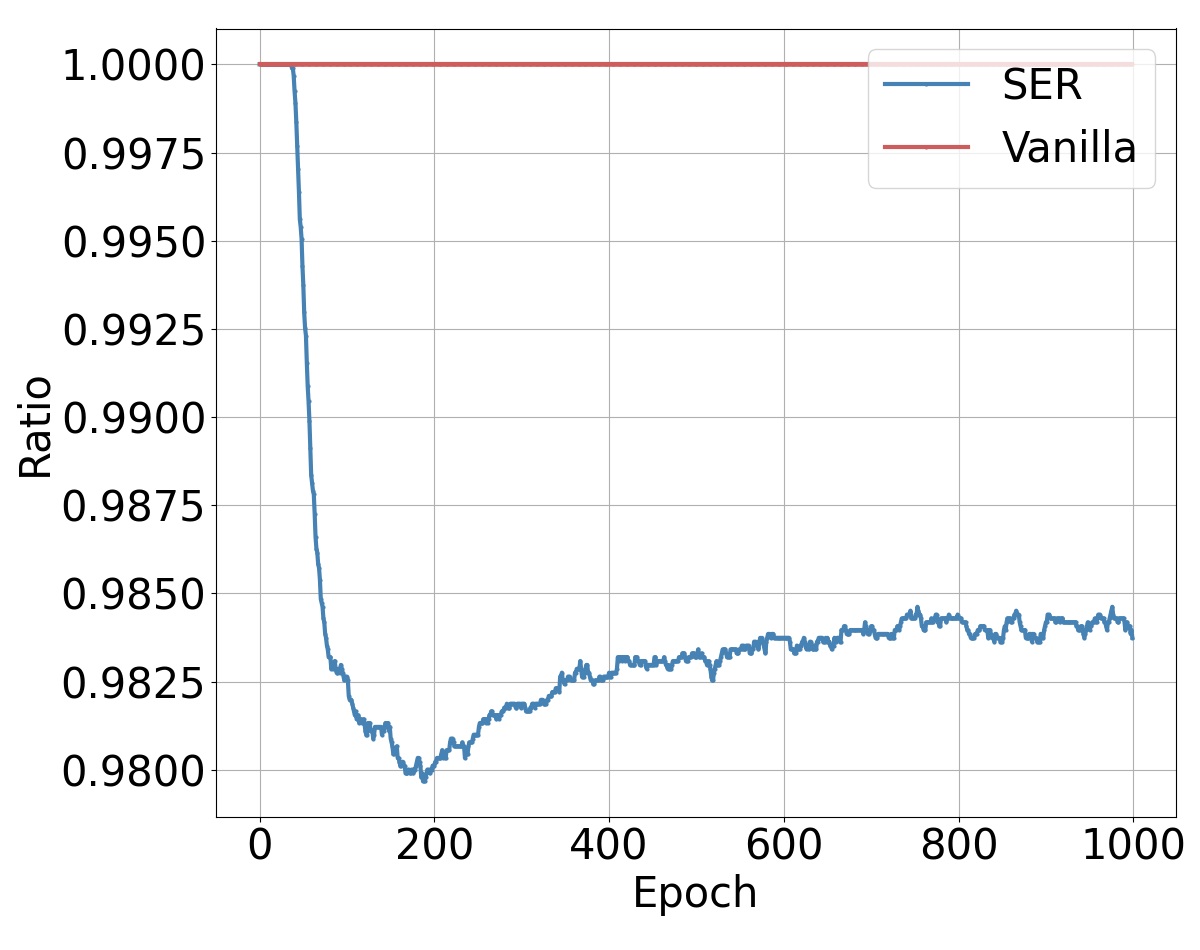}
        \includegraphics[width=1.04\linewidth,height=3.9cm]{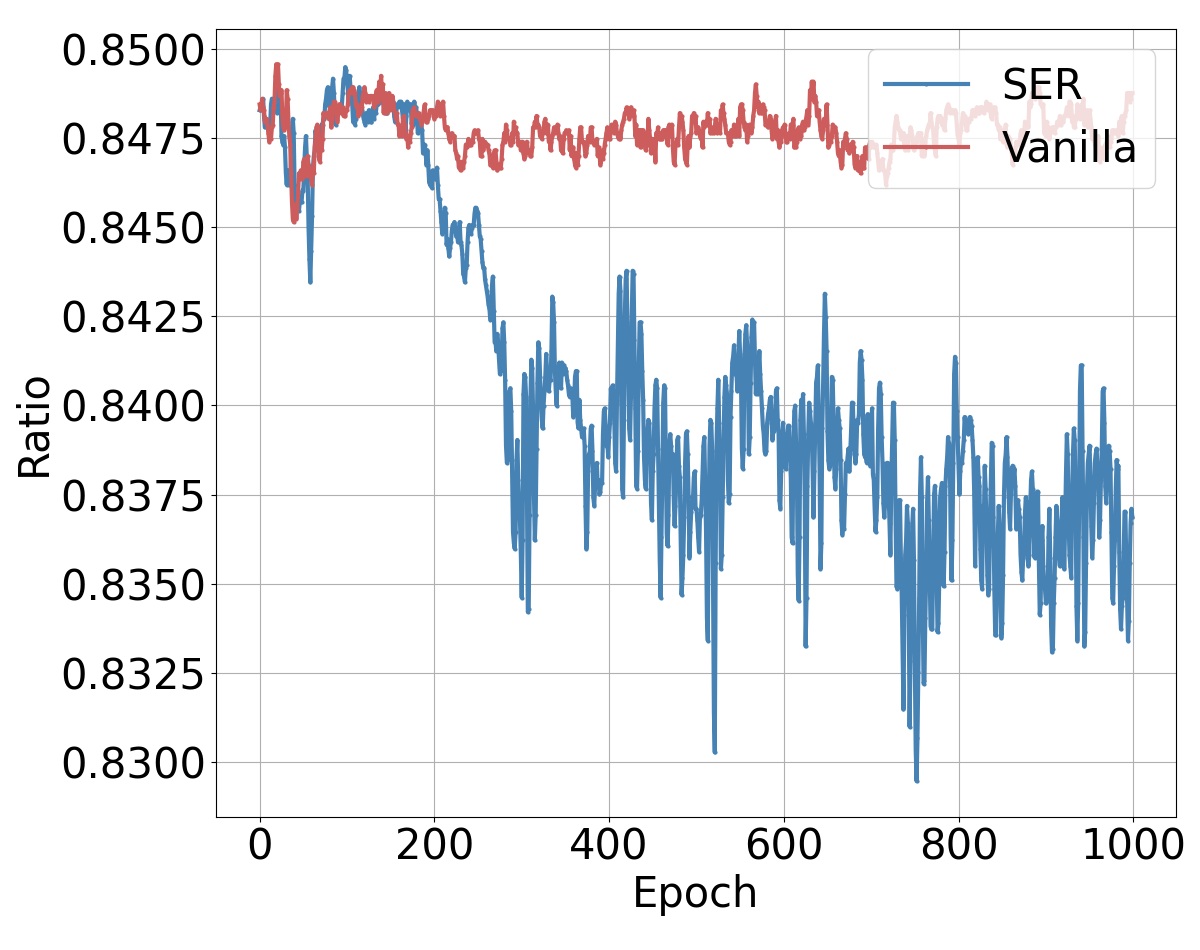}
        \includegraphics[width=1.04\linewidth,height=3.9cm]{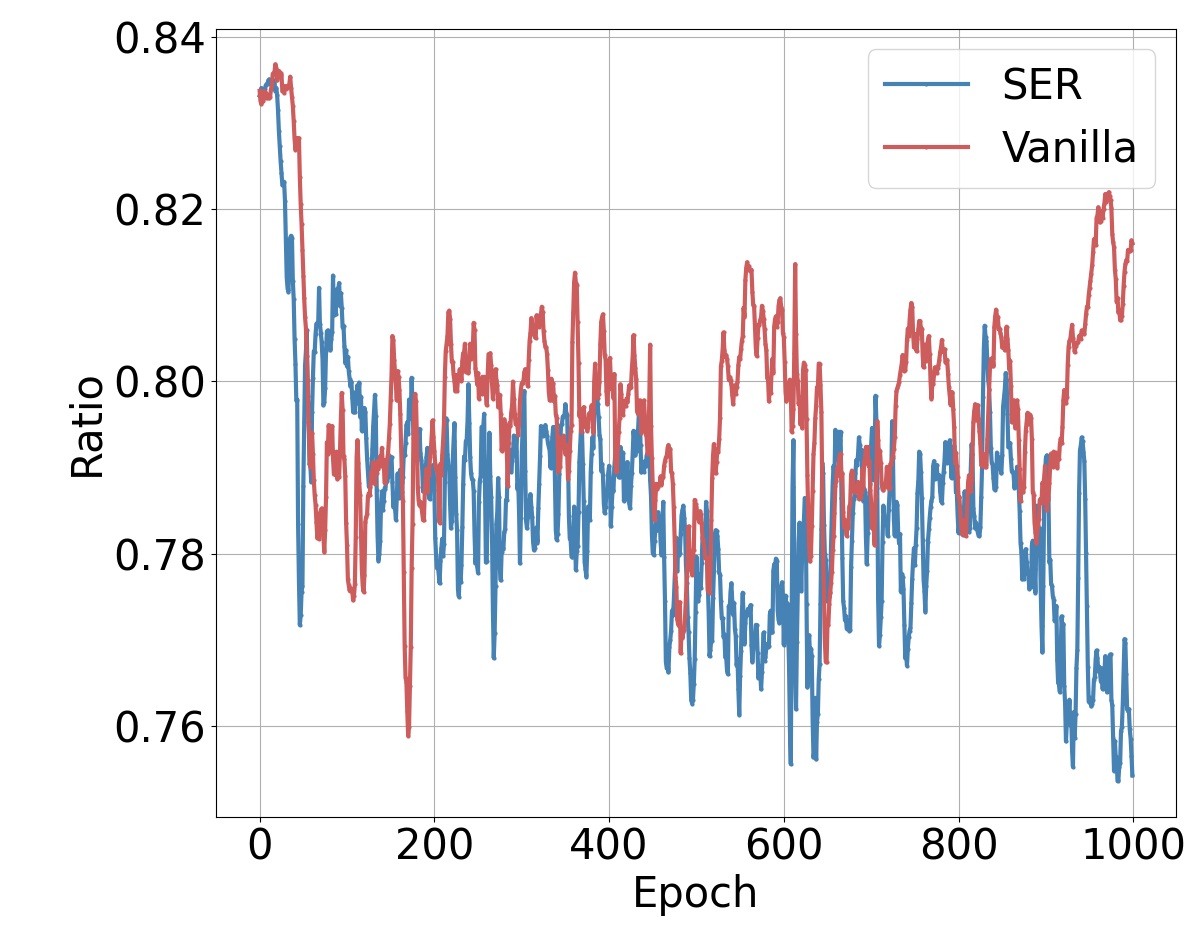}
    \end{minipage}
}
\subfigure[CS]
{
 	\begin{minipage}[b]{.3\linewidth}
        \centering
        \includegraphics[width=1.04\linewidth,height=3.9cm]{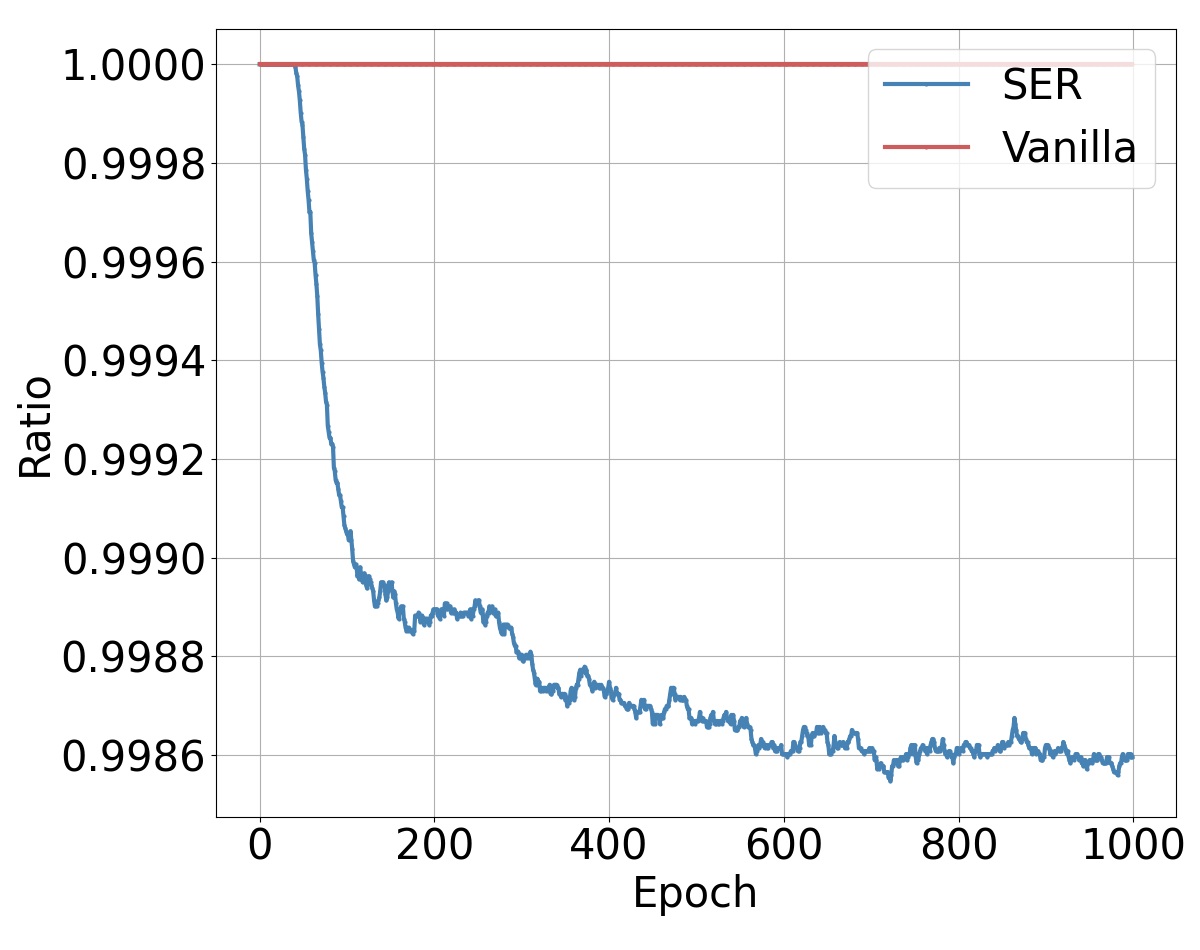}
        \includegraphics[width=1.04\linewidth,height=3.9cm]{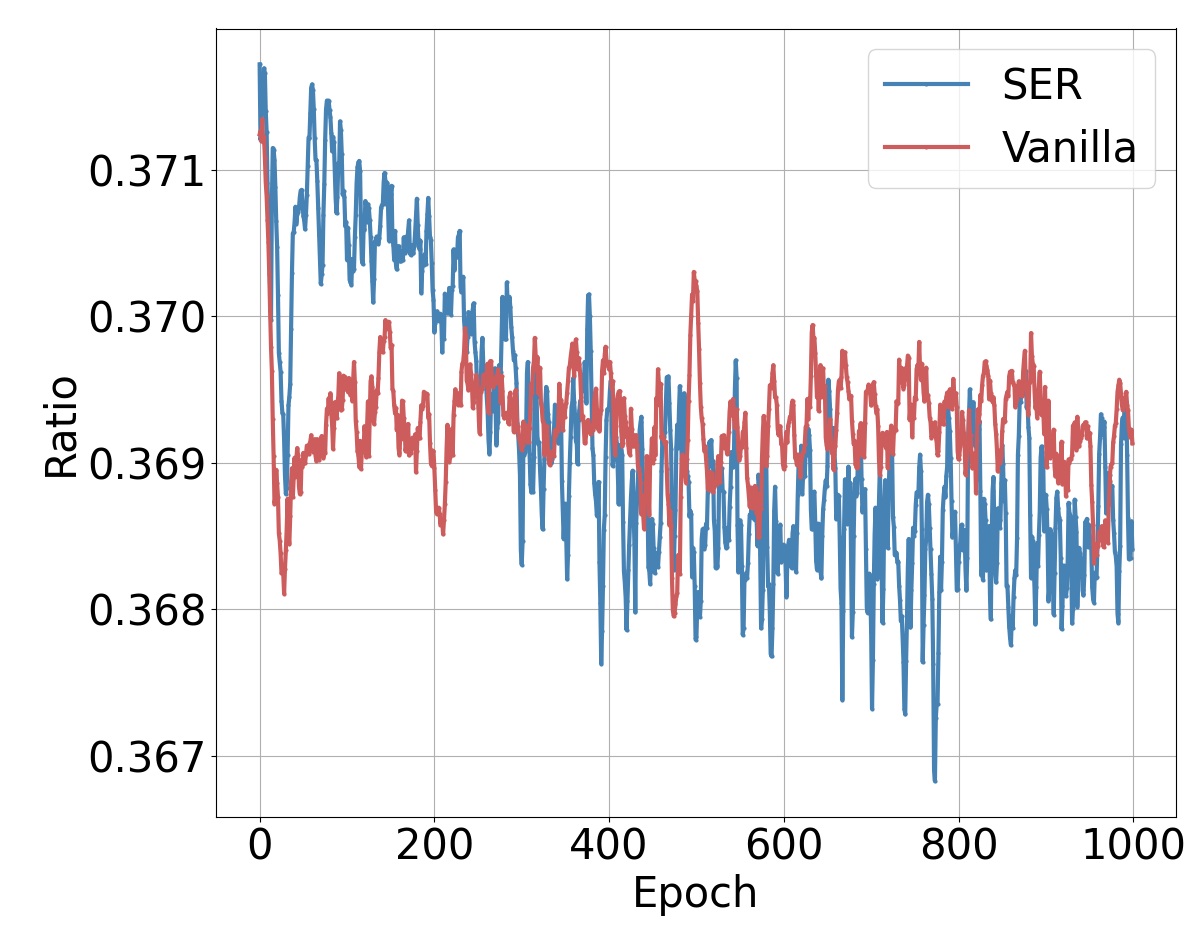}
        \includegraphics[width=1.04\linewidth,height=3.9cm]{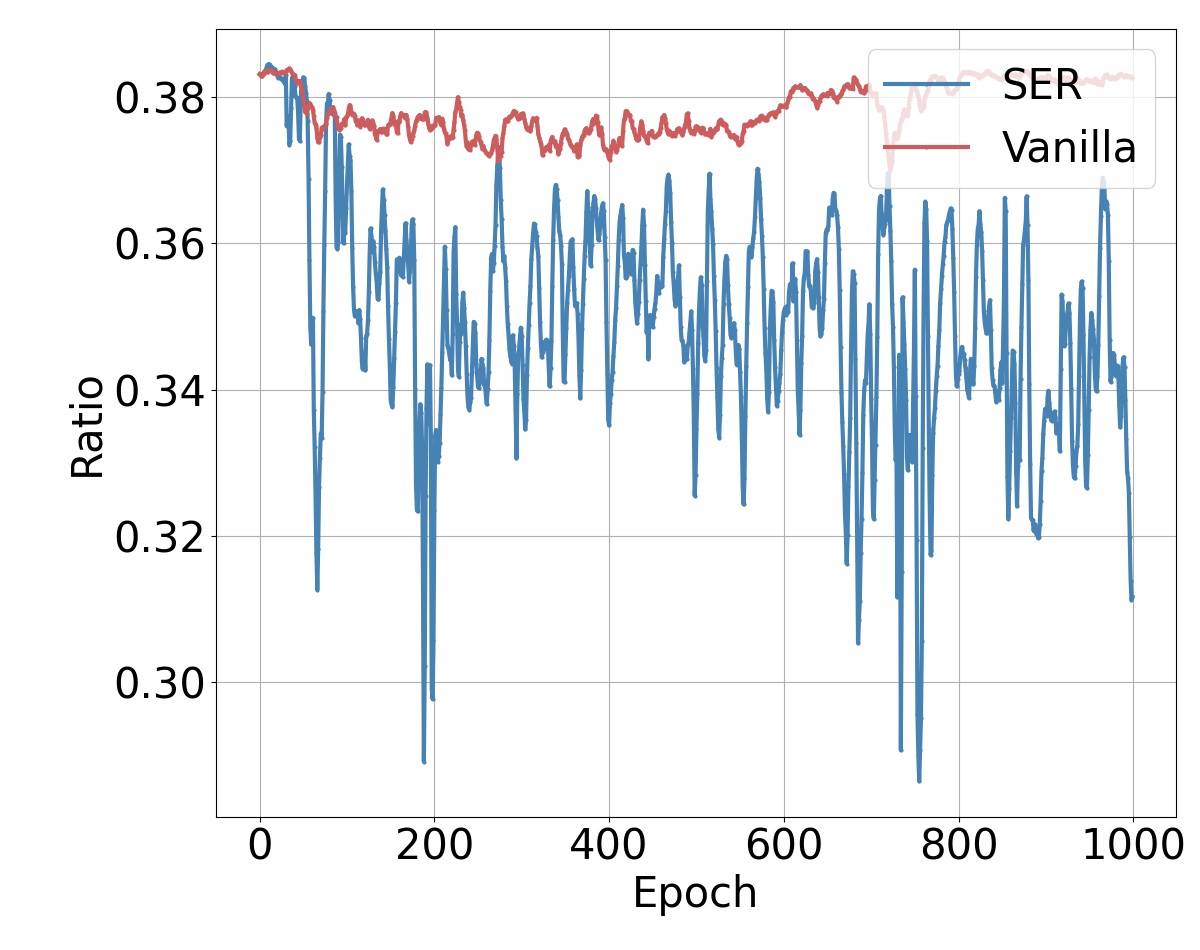}
    \end{minipage}
}
\caption{Evolution of the effective edge ratio over training epochs. The rows from top to bottom correspond to GCN, GAT, and GT backbones, respectively. The red curves denote the vanilla models, while the blue curves represent models with SER.}
\label{eff_edge}
\end{figure*}

\begin{figure*}[t!]
\centering
\subfigbottomskip=0.05pt
\subfigcapskip=-6pt

    \rotatebox{90}{\scriptsize~~~~~~~~~~~~~~~~~~~GT~~~~~~~~~~~~~~~~~~~~~~~~~~~~~~~~~~GAT~~~~~~~~~~~~~~~~~~~~~~~~~~~~~~~~~GCN}
\subfigure[Cora]
{
 	\begin{minipage}[b]{.3\linewidth}
        \centering
        \includegraphics[width=1.04\linewidth,height=3.9cm]{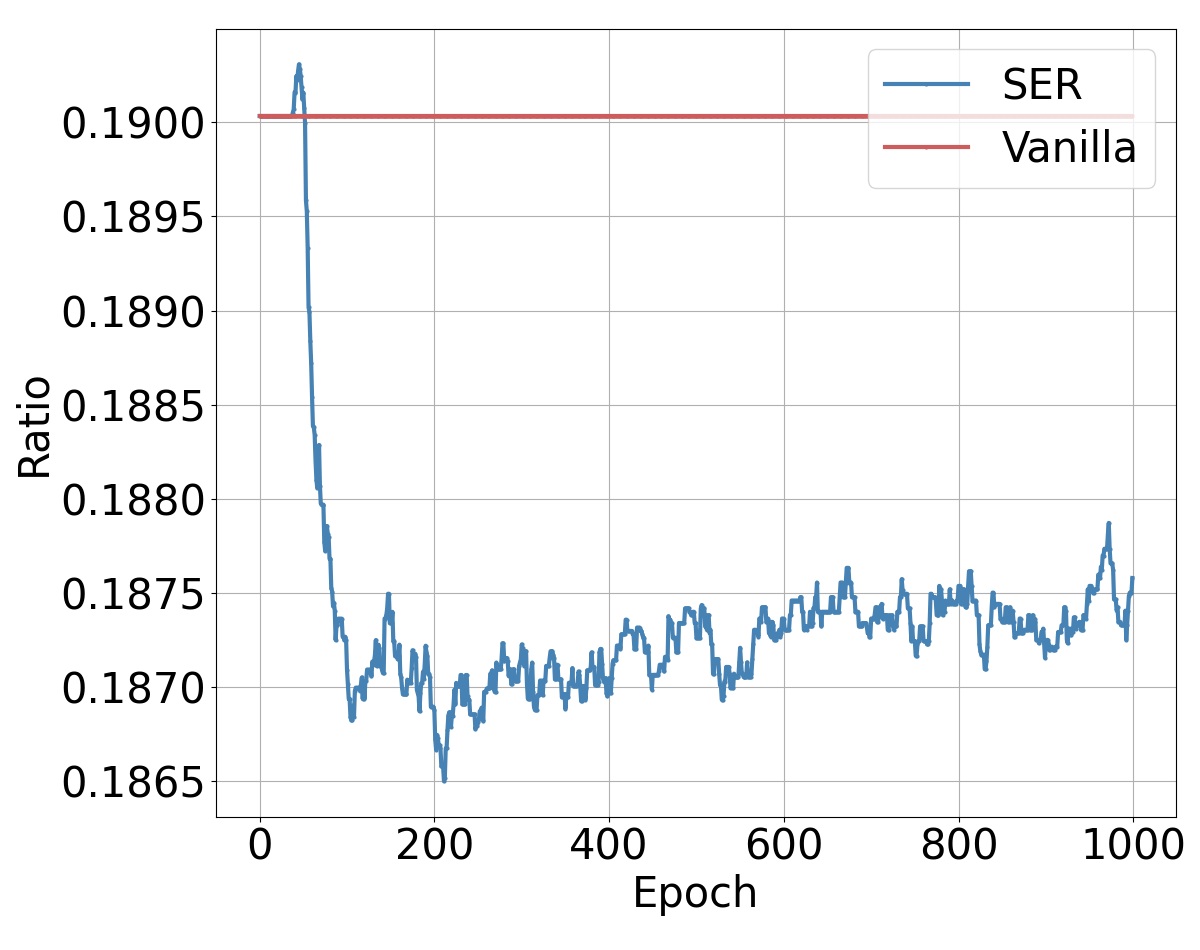}
        \includegraphics[width=1.04\linewidth,height=3.9cm]{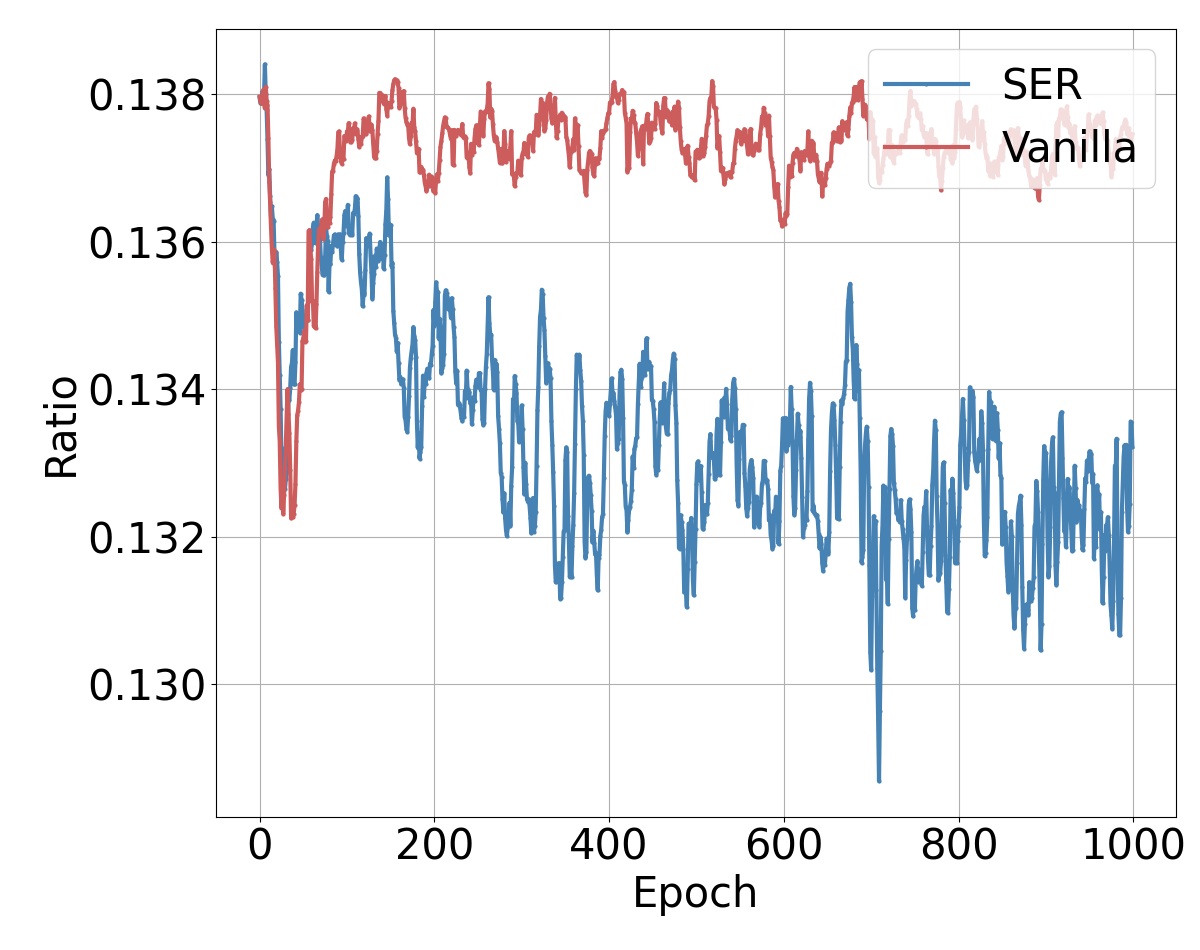}        \includegraphics[width=1.04\linewidth,height=3.9cm]{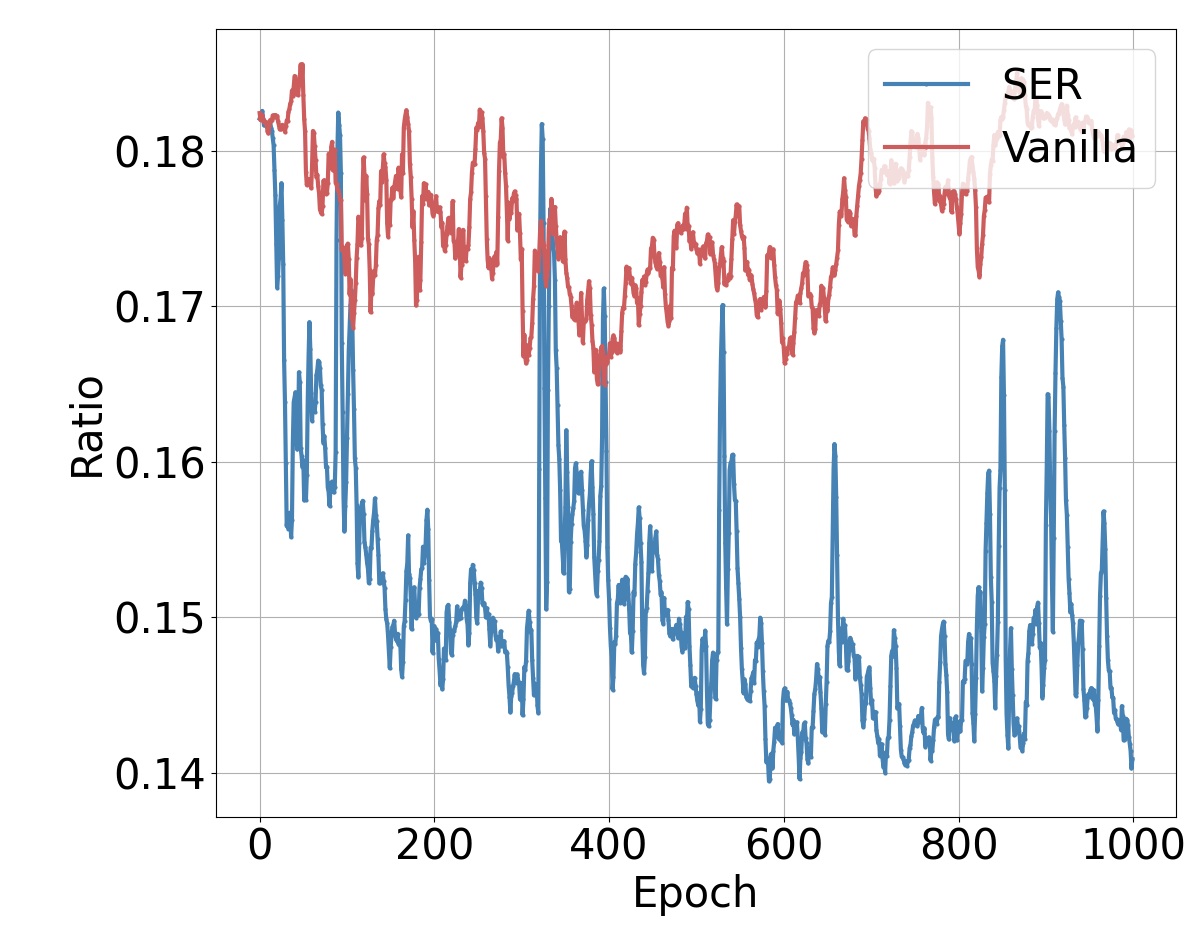}
    \end{minipage}
}
\subfigure[CiteSeer]
{
 	\begin{minipage}[b]{.3\linewidth}
        \centering
        \includegraphics[width=1.04\linewidth,height=3.9cm]{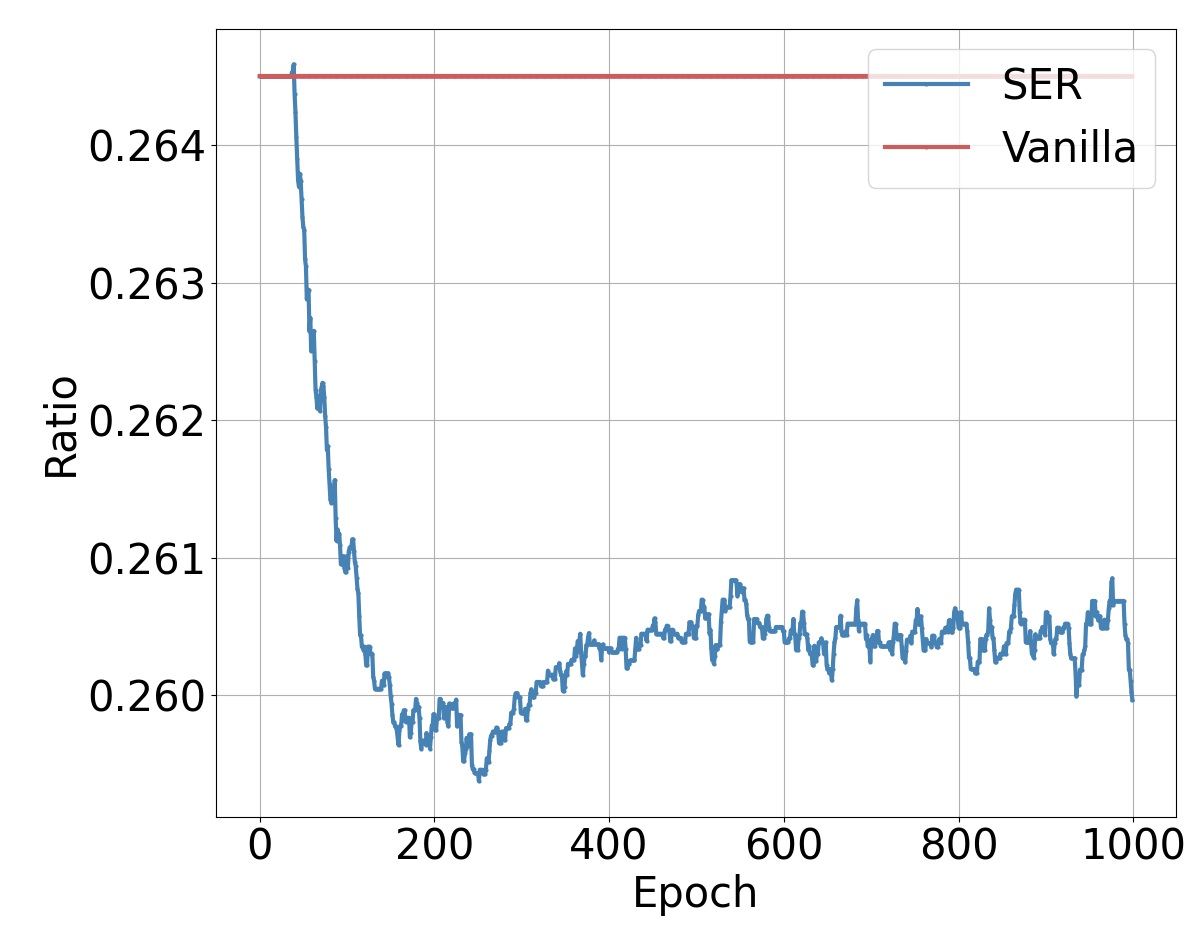}
        \includegraphics[width=1.04\linewidth,height=3.9cm]{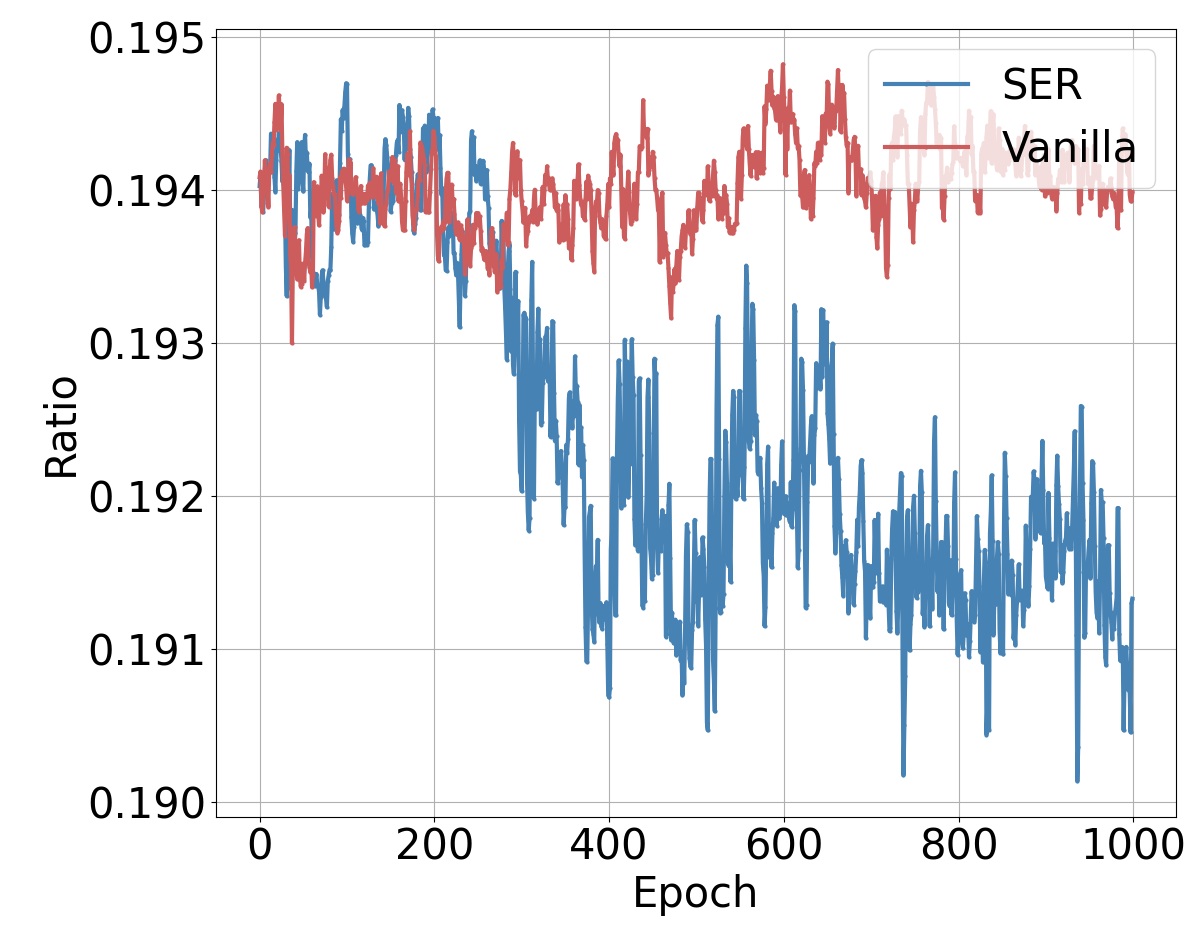}
        \includegraphics[width=1.04\linewidth,height=3.9cm]{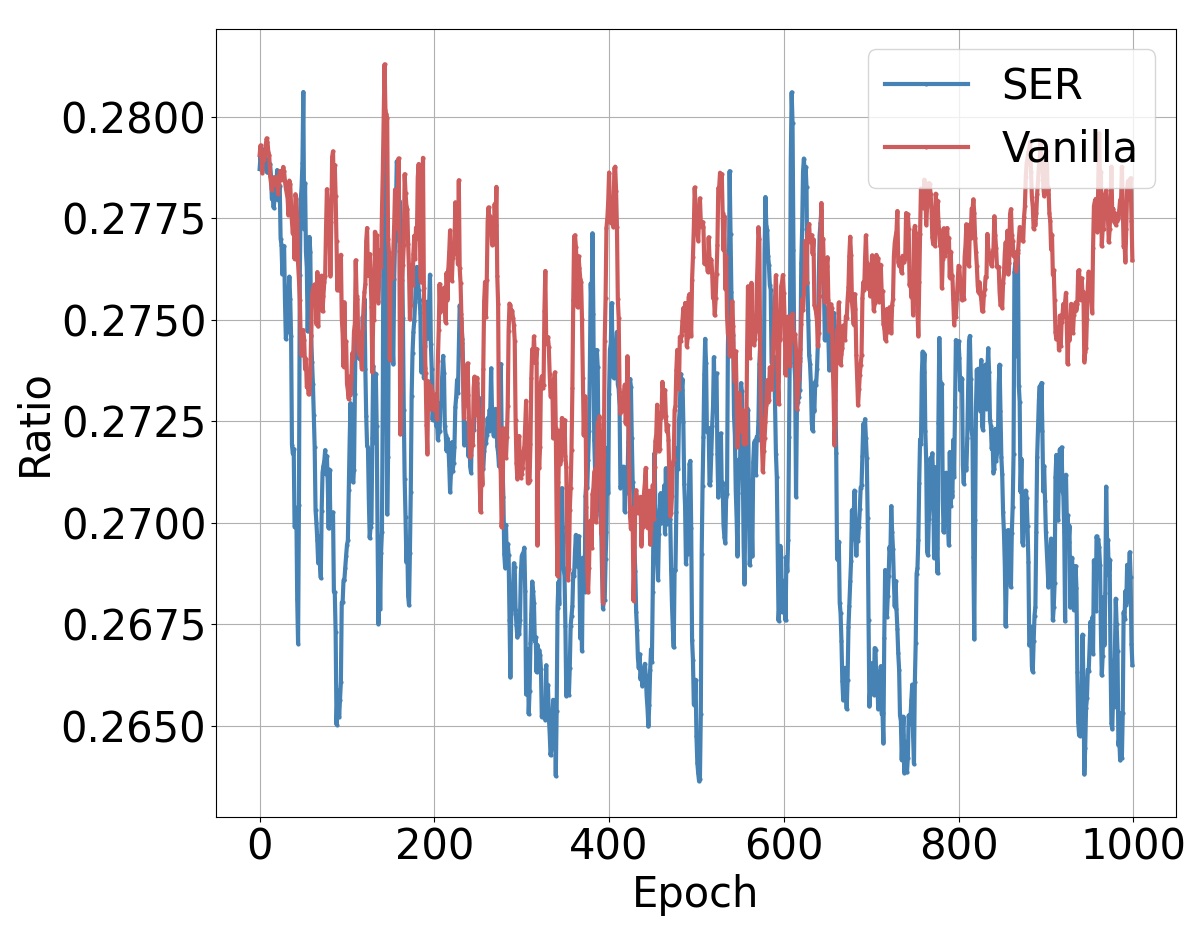}
    \end{minipage}
}
\subfigure[CS]
{
 	\begin{minipage}[b]{.3\linewidth}
        \centering
        \includegraphics[width=1.04\linewidth,height=3.9cm]{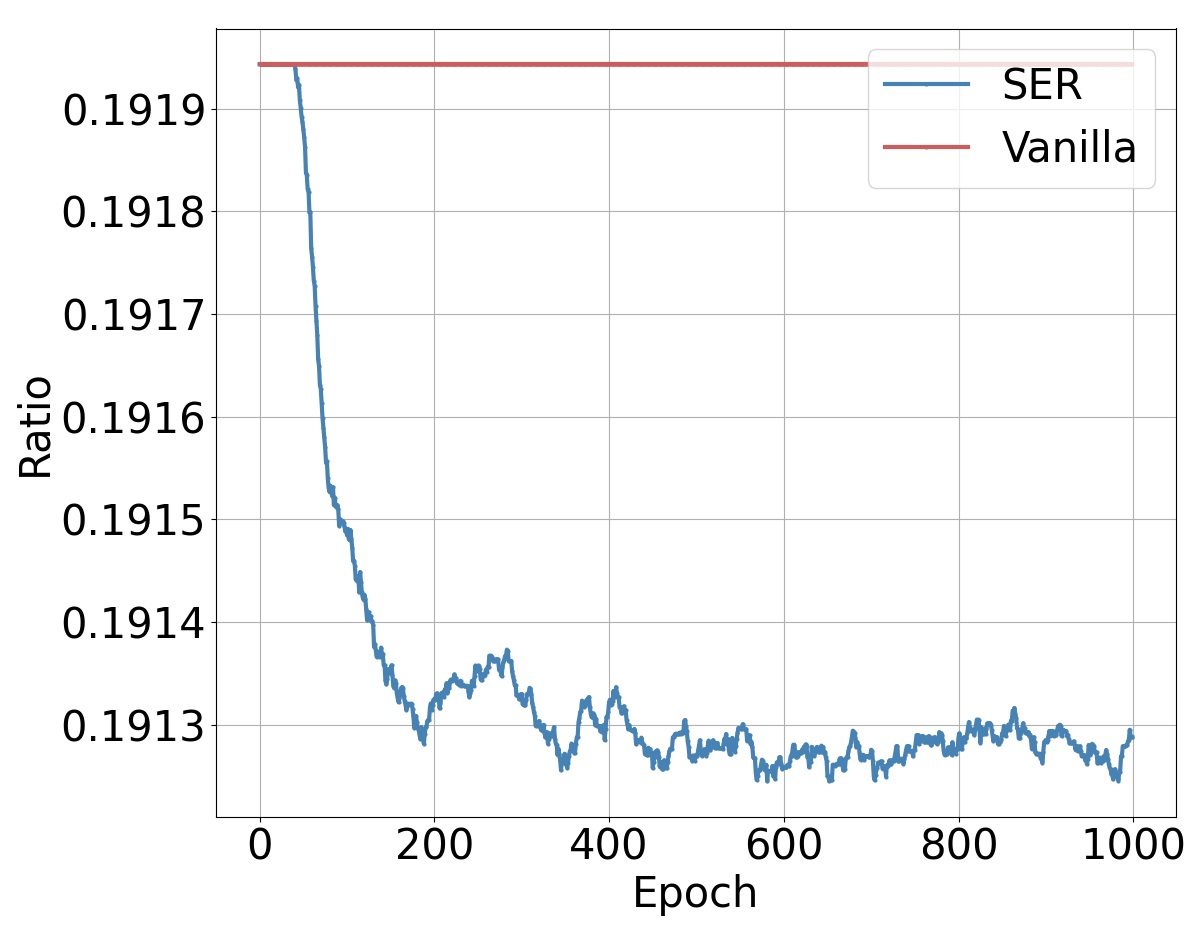}
        \includegraphics[width=1.04\linewidth,height=3.9cm]{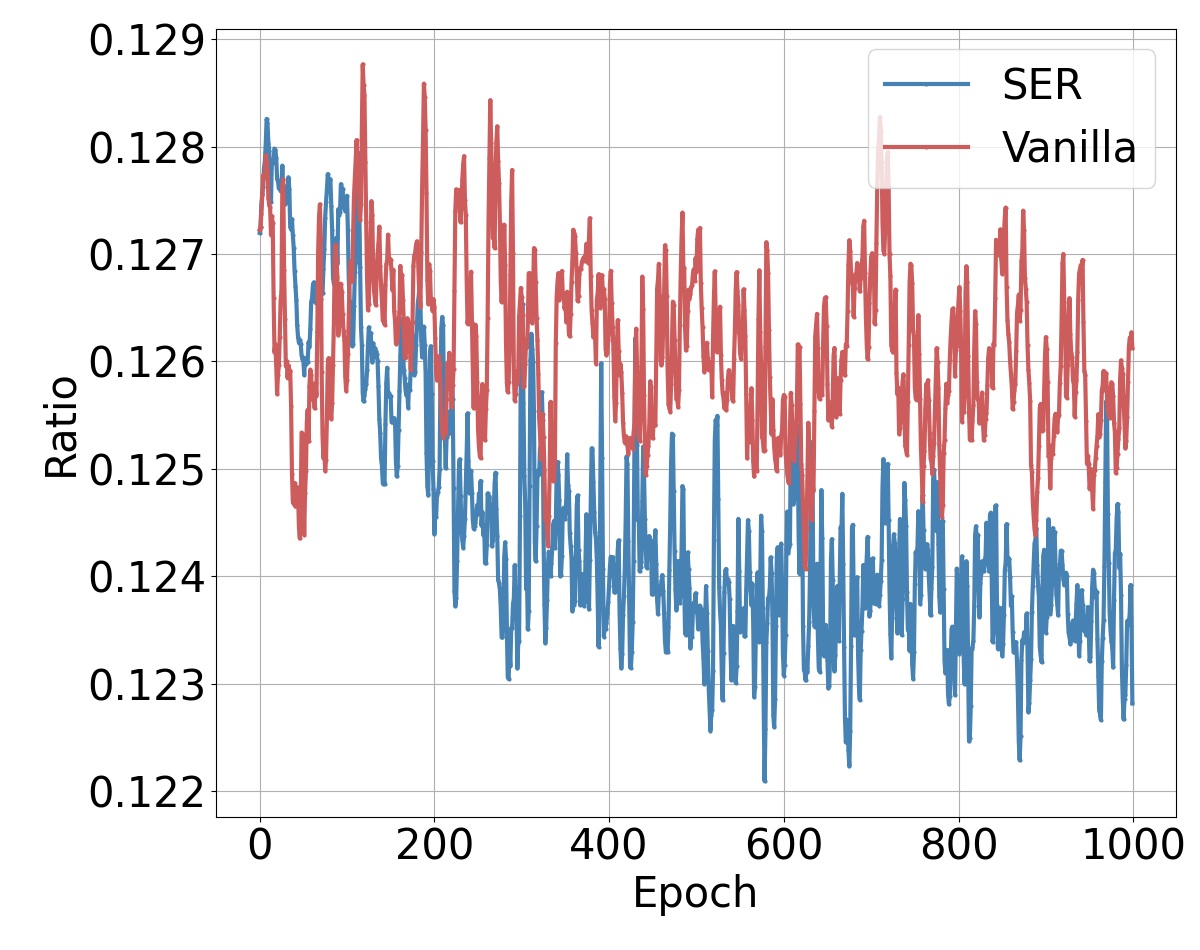}
        \includegraphics[width=1.04\linewidth,height=3.9cm]{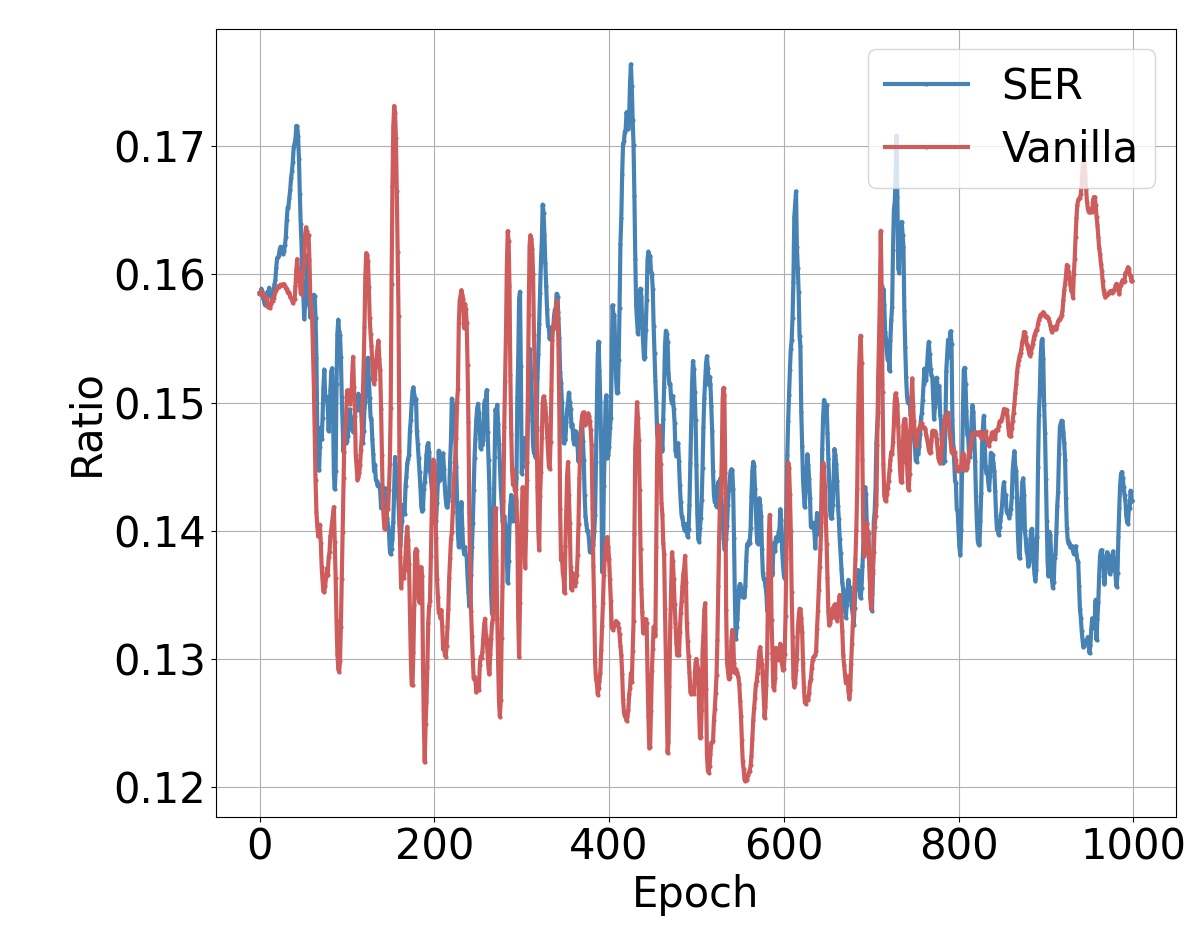}
    \end{minipage}
}
\caption{Evolution of the cross-class aggregation ratio over training epochs. The rows from top to bottom correspond to GCN, GAT, and GT backbones, respectively. The red curves denote the vanilla models, while the blue curves represent models with SER.}
\label{cross_edge}
\end{figure*}

\subsection{Sensitivity to Structural Regularization Strength}

To further investigate the impact of the regularization strength, we analyze the sensitivity of performance with respect to $\lambda$ on six benchmark datasets using GCN, GAT, and GT backbones. The experimental results are illustrated in Figure \ref{GCN_lam}, Figure \ref{GAT_lam} and Figure \ref{GT_lam}, respectively. We observe that moderate values of $\lambda$ consistently improve performance, while overly large values may slightly degrade accuracy. In most datasets, the performance begins to decline once $\lambda$ exceeds a moderate range, typically around 0.6 to 0.7, indicating that excessive regularization can be detrimental. From a theoretical perspective, this behavior aligns with the trade-off characterized in our generalization analysis. The proposed regularization reduces structural complexity by suppressing effective edges involved in aggregation, which helps control overfitting. However, when $\lambda$ becomes too large, the induced sparsity overly restricts message passing and limits the model’s ability to propagate informative signals across the graph. As a result, the effective receptive field becomes insufficient to capture meaningful structural dependencies, leading to underfitting. Therefore, $\lambda$ implicitly balances the trade-off between overfitting and underfitting. The existence of an optimal intermediate range of $\lambda$ empirically validates this trade-off predicted by our analysis.

These findings further demonstrate that structural entropy serves as a flexible and architecture-agnostic mechanism for controlling effective connectivity in message-passing networks.

\subsection{Analysis of Effective Edge Usage}

To further validate the mechanism of SER, we analyze how it influences the effective message-passing structure during training. In particular, we focus on two key aspects: (1) the proportion of effective edges involved in aggregation, and (2) the proportion of cross-class aggregation. Following our theoretical formulation, the effective edges at layer $k$ are associated with the support of the aggregation matrix $\Omega^{(k)}$. We therefore measure the effective edge ratio as the proportion of edges with non-negligible weights, and track its evolution during training. In addition, we quantify cross-class aggregation as the ratio of aggregation weights connecting nodes from different classes.

The results are illustrated in Figure \ref{eff_edge} and Figure \ref{cross_edge}, respectively. We observe that incorporating SER consistently reduces the number of effective edges compared to the vanilla models across all architectures. This indicates that SER encourages a sparser and more selective message-passing structure, effectively pruning redundant or less informative connections. This behavior aligns with our theoretical motivation: by minimizing structural entropy, the model seeks a sparse yet informative representation of the graph, effectively reducing the structural complexity term in the generalization bound. Moreover, SER significantly decreases the proportion of cross-class aggregation during training. While vanilla models tend to indiscriminately aggregate information from neighboring nodes, including those from different classes, SER suppresses these cross-class interactions and promotes within-class information propagation. This effect is particularly evident in later training stages, where the aggregation structure becomes increasingly aligned with the underlying class partition. These results provide strong empirical evidence that SER improves generalization by explicitly regulating the effective message-passing structure.

\section{Conclusion}
\label{conclusion}

In this paper, we develop a unified theoretical framework to understand the generalization performance of GNNs from the perspective of graph structure. We first theoretically show that graph structure itself can induce overfitting in GNNs. As structural information is embedded into the encoding and propagation process, node representations become smoother across the graph, which facilitates fitting on the training data but increases the reliance of the model on the specific structure of the training graph. Building on this observation, we introduce structural complexity defined by the number of effective aggregation edges. This structural perspective enables a principled analysis of how the extent of structural participation influences generalization. Based on this notion, we derive Rademacher complexity-based generalization bounds for representative GNN architectures, including GCNs, GATs, and GTs. Our results reveal that, beyond classical parameter norms and sample size, the generalization performance of GNNs depends explicitly on how extensively graph edges are utilized during message passing. This provides a unified characterization of model capacity that integrates both parameter complexity and structural complexity. Motivated by these theoretical insights, we proposed SER, a differentiable structure-aware regularization strategy that implicitly controls the number of effective edges while preserving the beneficial smoothing effect of aggregation. Empirical results demonstrate that SER consistently improves generalization performance across different architectures. Our framework not only fills a crucial gap in the theoretical understanding of the generalization of GNNs but also provides valuable insights into improving the performance of GNNs in practical applications. We believe this work opens new directions for future research on the generalization of GNNs to complex and structurally diverse unseen graphs.

% Acknowledgements and Disclosure of Funding should go at the end, before appendices and references

\acks{This work is supported by the National Natural Science Foundation of China (No.62432006, 62276159) and the Fundamental Research Program of Shanxi Province (No. 202303021223004).}


%\newpage

\vskip 0.2in
\bibliography{main}

\end{document}